\documentclass{article}

\usepackage{arxiv}

\usepackage[utf8]{inputenc} 
\usepackage[T1]{fontenc}    
\usepackage{hyperref}       
\usepackage{url}            
\usepackage{booktabs}       
\usepackage{amsfonts}       
\usepackage{nicefrac}       
\usepackage{microtype}      
\usepackage{lipsum}

\usepackage[title]{appendix}
\usepackage[nottoc]{tocbibind}

\pdfoutput=1

\title{The Trimmed Lasso: Sparse Recovery Guarantees and Practical Optimization by the Generalized Soft-Min Penalty\thanks{
Accepted to SIAM Journal on Mathematics of Data Science (SIMODS) on Jun 10, 
2021}}

\author{Tal Amir\thanks{Weizmann Institute of Science, Rehovot, IL 
  (\email{tal.amir@weizmann.ac.il},
  \email{ronen.basri@weizmann.ac.il}, \email{boaz.nadler@weizmann.ac.il}  ).}
\and Ronen Basri\footnotemark[2]
\and Boaz Nadler\footnotemark[2]}

\usepackage{bibunits}

\newcommand{\myparagraph}[1]{\paragraph{{#1}}}

\usepackage{amsmath}
\usepackage{amsthm}
\usepackage{hyperref}
\usepackage{cleveref}

\usepackage{lipsum}
\usepackage{amsfonts}
\usepackage{amssymb}
\usepackage{subfig}
\usepackage{graphicx}

\usepackage{pgfplots}
\pgfplotsset{compat=1.14}
\usetikzlibrary{plotmarks}

\usepackage{epstopdf}
\usepackage{placeins}
  
\usepackage{tabulary}
\usepackage{multirow}

\usepackage{caption} 

\makeatletter
\newcommand*{\addFileDependency}[1]{
  \typeout{(#1)}
  \@addtofilelist{#1}
  \IfFileExists{#1}{}{\typeout{No file #1.}}
}
\makeatother

\usepackage{algorithmicx}
\usepackage{algorithm}
\usepackage{algpseudocode}
\algnewcommand\algorithmicinput{\textbf{Input:}}
\algnewcommand\algorithmicoutput{\textbf{Output:}}
\algnewcommand\algorithmicvariables{\textbf{Arrays:}}
\algnewcommand\Input{\item[\algorithmicinput]}
\algnewcommand\Output{\item[\algorithmicoutput]}
\algnewcommand\Variables{\item[\algorithmicvariables]}

\usepackage{etoolbox}

\makeatletter
\AfterEndEnvironment{algorithm}{\let\@algcomment\relax}
\AtEndEnvironment{algorithm}{\kern2pt\hrule\relax\vskip3pt\@algcomment}
\let\@algcomment\relax
\newcommand\algcomment[1]{\def\@algcomment{\footnotesize#1}}
\renewcommand\fs@ruled{\def\@fs@cfont{\bfseries}\let\@fs@capt\floatc@ruled
  \def\@fs@pre{\hrule height.8pt depth0pt \kern2pt}%
  \def\@fs@post{}%
  \def\@fs@mid{\kern2pt\hrule\kern2pt}%
  \let\@fs@iftopcapt\iftrue}
\makeatother

\ifpdf
  \DeclareGraphicsExtensions{.eps,.pdf,.png,.jpg}
\else
  \DeclareGraphicsExtensions{.eps}
\fi

\graphicspath{{eps/}}

\newcommand{\subfighspace} {{\ \ \ \ }}

\usepackage[T1]{fontenc}
\usepackage[utf8]{inputenc}
\usepackage{mathtools}

\usepackage{listings}

\usepackage{enumitem}
\setlist[enumerate]{leftmargin=.5in}
\setlist[itemize]{leftmargin=.5in}

\usepackage{xparse}

\numberwithin{equation}{section}
\crefname{equation}{Equation}{Equations}

\newtheorem{theorem}{Theorem}
\numberwithin{theorem}{section}
\newtheorem{lemma}[theorem]{Lemma}
\newtheorem{corollary}[theorem]{Corollary}

\newtheorem*{definition*}{Definition}

\usepackage{amsopn}

\DeclareMathOperator*{\argmin}{\arg\!\min}
\DeclareMathOperator*{\sign}{sign}
\DeclareMathOperator*{\mean}{mean}

\newcommand{\mb}[1]{ \mbox{\boldmath$#1$} }

\newcommand{\beq}{\begin{equation}}
\newcommand{\eeq}{\end{equation}}
\newcommand{\beqr}{\begin{eqnarray}}
\newcommand{\eeqr}{\end{eqnarray}}

\newcommand{\x}{\mb{x}}
\newcommand{\xhat}{\mb{\hat{\x}}}
\newcommand{\xtilde}{\mb{\tilde{\x}}}
\newcommand{\xstar}{{\x^*}}
\newcommand{\e}{\mb{e}}
\newcommand{\y}{\mb{y}}

\newcommand{\uvec}{\mb{u}}
\newcommand{\vvec}{\mb{v}}

\newcommand{\bvec}{{\bf b}}

\newcommand{\zerovec}{{\bf 0}}

\newcommand{\w}{\mathbf{w}}

\newcommand{\floor}[1]{{ \left \lfloor {#1} \right \rfloor }}

\newcommand{\brs}[1]{{ \left[ {#1} \right] }} 
\newcommand{\bra}[1]{{ \left< {#1} \right> }} 
\newcommand{\brc}[1]{{ \left \{ {#1} \right \} }} 
\newcommand{\br}[1]{{ \left( {#1} \right) }} 

\newcommand{\bigbr}[1]{{ \bigl( {#1} \bigr) }} 
\newcommand{\Bigbr}[1]{{ \Bigl( {#1} \Bigr) }} 
\newcommand{\Biggbr}[1]{{ \Biggl( {#1} \Biggr) }} 

\newcommand{\Bigbrs}[1]{{ \Bigl[ {#1} \Bigr] }} 
\newcommand{\Biggbrs}[1]{{ \Biggl[ {#1} \Biggr] }} 

\newcommand{\abs}[1]{{ \left| {#1} \right| }} 
\newcommand{\xord}[1]{{{\left|{\mb{x}}\right|}_{\left({#1}\right)}}}
\newcommand{\norm}[1]{{ \left \| {#1} \right \| }} 
\newcommand{\eqdef} {{\ = \ }}
\newcommand{\assign} {{\ := \ }}
\newcommand{\subplus}[1]{ [{#1}]_+}

\newcommand{\ofx}{(\x)}
\newcommand{\ofxhat}{(\hat{\x})}
\newcommand{\ofxstar}{(\xstar)}

\newcommand{\z}{{\bf z}}

\newcommand{\ofz} {{\left( \z \right)}}

\newcommand{\ofzleft} {{\left( \zleft \right)}}
\newcommand{\ofzright} {{\left( \zright \right)}}

\newcommand{\qhat}{{\hat q}}

\newcommand{\nbdash}[0]{{\nobreakdash-\ignorespaces}}

\newcommand{\twosetbinom}[4]
{\alpha_{{#3},{#4}}^{{#1},{#2}}}

\newcommand{\leftmod}[1]{{\dot{#1}}}
\newcommand{\rightmod}[1]{{\ddot{#1}}}

\newcommand{\zleft}{{\leftmod{\z}}}
\newcommand{\zright}{{\rightmod{\z}}}

\newcommand{\zleftentry}{{\leftmod{z}}}
\newcommand{\zrightentry}{{\rightmod{z}}}

\newcommand{\dleft}{{\leftmod{d}}}
\newcommand{\dright}{{\rightmod{d}}}

\newcommand{\projk} {{\Pi _k}}
\newcommand{\tauk} {{{\tau}_{k}}}
\newcommand{\taudmk} {{{\tau}_{d-k}}}

\newcommand{\tauargs}[2] {{{\tau}_{{#1},{#2}}}}
\newcommand{\taukg} {{{\tau}_{k,\gamma}}}
\newcommand{\taukzero} {{{\tau}_{k,0}}}
\newcommand{\taukinfty} {{{\tau}_{k,\infty}}}

\newcommand{\mukg} {{\mu_{k,\gamma}}}
\newcommand{\muargs}[2] {{\mu_{{#1},{#2}}}}

\newcommand{\gsm} {{{\textrm{$\boldsymbol{\theta}$}}_{k,\gamma}}}
\newcommand{\gsmargs}[2] {{{\textrm{$\boldsymbol{\theta}$}}_{{#1},{#2}}}}
\newcommand{\gsmi}[1] {{{\theta_{k,\gamma}^{#1}}}}
\newcommand{\gsmiargs}[3] {{{\theta_{{#1},{#2}}^{#3}}}}
\newcommand{\boldtheta} {{{\textrm{$\boldsymbol{\theta}$}}}}

\newcommand{\Ftwol} {{\mathrm{F}_{\lambda}}}
\newcommand{\Ftwolg} {{\mathrm{F}_{\lambda,\gamma}}}
\newcommand{\Ftwolzero} {{\mathrm{F}_{\lambda,0}}}
\newcommand{\Ftwolinf} {{\mathrm{F}_{\lambda,\infty}}}
\newcommand{\Ftwolw} {{\mathrm{F}_{\lambda,\w}}}
\newcommand{\Gtwolg} {{\mathrm{G}_{\lambda,\gamma}}}
\newcommand{\Gtwolzero} {{\mathrm{G}_{\lambda,0}}}
\newcommand{\Gtwolinf} {{\mathrm{G}_{\lambda,\infty}}}
\newcommand{\Fonel} {{\mathrm{F}^1_{\lambda}}}
\newcommand{\Fonelg} {{\mathrm{F}^1_{\lambda,\gamma}}}
\newcommand{\Fonelw} {{\mathrm{F}^1_{\lambda,\w}}}

\newcommand{\Gonelg} {{\mathrm{G}^1_{\lambda,\gamma}}}

\newcommand{\taumajkg} {{{\phi}_{k,\gamma}}}

\newcommand{\taumajkinf} {{{\phi}_{k,\infty}}}

\newcommand{\lambdasmall} {{\lambda_a}}
\newcommand{\lambdalarge} {{\lambda_b}}
\newcommand{\lambdabar} {{\bar{\lambda}}}
\newcommand{\ripmin} {{\alpha_{2k}}}

\newcommand{\wfunc} {{\w_{k,\gamma}}}

\newcommand{\wfuncidx}[1] {{w_{k,\gamma}^{#1}}}
\newcommand{\wfunczeroidx}[1] {{w_{k,0}^{#1}}}

\newcommand{\wfunckinfty} {{\w_{k,\infty}}}
\newcommand{\wargs}[3] {w_{{#1},{#2}}^{#3}}

\newcommand{\xstatset}{{\mathcal{X}_{\mathrm{stat}}}}
\newcommand{\xambset}{{\mathcal{X}_{\mathrm{amb}}}}
\newcommand{\xoptset}{{\mathcal{X}_{\mathrm{lmin}}}}

\newcommand{\logonep}{{\mathrm{log1p}}}
\newcommand{\expmone}{{\mathrm{expm1}}}

\renewcommand{\subplus}[1]{{\brs{{#1}}_+}}
\newcommand{\subminus}[1]{{\brs{{#1}}_-}}

\newcommand{\setst}[2]{ \left\{ {#1} \,\middle|\, {#2} \right\} }

\newcommand{\inprod}[2]{{ \bra{{#1},\ {#2}} }}

\newcommand{\hadamard} {{\odot}} 

\newcommand{\argdot} {{\boldsymbol{\cdot}}}

\newcommand{\twofloors}[2]
{{ 
	\begin{matrix}
		{#1} \\
		{#2}
	\end{matrix}
}}

\newcommand{\threefloors}[3]
{{ 
	\begin{matrix}
		{#1} \\
		{#2} \\
		{#3}
	\end{matrix}
}}

\newcommand{\twocase}[4]
{{\left \{ 
\begin{matrix}
	{#1} & {#2} \\[10pt]
	{#3} & {#4}
\end{matrix}
\right.}}

\newcommand{\threecase}[6]
{{\left \{ 
		\begin{matrix}
			{#1} & {#2} \\[10pt]
			{#3} & {#4} \\[10pt]
			{#5} & {#6}
		\end{matrix}
		\right.}}

\newcommand{\fourcase}[8]
{{\left \{ 
		\begin{matrix}
			{#1} & {#2} \\[10pt]
			{#3} & {#4} \\[10pt]
			{#5} & {#6} \\[10pt]
			{#7} & {#8}
		\end{matrix}
		\right.}}

\DeclareDocumentCommand \R{ o }{%
	{ \mathbb{R}
		\IfNoValueF {#1} {^{#1}}%
	}
}

\DeclareDocumentCommand \Z{ o }{%
	{ \mathbb{Z}
		\IfNoValueF {#1} {^{#1}}%
	}
}

\DeclareDocumentCommand \eq{ o o g}
{	\IfNoValueTF {#1}
		{   
			\begin{equation*} \begin{split} #3 \end{split} \end{equation*}}
		{	\IfNoValueTF {#2}
			{\begin{equation} {\label{#1}} \begin{split} #3 \end{split} \end{equation}}
			{\begin{equation*} {\label{#1}} {\tag{#2}}  \begin{split} #3 \end{split} \end{equation*}}	}	} 

\DeclareDocumentCommand \summ{ o o }{%
	{ 
		\IfNoValueTF {#1}
		{   
			\sum 
		}	
		{	
			\IfNoValueTF {#2}		
			{ 
				\sum_{#1} 
			}
			{ 
				\sum_{#1}^{#2}
			}		
		}
	}
}

\DeclareDocumentCommand \summfixed{ o o }{%
	{ 
		\IfNoValueTF {#1}
		{   
			\sum 
		}	
		{	
			\IfNoValueTF {#2}		
			{ 
				\underset{\scriptstyle {#1}}{\sum} 
			}
			{ 
				\overset{\scriptstyle {#2}}{\underset{\scriptstyle {#1}}{\sum}}
			}		
		}
	}
}

\DeclareDocumentCommand \evalat{ g o o }{%
	{ 
		\IfNoValueTF {#2}
		{   
			\phantom{\bigg\lvert}{#1}\bigg\rvert
		}	
		{	
			\IfNoValueTF {#3}		
			{ 
				\phantom{\bigg\lvert}{#1}\bigg\rvert_{#2}
			}
			{ 
				\phantom{\bigg\lvert}{#1}\bigg\rvert_{#2}^{#3} 
			}		
		}
	}
}

\DeclareDocumentCommand \app{ o o }{%
	{ 
		\IfNoValueTF {#1}
		{   
			\IfNoValueTF {#2}
			{ 
				\rightarrow
			}
			{ 
				\overset{#2}{\longrightarrow}
			}
		}	
		{	
			\IfNoValueTF {#2}
			{ 
				\underset{#1}{\longrightarrow}
			}
			{ 
				\underset{#1}{\overset{#2}{\longrightarrow}}
			}
		}
	}
}

\DeclareDocumentCommand \limit{ o }{%
	{ 
		\IfNoValueTF {#1}
		{   
    		{\lim}
		}	
		{	
            \underset{#1}{\lim}
        }
	}
}

\usepackage{xcolor}
\definecolor{colred}{rgb}{0.75, 0, 0}
\definecolor{colviolet}{rgb}{0.5, 0, 1}

\ifpdf
\hypersetup{
  pdftitle={The Trimmed Lasso},
  pdfauthor={T. Amir, R. Basri, and B. Nadler},
  colorlinks=true,
  linkcolor=[rgb]{0,0.4,0},
  citecolor=[rgb]{0,0,0.7},
  urlcolor=[rgb]{0.5,0,0}
}
\fi

\ifpdf
\hypersetup{
  pdftitle={The Trimmed Lasso: Sparse recovery guarantees and practical optimization by the Generalized Soft-Min Penalty},
  pdfauthor={T. Amir, R. Basri, and B. Nadler}
}
\fi

\crefformat{equation}{\textup{#2(#1)#3}}
\crefrangeformat{equation}{\textup{#3(#1)#4--#5(#2)#6}}
\crefmultiformat{equation}{\textup{#2(#1)#3}}{ and \textup{#2(#1)#3}}
{, \textup{#2(#1)#3}}{, and \textup{#2(#1)#3}}
\crefrangemultiformat{equation}{\textup{#3(#1)#4--#5(#2)#6}}%
{ and \textup{#3(#1)#4--#5(#2)#6}}{, \textup{#3(#1)#4--#5(#2)#6}}{, and \textup{#3(#1)#4--#5(#2)#6}}

\Crefformat{equation}{Equation~#2\textup{(#1)}#3}
\Crefrangeformat{equation}{Equations~\textup{#3(#1)#4--#5(#2)#6}}
\Crefmultiformat{equation}{Equations~\textup{#2(#1)#3}}{ and \textup{#2(#1)#3}}
{, \textup{#2(#1)#3}}{, and \textup{#2(#1)#3}}
\Crefrangemultiformat{equation}{Equations~\textup{#3(#1)#4--#5(#2)#6}}%
{ and \textup{#3(#1)#4--#5(#2)#6}}{, \textup{#3(#1)#4--#5(#2)#6}}{, and \textup{#3(#1)#4--#5(#2)#6}}

\crefname{section}{Section}{Sections}
\Crefname{section}{Section}{Sections}
\crefformat{section}{#2Section~\textup{#1}#3}
\Crefformat{section}{#2Section~\textup{#1}#3}

\crefname{appendix}{Appendix}{Appendices}
\Crefname{appendix}{Appendix}{Appendices}
\crefformat{appendix}{#2Appendix~\textup{#1}#3}
\Crefformat{appendix}{#2Appendix~\textup{#1}#3}

\crefformat{figure}{#2figure~\textup{#1}#3}
\Crefformat{figure}{#2Figure~\textup{#1}#3}

\crefname{theorem}{Theorem}{Theorems}
\Crefname{theorem}{Theorem}{Theorems}
\crefformat{theorem}{#2Theorem~\textup{#1}#3}
\Crefformat{theorem}{#2Theorem~\textup{#1}#3}

\crefname{lemma}{Lemma}{Lemmas}
\Crefname{lemma}{Lemma}{Lemmas}
\crefformat{lemma}{#2Lemma~\textup{#1}#3}
\Crefformat{lemma}{#2Lemma~\textup{#1}#3}

\crefname{corollary}{Corollary}{Corollaries}
\Crefname{corollary}{Corollary}{Corollaries}
\crefformat{corollary}{#2Corollary~\textup{#1}#3}
\Crefformat{corollary}{#2Corollary~\textup{#1}#3}

\crefname{algorithm}{Algorithm}{Algorithms}
\Crefname{algorithm}{Algorithm}{Algorithms}
\crefformat{algorithm}{#2Algorithm~\textup{#1}#3}
\Crefformat{algorithm}{#2Algorithm~\textup{#1}#3}

\begin{document}

\author{Tal Amir, Ronen Basri, and Boaz Nadler\thanks{Weizmann Institute of Science, Rehovot, IL. Emails: \{\texttt{tal.amir, ronen.basri, boaz.nadler}\}\texttt{@weizmann.ac.il}}}

\maketitle

\begin{abstract}
We present a new approach to solve the sparse approximation or best subset selection problem, namely find a $k$-sparse vector ${\bf x}\in\mathbb{R}^d$ that minimizes the $\ell_2$ residual $\lVert A{\bf x}-{\bf y} \rVert_2$. We consider a regularized approach, whereby this residual is penalized by the non-convex $\textit{trimmed lasso}$, defined as the $\ell_1$-norm of ${\bf x}$ excluding its $k$ largest-magnitude entries. We prove that the trimmed lasso has several appealing theoretical properties, and in particular derive sparse recovery guarantees assuming successful optimization of the penalized objective. Next, we show empirically that directly optimizing this objective can be quite challenging. Instead, we propose a surrogate for the trimmed lasso, called the $\textit{generalized soft-min}$. This penalty smoothly interpolates between the classical lasso and the trimmed lasso, while taking into account all possible $k$-sparse patterns. The generalized soft-min penalty involves summation over $\binom{d}{k}$ terms, yet we derive a polynomial-time algorithm to compute it. This, in turn, yields a practical method for the original sparse approximation problem. Via simulations, we demonstrate its competitive performance compared to current state of the art.
\end{abstract}

\section{Introduction}
\label{sec:intro}
%
Consider the {sparse approximation} or {best subset selection} problem: Given an $n\times d$ matrix $A$, a vector $\y\in\mathbb{R}^n$ and a sparsity parameter $k\ll\min\brc{n,d}$, solve 
\begin{equation}
        \tag{P0}
                \label{pr:P0}
\underset{\x}{\min} \ \norm{A \x- \y}_2 \quad \mbox{s.t.} \quad \norm{\x}_0 \leq k.
\end{equation}
The inverse problem \cref{pr:P0} and related variants play a key role in multiple fields. Examples include signal and image processing \cite{bruckstein2009sparse,elad2010sparsebook,Mallat_book}, 
compressed sensing \cite{eldar2012compressed,Foucart_Rauhut},
medical imaging \cite{lustig2007sparse,provost2009application}, computer vision \cite{Wright}, 
high dimensional statistics \cite{Hastie_book}, 
biology \cite{Wu_Lasso} and economics \cite{brodie2009sparse,fan2011sparse}.
While in few cases the sparsity level $k$ is known, it often needs to be estimated. Typically one solves \cref{pr:P0} for several values of $k$ and 
applies cross validation \cite{ward2009compressed} or a model selection criterion. Here we focus on solving \cref{pr:P0} for a given value of $k$.

Even though solving \cref{pr:P0} is NP-hard \cite{davis1997adaptive,natarajan1995sparse},  several approaches were developed to seek approximate solutions. 
One approach is to search for the best $k$ variables greedily, 
adding one or several variables at a time, and possibly discarding some of the previously chosen ones. Examples include iterative thresholding,
matching pursuit and forward-backward methods; see \cite{blumensath2008ihtfirst,dai2009subspace,davis1994,mallat1993matching,Miller_book,needell2009cosamp,needell2010signal,pati1993orthogonal,
rebollo2002optimized,Tropp_Gilbert} and references therein. 

A different approach is to replace the constraint $\norm{\x}_0 \leq k$
by a penalty term $\rho\ofx$, and instead of \cref{pr:P0} solve the regularized problem 
\begin{equation}
\label{pr:P_rho}  
        \min_{\x}\ \tfrac{1}{2} \norm{A\x-\y }^2 + \lambda \rho\ofx.
\end{equation}
With few exceptions detailed below, most penalties are separable and do not depend on $k$.
For some scalar function $f(x)$, they take the form $\rho\ofx = \sum_{i=1}^d f\br{x_i}$. 
A $k$-sparse solution can be obtained by tuning the penalty parameter $\lambda$ and projecting the obtained solutions to be $k$-sparse.
The most popular penalty is the $\ell_1$ norm
\cite{chen2001atomic,tibshirani1996regression}.
Many fast algorithms were developed to optimize \cref{pr:P_rho} with
a convex $\rho(\x)$; see \cite{beck2009fast,Friedman_Reg_Path,Yin_Osher} and references therein.

In many practical settings, matching pursuit algorithms and the $\ell_1$-penalty approaches may output solutions that are quite suboptimal for the original problem \cref{pr:P0}. Thus, various nonconvex penalties have been proposed, such as $\ell_p$ penalties with $p<1$ \cite{chartrand2007exact,figueiredo2005bound}, minimax concave penalty \cite{zhang2010nearly}, smoothly clipped absolute deviation \cite{fan2001variable, hunter2005variable, zou2008one} and the smoothed $\ell_0$ penalty \cite{mohimani2009fast}. 
For even better solutions to \cref{pr:P0}, non-separable penalties were proposed \cite{bogdan2015slope,selesnick2017sparse,wipf2010iterative,Zeng}.  Local minima of these penalized objectives are
often found by iteratively reweighted least squares (IRLS) \cite{chartrand2007exact,chartrand2008iteratively,daubechies2010iteratively,figueiredo2007majorization} or iteratively reweighted $\ell_1$ (IRL1) 
\cite{candes2008enhancing,figueiredo2007majorization,foucart2009sparsest}.

A different approach to \cref{pr:P0} is to pose it as a {mixed integer program} (MIP) and solve it globally; see for example \cite{furnival1974regressions,arthanari1981mathematical,bienstock1996computational}. 
MIP-based methods update both the current solution and a lower bound on the optimal objective. When the current objective
equals this lower bound, the algorithm terminates with a certificate of global optimality. 
With vast improvements in algorithmic efficiency and computing power, this approach gained increased popularity, and various works illustrated its applicability to problems with thousands of variables
\cite{liu2007variable,bertsimas2009algorithm,bertsimas2016best,mazumder2017discrete,bertsimas2020sparse}. 
Recently, \cite{bertsimas2020sparse}
developed a cutting-plane method able to globally solve problems with $d=100000$
variables in less than a minute, provided the sparsity $k$ is not too high. 
MIP may also be combined with other approaches. One recent example is \cite{hazimeh2020fast},
whereby first a local optimum is found by coordinate descent, and then improved upon by a local combinatorial search, which replaces only a subset of the selected variables.
This method is extremely fast and can handle problems with even a million variables or more.
However, it is not guaranteed to find the globally optimal solution. Rather, it finds a local minimum that cannot be improved by replacing a small number of variables.  
Yet, as we demonstrate 
in \cref{sec:numerical_experiments}, despite these significant advancements in MIP-based approaches, our proposed method can successfully handle challenging settings in practical time, where the above methods either take a prohibitive runtime, or find highly suboptimal solutions.
%

%

In this work we focus on the non-separable \emph{trimmed lasso} penalty $\tauk \ofx$, whereby a candidate vector $\x$ is penalized by its $\ell_1$-distance to the nearest $k$-sparse vector. Namely,
\begin{equation}
	\label{eq:def_tauk_a}
\tauk \ofx = \sum_{i=k+1}^d \xord{i},
\end{equation}
where $\xord{1} \geq \xord{2} \geq \ldots \geq \xord{d}$
are the absolute values of the entries of $\x$, sorted in decreasing order. The original problem~\cref{pr:P0} is replaced by either the 
regularized problem
\begin{equation}
\label{pr:P2l} 
\min_{\x}\ \Ftwol \ofx \eqdef \tfrac{1}{2} \norm{ A\x - \y }_2^2 + \lambda \tauk\ofx,
\end{equation}
or by the following variant, which involves the residual norm to the power 1 rather than 2,
\begin{equation}
	\label{pr:P1l}
    \underset{\x}{\min} \ \Fonel \ofx \eqdef \norm{A \x- \y}_2 + \lambda \tauk \ofx.
    \end{equation}

An important property of the trimmed lasso is that it explicitly promotes the sparsity level $k$ of the original problem~\cref{pr:P0}, 
with $\tauk\ofx=0$ if and only if $\x$ is $k$-sparse. 
The penalty $\tauk \ofx$ was discussed in \cite{gotoh2018dc,tono2017efficient}, 
who developed a {difference-of-convex} (DC) programming scheme to optimize 
\cref{pr:P2l}. A closely related penalty that combines 
the lasso and the trimmed lasso was studied independently in \cite{huang2015two}, where a similar DC programming scheme was proposed. The name \emph{trimmed lasso} was coined by 
\cite{bertsimas2017trimmed}, who proposed an {alternating direction method of 
multipliers} (ADMM) scheme. Recently, \cite{yun2019trimming} studied the statistical properties of the trimmed
lasso and developed a block coordinate descent algorithm to optimize \cref{pr:P2l}.

In this work we make several theoretical and algorithmic contributions towards solving the best subset selection problem \cref{pr:P0}.
First, in \cref{sec:theory_TRL} we advocate solving \cref{pr:P0} by minimizing trimmed-lasso penalized objectives.
In particular, we prove that for a sufficiently large penalty parameter $\lambda$, the local minimizers of \cref{pr:P2l} and of \cref{pr:P1l}
are $k$-sparse. Hence, the global minima coincide with those of the original problem \cref{pr:P0}. 
Our results extend those of
\cite{gotoh2018dc,bertsimas2017trimmed}.
Second, we present novel recovery guarantees when optimizing  \cref{pr:P2l} or 
\cref{pr:P1l} for any positive $\lambda$. 
Specifically, assuming that $\y=A\x_0+\mb{e}$, with $\x_0$ approximately $k$\nbdash sparse, we study how well $\x_0$ can be estimated
by optimizing \cref{pr:P2l} or \cref{pr:P1l}. We prove that even at low values of $\lambda$, where optimal solutions of these problems may not coincide with those of \cref{pr:P0}, $\x_0$ can still be well-approximated by these solutions, with the recovery stable to the measurement error $\mb{e}$.

Our theoretical results motivate the development of practical methods to optimize the trimmed-lasso regularized 
objectives \cref{pr:P2l} and \cref{pr:P1l}. 
Unfortunately, directly optimizing these objectives can be challenging. 
One of our key contributions is the development of a new method to solve \cref{pr:P2l} and \cref{pr:P1l}. 
As we show in \cref{sec:numerical_experiments}, our method is able to find solutions 
with significantly lower
objective values than of those computed by DC programming and ADMM. 

In our approach we replace the trimmed lasso $\tau_k\ofx$ by a surrogate penalty called the \emph{generalized soft-min} and denoted  $\tau_{k,\gamma}\ofx$. 
Our proposed penalty depends not only on the sparsity level $k$ but also on a smoothness
parameter $\gamma$. As described in \cref{sec:gsm_penalty}, $\taukg \ofx$ has two important properties: First, for any finite $\gamma$, 
it is $C^\infty$-smooth as a function of $\abs{\x} = \br{|x_1|,\ldots,|x_d|}$. 
This facilitates the use of continuous optimization techniques.  
Second, as $\gamma$ increases, 
$\taukg\ofx$
varies smoothly from the convex $\ell_1$\nbdash norm (at $\gamma=0$) to the trimmed lasso (at $\gamma=\infty$). 
For any $\gamma>0$, $\taukg \ofx$ takes into account all $\binom{d}{k}$ sparsity patterns of $\x$. Its na\"{i}ve computation is thus intractable. Another important contribution, described in \cref{sec:algorithm_gsm},
is the development of a polynomial time algorithm to calculate it in $\mathcal O(k d)$ operations. 

Given the  aforementioned properties of $\taukg$,
in \cref{sec:gsm_penalty} we describe the following approach to
solve the trimmed lasso penalized \cref{pr:P2l}: 
We replace $\tauk$ by $\taukg$, start
from an easy convex problem at $\gamma=0$, and smoothly transform it to \cref{pr:P2l} 
by increasing $\gamma$ towards infinity. 
At $\gamma=0$, our penalty coincides with the lasso, leading to a convex problem. 
Next, we gradually increase $\gamma$ while tracing the path of solutions. 
At each intermediate $\gamma$, we optimize the corresponding objective
by a majorization-minimization scheme, initialized at the 
solution found for the previous $\gamma$.  
Empirically, the resulting path of solutions often converges to a better solution of \cref{pr:P2l} than that obtained by directly optimizing \cref{pr:P2l} with the nonsmooth trimmed lasso.

Finally, we seek solutions to the original problem \cref{pr:P0} by solving
\cref{pr:P2l}  for several values of $\lambda$, followed by projecting each solution to the nearest $k$\nbdash sparse vector and solving a least-squares problem on its support. 
The $k$-sparse vector with the smallest residual norm $\norm{A\x-\y}_2$ is chosen. As we demonstrate empirically in \cref{sec:numerical_experiments}, 
optimizing our smooth surrogate of the trimmed lasso yields 
state-of-the-art results in sparse recovery. 

\vspace{1em}

\myparagraph{Notations and Definitions} We denote by ${\bf a}_i$ the $i$-th column of the $n\times d$ matrix $A$. We denote  $[d]=\{1,\ldots,d\}$. 
For a scalar $x$, $\subplus{x} = \max\brc{x,0}$ and $\subminus{x} = \max\brc{-x,0}$. 
For a vector $\x$, $\projk(\x)$ is its $k$-sparse projection, namely the $k$-sparse vector closest to $\x$ in $\ell_1$-norm, breaking ties arbitrarily. 
For a function $f  : \R[d] \rightarrow \R$, we denote its  directional derivative at a point $\x \in \R[d]$ in direction $\vvec \in \R[d]$ 
by 
\eq{\ensuremath{\nabla_{\vvec} f  \ofx \eqdef \lim_{t \searrow 0} \frac{f  \br{\x+t\vvec} - f  \br{\x}}{t}.}}

\section{Theory for the trimmed lasso}
%
\label{sec:theory_TRL}
%
In this section we 
study the following two key theoretical questions related to the trimmed lasso penalty: 
(i) what is the relation between minimizers of  \cref{pr:P2l} or \cref{pr:P1l} and those of the original problem \cref{pr:P0}, and do the two coincide for sufficiently large $\lambda$? and 
(ii)  assuming that $\x_0$ is approximately $k$\nbdash sparse, can it be recovered from an observed vector $\y=A\x_0+\mb{e}$ by optimizing problems \cref{pr:P2l} or \cref{pr:P1l}? 
The practical question of how to optimize \cref{pr:P2l} or \cref{pr:P1l} is addressed in \cref{sec:gsm_penalty}. 
Proofs are in \cref{sec:proofs_theory_tls_appendix}.

\subsection{Penalty thresholds}
\label{sec:penalty_thresholds}
%
Denote by $\lambdabar$ the threshold
\begin{equation}
    \label{eq:lambdabar}
     \lambdabar \eqdef \norm{\y}_2 \cdot \max_{i=1,\ldots,d} \norm{\mb{a}_i}_2.
 \end{equation}
The following \lcnamecref{thm:large_lambda_p2l} shows that for large $\lambda$, problem~\cref{pr:P2l} is intimately related to \cref{pr:P0}.  
\begin{theorem} \label{thm:large_lambda_p2l}
If $\lambda > \lambdabar$, then any local minimum of \cref{pr:P2l} is $k$\nbdash sparse.
\end{theorem}
A key implication of this theorem is that for $\lambda > \lambdabar$, the optimal solutions of \cref{pr:P0} and \cref{pr:P2l} coincide. 
With some differences, similar results were proven in \cite{bertsimas2017trimmed,gotoh2018dc}. 
In \cite{gotoh2018dc} the guarantee required the set of optimal solutions of \cref{pr:P0} to be bounded, and the threshold depends on this bound. Hence, in some cases, it may be larger than $\lambdabar$. The authors of  \cite{bertsimas2017trimmed} proved a result similar to \cref{thm:large_lambda_p2l}, with the same threshold $\lambdabar$, but for a trimmed-lasso regularized objective with an additional $\ell_1$\nbdash penalty term.

To derive an analogous result for problem \cref{pr:P1l}, we introduce the following two thresholds
\eq[eq:def_lambda_b]{ \ensuremath{
    \ensuremath{\lambdasmall \eqdef \frac{\sigma_n \br{A}}{\sqrt{d-k}}, \qquad
     \lambdalarge \eqdef \max_{i=1,\ldots,d} \norm{\mb{a}_i}_2,
     }}}
where $\sigma_n \br{A}$ is the $n$\nbdash th singular value of $A$.
\begin{theorem}\label{thm:large_lambda_p1l}
    Suppose that $\lambda > \lambdalarge$. Then any local minimum of \cref{pr:P1l} is $k$\nbdash sparse.
\end{theorem}
\begin{theorem}\label{thm:small_lambda_p1l}
    Suppose that the $n\times d$ matrix $A$, with $d\geq n$ is of full rank, so that $\lambdasmall>0$. Assume that $0 < \lambda < \lambdasmall$. Then any 
    local minimum $\xstar$ of \cref{pr:P1l} satisfies $A \xstar = \y$.
\end{theorem}
\cref{thm:large_lambda_p1l} implies that for $\lambda > \lambdalarge$, the optimal solutions of the trimmed lasso objective \cref{pr:P1l} coincide with those
of \cref{pr:P0}. In contrast, by \cref{thm:small_lambda_p1l}, for $\lambda<\lambdasmall$ the trimmed lasso penalty has little effect, as the minimizers of the objective lie in the subspace of zero residual. 

\Cref{thm:large_lambda_p2l} seems to suggest that
to solve \cref{pr:P0}, one should optimize
the trimmed lasso objective $\Ftwol$
with $\lambda > \lambdabar$. However,
since $\tauk$ is nonconvex,
larger values of $\lambda$ increase
the dominance of the nonconvex
part of the objective, making it more difficult to optimize.
The following theorem sheds further light on this difficulty,
showing that for $\lambda \geq \lambdabar$,
the optimization landscape is riddled with poor local minima.
\begin{theorem} \label{thm:bad_local_minma_everywhere}
 Let $\Lambda \subset \brs{d}$ be any index set of size $k$, and let $\xtilde$ be the minimizer of $\norm{A\x-\y}_2$ over all vectors $\x$ 
 with $\mbox{supp}(\x)\subseteq\Lambda$.
If $\|\xtilde\|_0=k$,
then it is a local minimum of $\Ftwol$ for any $\lambda \geq \lambdabar$. The same claim holds for $\Fonel$ with $\lambda \geq \lambdalarge$. 
\end{theorem}
Finally, regarding the power\nbdash 1 case, our algorithm to minimize $\Fonelg$, presented in \cref{sec:gsm_penalty}, is guaranteed under mild assumptions to output the same solution for all $\lambda < \lambdasmall$; see \cref{sec:optimizing_fonel_with_small_lambda}.

In light of the above discussion and theorems, our strategy
to solve \cref{pr:P0} is to minimize $\Ftwol$ or $\Fonel$ for several values of $\lambda <\lambdabar$ or $\lambda \in \brs{\lambdasmall,\lambdalarge}$, respectively. 
%
As we discuss theoretically in the next subsection, an appealing property of the trimmed lasso is that even at low values of $\lambda$, where the solutions are not necessarily $k$-sparse, they may still be close to the optimal solution of \cref{pr:P0}.

\subsection{Sparse recovery guarantees}
\label{sec:sparse_recovery_guarantees}
%
A fundamental and well-studied problem is the ability of various methods to recover a sparse vector $\x_0$ from few linear and potentially noisy measurements; 
see \cite{Foucart_Rauhut,elad2010sparsebook,eldar2012compressed} and references therein. 
Here we present sparse recovery guarantees for
the trimmed lasso relaxations \cref{pr:P2l} and \cref{pr:P1l}.
Specifically, let $\x_0 \in \R[d]$ be approximately $k$-sparse, 
in the sense that $\tauk(\x_0) \ll \norm{\x_0}_1$. Given $A$ and $k$, our goal is to estimate $\x_0$ from
 \eq[eq:y_generated_by_x0]{\ensuremath{
    \y = A \x_0 + \e,
}}
where $\e$ is an unknown measurement error.
In our analysis, we consider an adversarial model, 
whereby do not make any probabilistic assumptions on $\e$. Our bounds thus depend on $\|\e\|_2$. 

Without further assumptions on $A$, this problem is ill posed. 
As is well known, even in the absence of noise, a necessary condition for unique recovery of a $k$-sparse $\x_0$ is that any subset of $2k$ columns of $A$ are linearly independent 
\cite[Theorem 2.13]{Foucart_Rauhut}. 
Similarly to previous works, 
we require that any $2k$ columns of $A$ are sufficiently far from being linearly dependent. Specifically, 
we assume that there exists a constant $\ripmin > 0$ such that for any $\x$ with $\norm{\x}_0 \leq 2k$,
\begin{equation}
\norm{A \x}_2  \geq \ripmin \norm{\x}_1.
\label{eq:def_RIP_condition}
\end{equation}
Equation~\cref{eq:def_RIP_condition} is a one-sided variant of the \emph{restricted isometry property} (RIP), commonly used to derive sparse recovery guarantees. 
RIP was 	 introduced in \cite{candes2005decoding} and further studied in 
\cite{blanchard2011compressed, candes2008restricted, candes2006quantitative, 
candes2006robust, candes2006stable, foucart2009sparsest, cohen2009compressed}. 
Variants with norms other than $\ell_2$ \cite{allen2016restricted, berinde2008combining, chartrand2008restricted, dirksen2016gap} were used to provide sharper guarantees. 
As shown below, \cref{eq:def_RIP_condition} is the natural combination of powers for studying the trimmed lasso in conjunction with an $\ell_2$\nbdash residual $\norm{A\x-\y}_2$. We remark that calculating $\ripmin$, or even approximating it, is NP-hard \cite{tillmann2014computational}. However, $\ripmin$ can be bounded from below by the {cumulative coherence}, also known as the {Babel function} \cite{tropp2004greed}.

We now present our main theoretical result, which links success in optimizing problems~\cref{pr:P2l} or~\cref{pr:P1l} at a given $\lambda>0$, with 
accurate estimation of $\x_0$. 
By \emph{success} we mean that the solution $\xhat$, computed by some algorithm, 
satisfies that $\Ftwol \br{\xhat} \leq \Ftwol \br{\projk\br{\x_0}}$. 
The following theorem shows that under this condition, $\projk(\hat \x)$ is close to $\x_0$. 

\begin{theorem} \label{thm:sparse_reconstruction_p2}
Let $\y=A\x_0+\mb{e}$ and let $\ripmin$, $\lambdalarge$ be the constants defined in \cref{eq:def_RIP_condition,eq:def_lambda_b}.
        Let $\lambda > 0$ and suppose that $\xhat \in \R[d]$ satisfies  $\Ftwol \br{\xhat} \leq \Ftwol \br{\projk\br{\x_0}}$. Then, 
        \begin{enumerate}
                \item \label{item:thm_sparse_reconstruction1_p2_general_case} The projected vector $\projk \br{\xhat}$ is close to $\x_0$ in $\ell_1$ norm, 
                \eq[eq:recovery_error_bound_1]{\ensuremath{
                \norm{\projk \br{\xhat} - \x_0}_1 \leq \tauk \br{\x_0} + \tfrac{2}{\ripmin}\br{\norm{\e}_2 +  \lambdalarge \tauk\br{\x_0}} + \tfrac{1}{2 \lambda\ripmin}\br{\norm{\e}_2 +  \lambdalarge\tauk\br{\x_0}}^2.
                }}
                \item \label{item:thm_sparse_reconstruction1_p2_sparse_case}If $\xhat$ itself is $k$\nbdash sparse, then the following tighter bound holds,
                \eq[eq:recovery_error_bound_2]{\ensuremath{
                                \norm{\xhat - \x_0}_1 \leq 
                                       \tfrac{2}{\ripmin} \norm{\e}_2 + (1+2\tfrac{\lambdalarge}{\ripmin})\tauk \br{\x_0}.              
                }}
        \end{enumerate}
\end{theorem}
The next theorem provides an analogous recovery guarantee for the power-1 objective \cref{pr:P1l}.

\begin{theorem} \label{thm:sparse_reconstruction_p1}
Let $\y=A\x_0+\mb{e}$, and let $\ripmin$, $\lambdalarge$ be the constants defined in \cref{eq:def_RIP_condition,eq:def_lambda_b}.
       Suppose that $\xhat \in \R[d]$ satisfies $\Fonel \br{\xhat} \leq \Fonel \br{\projk\br{\x_0}}$ for some $\lambda > 0$. Then, 
        \begin{enumerate}
        \item \label{item:thm_sparse_reconstruction1_p1_general_case} The projected vector $\projk \br{\xhat}$ is close to $\x_0$ in $\ell_1$ norm,
             \begin{equation}\label{eq:recovery_error_bound_12}
                     \norm{\projk \br{\xhat} - \x_0}_1 \leq
                     \tauk \br{\x_0} +
                     \tfrac{1}{\ripmin} \br{1+\max\brc{1,\tfrac{\lambdalarge}{\lambda}}} \bigbr{\norm{\e}_2 +  \lambdalarge \tauk \br{\x_0}}.
               \end{equation}
                \item \label{item:thm_sparse_reconstruction1_p1_sparse_case}If $\xhat$ itself is $k$\nbdash sparse, then the tighter bound  \cref{eq:recovery_error_bound_2} holds.  
        \end{enumerate}
\end{theorem}
\cref{thm:sparse_reconstruction_p2,thm:sparse_reconstruction_p1} do not assume that an algorithm found the global minimum of the respective objectives, only that it was able to find an approximate
solution, with an objective value close to the optimal one.
Thus, \cref{thm:sparse_reconstruction_p2,thm:sparse_reconstruction_p1} imply that successful optimization of \cref{pr:P2l} or~\cref{pr:P1l} yields a successful recovery of $\x_0$, which is stable with respect to the measurement error $\e$ and the deviation of $\x_0$ from being exactly $k$-sparse.

Another implication of \cref{thm:sparse_reconstruction_p2,thm:sparse_reconstruction_p1} is that 
$\x_0$ can be accurately estimated by solving \cref{pr:P2l} or~\cref{pr:P1l} even at low values of $\lambda$, where the global minimizers are not necessarily $k$\nbdash sparse. While the above guarantees improve as $\lambda$ increases, as discussed in the previous subsection, optimizing these problems at low values of $\lambda$ is potentially an easier task. Indeed, in our 
simulations described in \cref{sec:numerical_experiments}, the best 
solutions, 
with smallest $\norm{A\projk(\hat\x)-\y}_2$, 
were often obtained at low values of $\lambda$, where $\xhat$ itself was not $k$-sparse. 

Recently, \cite{yun2019trimming} also analyzed the trimmed lasso penalized objective and derived support recovery guarantees.
There are some key differences between our analysis and theirs. They assumed that the rows of $A$ are i.i.d. samples from a $d$-dimensional sub-Gaussian distribution, that the true $\x_0$ is exactly $k$-sparse, and that the noise  $\e$ has i.i.d. entries.   

We now compare the guarantee of \cref{thm:sparse_reconstruction_p2,thm:sparse_reconstruction_p1} to those for basis pursuit. Suppose that $A$ has $\ell_2$-normalized columns and satisfies the RIP of order $2k$ with  constant $\delta_{2k} < 1$,
\begin{equation}
    \br{1 - \delta_{2k}} \norm{\x}^2_2 \leq \norm{A\x}_2^2 \leq \br{1 + \delta_{2k}} \norm{\x}^2_2
    \qquad
    \forall \x \text{ s.t. } \norm{\x}_0 \leq 2k.
            \label{eq:RIP_delta}
\end{equation}
Let $\xhat_{\textup{BP}}$ be the solution
of the following quadratically constrained Basis Pursuit problem, 
\begin{equation*}
    \xhat_{\textup{BP}} = \argmin_{\x}\ \norm{\x}_1 \quad \text{s.t.} \quad \norm{A\x-\y}_2 \leq \norm{\e}_2
\end{equation*}
By \cite[Theorem 6.12]{Foucart_Rauhut}, if $\delta_{2k} < 4/\sqrt{41} \approx 0.6246$,
then the following guarantee holds, 
with $c_1, c_2$ constants that depend only on $\delta_{2k}$,
\begin{equation}
   \norm{\xhat_{\textup{BP}} - \x_0}_1 \leq  c_1 \sqrt{k} \norm{\e}_2 + c_2 \tauk \br{\x_0}.
    \label{eq:guarantee_BP}
\end{equation}
A lower bound on $\delta_{2k}$ is necessary for
such a recovery guarantee to hold. Indeed, as
shown in \cite{davies2009restricted}, there exist matrices with $\delta_{2k}$ arbitrarily close to $1/\sqrt{2} \approx 0.7071$, for which BP fails. Recovery guarantees of other polynomial-time methods such as OMP and iterative thresholding require similar, and often stricter, conditions on the RIP
constant; see \cite[Chapter 6]{Foucart_Rauhut}.  

Let us compare \cref{eq:guarantee_BP} to our guarantees. It is easy to show that \cref{eq:RIP_delta} implies that 
 $\alpha_{2k}$, defined in \cref{eq:def_RIP_condition}, satisfies $\alpha_{2k} \geq {\sqrt{1-\delta_{2k}}} / {\sqrt{2k}}$. Consider a vector $\xhat$ that is a global minimizer of either \cref{pr:P2l} or \cref{pr:P1l}. To simplify the analysis, suppose that $\lambda$ is high enough so that $\xhat$ is $k$-sparse. Then both  \cref{thm:sparse_reconstruction_p2,thm:sparse_reconstruction_p1} yield
\begin{equation}
   \norm{\xhat - \x_0}_1 \leq 2 \tfrac{\sqrt{2k}}{\sqrt{1-\delta_{2k}}} \norm{\e}_2 + \br{1+2\sqrt{2k} \sqrt{\tfrac{1+\delta_{2k}}{1-\delta_{2k}}}} \tauk \br{\x_0}.    
\end{equation}
 While in our bound the coefficient of $\tauk \br{\x_0}$ is larger by a multiplicative factor of $\mathcal{O}\br{\sqrt{k}}$, the key point is that our guarantee holds for {\em any} $\delta_{2k} < 1$. We remark that the condition $\delta_{2k}<1$ is necessary for successful recovery of $\x_0$ by {\em any} algorithm (see e.g. \cite[Lemma 3.1]{cohen2009compressed}).

\subsection{Solving the best subset selection problem by the trimmed lasso penalty}
\label{sec:solving_P0}
%
The above theoretical results motivate solving \cref{pr:P0}
by optimizing the trimmed-lasso penalized $\Ftwol$ or $\Fonel$. 
For $\Ftwol$ we propose the following approach: 
Choose $0 < \lambda_1 < \lambda_2 < \ldots \leq \lambda_T = \lambdabar$. For each $\lambda_i$, optimize $\Ftwol$ with $\lambda=\lambda_i$. Next, project each solution $\xhat_{\lambda_i}$ to its nearest $k$-sparse vector, 
and solve a least-squares problem on its support. 
Output the vector with smallest residual $\|A\x-\y\|_2$. 
For the power-1 objective $\Fonel$,
we propose a similar approach with $\lambda \in[\lambdasmall,\lambdalarge]$.

To carry out this approach, one needs a practical algorithm to optimize trimmed lasso penalized objectives. In the next section we present a new approach to do so. As we demonstrate in \cref{sec:numerical_experiments}, our method often finds much better solutions 
than those found by previously proposed methods such as DC programming \cite{gotoh2018dc} and ADMM \cite{bertsimas2017trimmed}.
This, in turn, yields state-of-the-art results in solving the best subset selection problem \cref{pr:P0}.

\section{The generalized soft-min penalty}
\label{sec:gsm_penalty}
%
Optimizing trimmed-lasso penalized objectives is difficult for several reasons. First, $\tauk\ofx$ is nonconvex. 
Second, updating a solution $\x$ along the negative gradient $-\nabla\tauk\ofx$ does not change its top\nbdash $k$ largest magnitude coordinates, since $\nabla\tauk\ofx$ is zero there, and shrinks the remaining $d-k$ coordinates.
While this property is desirable if the current solution has the correct support, it may be detrimental otherwise.




To overcome this difficulty, we propose a soft surrogate for the trimmed lasso. To motivate it,
consider the following equivalent definition of $\tauk$  \cite{bertsimas2017trimmed}, 
\begin{equation}
\label{eq:def_tauk_b}
\tauk \ofx = \underset{ \abs{\Lambda} = d-k } {\min} \  \sum_{i\in\Lambda}\ \abs{x_i},
\end{equation}
where $\Lambda\subset [d]$ is a set of $d-k$ distinct indices. Equation~\cref{eq:def_tauk_b} shows that $\tauk\ofx$ can be expressed as a \emph{hard} assignment: Out of all $\binom{d}{k}$ subsets of size $d-k$, it takes the one with the smallest $\ell_1$ residual. In $k$-means data clustering, \cite{zhang1999k} relaxed a hard assignment of each point to its nearest cluster center, to a \emph{soft} assignment, with weights depending on the distances to all cluster centers. This led to significantly improved clustering results  \cite{hamerly2002alternatives}. Similarly, in multi-class classification, the one-hot vector is often replaced by the soft-max function \cite{bridle1989probabilistic}.

Inspired by these works, we propose the following surrogate, which we call the \emph{generalized soft-min} (GSM) penalty $\taukg \ofx$. In addition to $k$, it depends on a softness parameter $\gamma \in \brs{0,\infty}$. 
For $\gamma \in \br{0,\infty}$, it is defined by
\eq[eq:def_taukg_gamma_finite]{\ensuremath{
\taukg\ofx \eqdef -\frac{1}{\gamma} \log \br{ \frac{1}{\binom{d}{d-k}} \sum_{{ \abs{\Lambda} = d-k }} \exp \br{ -\gamma \sum_{i\in\Lambda} \abs{x_i} } }.
}}
The sum in \cref{eq:def_taukg_gamma_finite} is over all $\binom{d}{k}$ sets $\Lambda \subseteq \brs{d}$ of $d-k$ distinct indices. 
We define $\taukzero\ofx$ and $\taukinfty\ofx$ by the respective limits,
described in the following \lcnamecref{thm:tau_limits_gamma}.
\begin{lemma}\label{thm:tau_limits_gamma}
        For any $\x \in \mathbb{R}^d$, the function $\taukg\ofx$ is monotone-decreasing w.r.t. $\gamma$, and
\begin{equation}
        \begin{split}
        \underset {\gamma \rightarrow 0 } {\lim} \ \taukg\ofx = \frac{d-k}{d} \|\x\|_1
        \qquad
        \underset {\gamma \rightarrow \infty } {\lim} \ \taukg\ofx = \tauk\ofx,
        \label{eq:tau_limit_gamma_inf}
        \end{split}
\end{equation} 
where the latter limit is uniform over $\x \in \R[d]$. Specifically, for all $\x \in \mathbb{R}^d$ and $\gamma \in \br{0,\infty}$,
\eq[eq:tau_approximation_bound]{\ensuremath{
    \tauk \br{\x} \leq \taukg \ofx \leq \tauk \ofx + \frac{1}{\gamma} \log {\binom{d}{k}}.
    }}
\end{lemma}
The proof of this and the following lemmas are in  \cref{appendix:proofs_properties_of_gsm_penalty}.

\subsection{Soft surrogate problem}
\label{sec:soft_surrogate_problem}
%
We propose the following problem as a surrogate for \cref{pr:P2l},
\eq[pr:P2lg]{\ensuremath{
\underset{\x}{\min} \ \Ftwolg \ofx \eqdef 
        \frac{1}{2}\norm{A \x- \y}_2^2 + \lambda \taukg \ofx,}}
with $\gamma \in \brs{0,\infty}$. By \cref{thm:tau_limits_gamma}, 
at $\gamma=0$, \cref{pr:P2lg} coincides with the convex lasso problem
\eq[pr:P2lzero]{\ensuremath{
    \underset{\x}{\min} \ \Ftwolzero \ofx = \frac{1}{2} \norm{A \x- \y}_2^2 + \lambda \tfrac{d-k}{d} \norm{\x}_1.}}
As $\gamma \nearrow \infty$, $\Ftwolg \ofx$ decreases smoothly towards the trimmed-lasso penalized $\Ftwol \ofx$ of \cref{pr:P2l} and coincides with it at $\gamma = \infty$. Similarly, we propose the following soft surrogate for \cref{pr:P1l},
\eq[pr:P1lg]{\ensuremath{
    \underset{\x}{\min} \ \Fonelg \ofx \eqdef 
    \norm{A \x- \y}_2 + \lambda \taukg \ofx.
}}

There are two challenges in solving the surrogate problems \cref{pr:P2lg,pr:P1lg}.
The first is to calculate $\taukg\ofx$ and its gradients. 
The second is to minimize the nonconvex objective $\Ftwolg$ or $\Fonelg$. 
Here we address the second challenge, whereas the first one
is addressed in \cref{sec:algorithm_gsm}.

\subsection{Auxiliary weight vector}
\label{sec:w_definition}
%
To optimize  \cref{pr:P2lg,pr:P1lg}, we introduce the following vector function
$\wfunc : \R[d] \rightarrow \R[d]$, closely related to the gradient of $\taukg\ofx$. For $0 \leq \gamma < \infty$, 
\eq[eq:def_wkg]{\ensuremath{
\wfuncidx{i} \ofx \eqdef 
\frac{ \summ[\abs{\Lambda} = d-k,\ i \in \Lambda] \exp \br{ -\gamma \sum_{j\in\Lambda}         \abs{x_j} } }
{ \summ[ \abs{\Lambda} = d-k] \exp \br{ -\gamma \sum_{j\in\Lambda} \abs{x_j} } },\quad  i=1,\ldots,d.}}
The sum in the denominator is over all sets $\Lambda \subset \left[d\right]$ of size $d-k$, while in the numerator the sum is restricted to those sets $\Lambda$ that contain the index $i$. It is easy to show that at $\gamma=0$,
\eq[eq:w_at_gamma_zero]{\ensuremath{
    \wfunczeroidx{i}\ofx = \frac{d-k}{d},\quad \ i=1,\ldots,d.
}}
We define $\wfunckinfty\ofx$ as the limit of $\wfunc\ofx$ as $\gamma \app \infty$, characterized in the following lemma.
\begin{lemma}\label{thm:w_limits_gamma}
    Let $\x \in \mathbb{R}^d$ whose $k$ largest-magnitude entries are at a uniquely determined index-set $\Lambda$. That is, $\xord{k} > \xord{k+1}$. Then, as $\gamma \to \infty$, $\wfunc \ofx$ approaches zero at $\Lambda$ and $1$ elsewhere. For a general  $\x$, let $\Lambda_a = \setst{i}{\abs{x_i} < \xord{k}}$ and $\Lambda_b = \setst{i}{\abs{x_i} = \xord{k}}$. Then
    \eq[eq:w_at_gamma_infty]{ \ensuremath{  
        \limit[\gamma \app \infty] \ \wfuncidx{i}\ofx = \threecase
        {1}{i \in \Lambda_a}
        {\frac{d-k-\abs{\Lambda_a}}{\abs{\Lambda_b} }}{i \in \Lambda_b}
        {0}{\text{otherwise.}}
        }}
\end{lemma}

As described below, $\wfunc \ofx$ will be used as a weight vector in a reweighting scheme. The next lemma
is thus of practical importance, as it shows that the entries of $\wfunc \ofx$ remain bounded for all values of $\gamma$, with their sum a constant independent
of both $\x$ and $\gamma$. 
\begin{lemma}\label{thm:sum_w}
        For any $\x \in \R[d]$ and $\gamma \in \brs{0,\infty}$,
        \eq[eq:w_constant_sum]{\ensuremath{
        \forall i \in \brs{d} \ \wfuncidx{i}\ofx \in[0,1]
        \quad \mbox{and} \quad
        \sum_{i=1}^d\wfuncidx{i}\ofx = d-k
        .} }
\end{lemma}

\subsection{Solving the soft surrogate problem: A majorization-minimization approach}
\label{sec:mm}
%
One possible way to optimize \cref{pr:P2lg} or~\cref{pr:P1lg} is by gradient descent. First, this may require careful tuning of the step size. Second, as gradient descent makes local updates, it may take many iterations to converge. Here we propose an alternative approach, based on \emph{Majorization-Minimization} (MM). The principle is to replace a hard-to-minimize objective $f  \ofx$ by a sequence of easy-to-minimize functions that upper-bound it \cite{ortega1970iterative}\cite{hunter2004tutorial}. Specifically,
in MM one constructs a bivariate function 
$g :\mathbb{R}^d\times\mathbb{R}^d\rightarrow\mathbb{R}$, known as
a {\em majorizer} of $f$, 
that satisfies
\begin{equation}
        \label{eq:majorizer}
        g \br{\x,\xtilde} \geq f  \ofx
        \quad \mbox{and} \quad
        g  \br{\x, \x} = f  \ofx
        \quad
        \forall\x,\xtilde\in \mathbb{R}^d.
\end{equation}
Given the function $g$,  
minimization of $f $ is done iteratively. Starting at an initial point $\x^0$, 
the iterate $\x^{t}$ at iteration $t$ is given by 
$\x^{t} \eqdef 
\argmin_{\x}
\ g  \br{\x, \x^{t-1}}.$
For this approach to be practical, $g  \br{\x,\xtilde}$ must be easy to minimize with respect to $\x$ for any fixed $\xtilde$.
This scheme guarantees monotone decrease of the objective, since for any $t \geq 1$,
\eq{\ensuremath{
f  \br{ \x^{t} } \leq
g  \br{ \x^{t}, \x^{t-1} } \leq 
g  \br{ \x^{t-1}, \x^{t-1} } =
f  \br{ \x^{t-1} }.}}

To apply the MM framework in our setting, we replace the penalty $\taukg \ofx$ by the function
\begin{equation}
\taumajkg \br{\x,\xtilde} \eqdef \taukg \br{\xtilde} +
\inprod{ \wfunc \br{\xtilde}} {\abs{\x}-\abs{\xtilde} }.
\label{eq:def_taumaj}
\end{equation}

\begin{lemma}
        \label{thm:Ftwolg_majorizer}
        For all $\gamma \in \brs{0,\infty}$, $\taumajkg \br{\x,\xtilde}$ is a majorizer of $\taukg \ofx$. 
\end{lemma}

It follows from \cref{thm:Ftwolg_majorizer} that the following two functions
\begin{equation}\label{eq:def_G}
\begin{split}
\Gtwolg \br{\x,\xtilde} \eqdef & \frac{1}{2} \norm{A \x- \y}^2_2 + \lambda \taumajkg \br{\x,\xtilde},
\\ 
\Gonelg \br{\x,\xtilde} \eqdef & 
\norm{A \x- \y}_2 + \lambda \taumajkg \br{\x,\xtilde},
\end{split}
\end{equation}
are majorizers of the corresponding objectives $\Ftwolg \ofx$, $\Fonelg\ofx$ of problems \cref{pr:P2lg} and \cref{pr:P1lg}.

Before describing our MM scheme to minimize $\Ftwolg \ofx$,  
we make two important observations. First, it follows from \cref{eq:w_at_gamma_zero} that at $\gamma=0$, the majorizer $\Gtwolzero\br{\x,\xtilde}$ coincides with the convex objective $\Ftwolzero\br{\x}$, and is independent of $\xtilde$.
Second, for any $\gamma \in \brs{0,\infty}$, minimizing $\Gtwolg \br{\x,\xtilde}$ with respect to $\x$ is equivalent to the convex weighted $\ell_1$ problem 
\eq[pr:P2lw]{\ensuremath{
\underset{\x}{\min}\ \Ftwolw\ofx \eqdef \frac{1}{2}\norm{A\x -\y}_2^2 + \lambda \inprod{\w} {\abs{\x}}}, 
}
with $\w = \wfunc \br{\xtilde}$. Hence, our MM approach leads to an \emph{iterative reweighting} scheme, where the weights at iteration $t$ are given by the vector $\wfunc \br{\x^{t-1}}$. This is similar to the IRL1 and IRLS methods, with an important difference: To avoid ill conditioning, weights in IRL1 and IRLS must be regularized, otherwise they may tend to infinity. Indeed, several regularization schemes were proposed 
\cite{chartrand2008iteratively,daubechies2010iteratively,foucart2009sparsest,lai2013improved}. 
By \cref{thm:sum_w}, our weights $\wfunc \ofx$ are bounded and have a constant sum.  Therefore our method does not require weight regularization.

\begin{algorithm}[t]
\caption{Solve problem~\cref{pr:P2lg} by Majorization-Minimization}
\label{alg:mm}
\begin{algorithmic}[1]
\Input $ $
    \\$A \in \R^{n\times d}$,\quad $\y \in \R[n]$,\quad $0<k<d$,\quad $\lambda >0$,\quad $\gamma \in \brs{0,\infty}$
    \\Initialization $\x^0 \in \R[d]$ (required only if $\gamma > 0$)
\Output Estimated solution $\xstar$ of problem~$\text{\cref{pr:P2lg}}$
\If{$\gamma=0$}
\State Return $\xstar \assign \argmin_{\x} \ \frac{1}{2}\norm{A\x - \y}_2^2 + \lambda \frac{d-k}{d} \norm{\x}_1$
\Else
\For{$t=1,2,\ldots$}
\State Calculate $\taukg \br{\x^{t-1}}$ and the weight vector $\w^t= \wfunc \br{\x^{t-1}}$
\State Solve majorization problem $\mbox{\cref{pr:P2lw}}$: $\x^{t} \assign \argmin_{\x} \ \frac{1}{2}\norm{A\x - \y}_2^2 + \lambda \inprod{\w^t} {\abs{\x}}$
\State If converged, break
\EndFor
\EndIf
\State\Return $\xstar \assign \x^T$, where $T$ is the last iteration. 
\end{algorithmic}
\end{algorithm}

Our MM scheme for solving \cref{pr:P2lg} is outlined in \cref{alg:mm}. At iteration $t$
we solve the convex problem~\cref{pr:P2lw} with weight vector $\w^t=\w_{k,\gamma}\br{\x^{t-1}}$.
Problem~\cref{pr:P2lw} can be cast as a quadratic program (QP) and solved by standard convex optimization software, 
or by dedicated $\ell_1$ solvers such as \cite{beck2009fast,zhang2011yall1,asif2013fast}.  \Cref{alg:mm} can be adapted to the power\nbdash 1 variant \cref{pr:P1lg} as follows: Instead of
problem~\cref{pr:P2lw}, iteratively solve 
\begin{equation}\label{pr:P1lw}
\underset{\x}{\min}\ \Fonelw\ofx \eqdef \norm{A\x -\y}_2 + \lambda \inprod{\w} 
{\abs{\x}}.
\end{equation}
Problem~\cref{pr:P1lw} can be cast as a second order cone program, which may 
also be solved by standard convex optimization software. In our implementation 
we use the fast iterative shrinkage-thresholding algorithm (FISTA) 
\cite{beck2009fast} for problem~\cref{pr:P2lw} and 
 MOSEK \cite{mosekcite} for problem~\cref{pr:P1lw}.


\subsection{Solving the trimmed lasso penalized problem: A homotopy approach}
\label{sec:homotopy_optimization}
%
To seek solutions of problem~\cref{pr:P2l}, we propose a \emph{homotopy} scheme \cite{watson1989modern}, whereby we optimize a sequence of surrogate problems \cref{pr:P2lg} while tracing the path of solutions. As outlined in \cref{alg:homotopy}, we start at $\gamma_0 = 0$ and find a global minimizer $\x_0^*$ of the convex problem~\cref{pr:P2lzero}. 
Next, we iteratively solve \cref{pr:P2lg} for an increasing sequence of values of $\gamma$. At iteration $r$, with $\gamma=\gamma_r$,  we find a local minimizer of \cref{pr:P2lg} by \cref{alg:mm}, initialized at the previous iterate $\x_{r-1}^*$. 
For details on the stopping criteria and update rule for $\gamma$, see \cref{appendix:computational_details}.

\begin{algorithm}[t]
\caption{Solve problem~\cref{pr:P2l} by homotopy optimization}
\label{alg:homotopy}
\begin{algorithmic}[1]
\Input $A \in \R^{n\times d}$,\quad $\y \in \R[n]$,\quad $\lambda >0$
\Output Estimated solution $\xhat$ of problem~$\text{\cref{pr:P2l}}$
\State Set $\gamma_0 = 0$. Calculate solution $\x_{0}^*$ of $\text{\cref{pr:P2lzero}}$ by $\text{\cref{alg:mm}}$
\For{$r=1,2,\ldots$}
   \State Increase $\gamma_{r-1}$ to new value $\gamma_r$
   \State\label{line:homotopy_solve}Solve \cref{pr:P2lg} with $\gamma=\gamma_r$ by \cref{alg:mm} initialized at $\x_{r-1}^*$. Let $\x_r^*$ be the output.
   \State{If converged, break}
\EndFor
\State\Return $\xhat = \x_{R}^*$, where $R$ is the last iteration.
\end{algorithmic}
\end{algorithm}

\subsection{Convergence analysis}
\label{sec:convergence_analysis}
%
We study the convergence of Algorithms~\ref{alg:mm} and \ref{alg:homotopy}, and properties of their output solutions. For brevity, we focus on the power\nbdash 2 case. However, the analysis below applies also to the power\nbdash 1 case. The lemmas below are proven in  \cref{appendix:proofs_convergence_analysis}.
We start with the MM scheme of \cref{alg:mm}. Recall that the objective values $\Ftwolg \br{\x^t}$ in \cref{pr:P2lg} decrease monotonically. However, this does not guarantee that the iterates $\x^t$ converge to a local minimum. Indeed, there exist examples where MM converges to saddle points \cite[Chapter~3]{mclachlan2008algorithm}. We thus prove a slightly weaker result --- that the iterates approach a level-set of $\Ftwolg\ofx$ that consists of \emph{stationary points}. 

While $\Ftwolg\ofx$ is not everywhere differentiable, its directional derivatives exist everywhere (see \cref{appendix:proofs_convergence_analysis}). To formulate our claims, we introduce the following definition, which extends the notion of stationary points to nonsmooth functions \cite{razaviyayn2013unified}.
\begin{definition*}
    A point $\x \in \R[d]$ is a \emph{stationary point} of  $f : \R[d] \rightarrow \R$ if for all $\vvec \in \R[d]$, the directional derivative $\nabla_{\vvec} f  \br{\x}$ exists and is nonnegative.
\end{definition*}
The following \lcnamecref{thm:alg1_convergence_gamma_finite} describes the convergence of \cref{alg:mm} at finite values of $\gamma$. 

\begin{lemma}\label{thm:alg1_convergence_gamma_finite}
Suppose that any $k$ columns of $A$ are linearly independent. Then, for any $0 \leq \gamma < \infty$, 
the iterates $\brc{\x^t}_{t=0}^{\infty}$ of \cref{alg:mm} approach the set $\xstatset \subset \R[d]$ of all stationary points of $\Ftwolg$. Namely, $\lim_{t \rightarrow \infty} d \br{\x^t, \xstatset} = 0$, where $d \br{\x,S}$ is the Euclidean distance of a point $\x$ from the set $S$.                 
Furthermore,  any partial limit of $\x^t$ is a stationary point of $\Ftwolg$, and all partial limits have the same objective value.
\end{lemma}
Thus, under the assumptions of \cref{thm:alg1_convergence_gamma_finite}, the output of \cref{alg:mm} at $\gamma < \infty$ is guaranteed to be a stationary point of $\Ftwolg$. We remark that in practice the iterates invariably converge to a single limit point, which is a local minimum of $\Ftwolg$. The convergence result at $\gamma = \infty$ appears in \cref{appendix:proofs_convergence_analysis}. The result is slightly different, since $\taukinfty\ofx=\tauk\ofx$ is not everywhere differentiable as a function of $\abs{\x}$.

Next, we study the convergence of the homotopy scheme of \cref{alg:homotopy}. To this end, we first present a key property of the objective $\Ftwolinf=\Ftwol$
and introduce useful notation.
\begin{lemma}\label{thm:stationary_pt_is_localmin}
    Any stationary point of $\Ftwol$ is a local minimum.
\end{lemma}
\begin{definition*}
    A vector $\x \in \R[d]$ is called \emph{ambiguous} if 
    $\xord{k} = \xord{k+1}$.
\end{definition*}
The following lemma shows that the output $\xhat$ of \cref{alg:homotopy} is guaranteed to be either a local minimum of $\Ftwol$ or an ambiguous vector.
\begin{lemma}\label{thm:alg2_convergence}
Suppose that any $k$ columns of $A$ are linearly independent. 
Let $\brc{\gamma_r}_{r=0}^\infty$ and $\brc{\x^*_r}_{r=0}^\infty$ be the intermediate values of \cref{alg:homotopy}. Suppose that $\limit[r \app \infty]\gamma_r = \infty$. Then,  
\begin{enumerate}
\item\label{item:alg2_convergence_part1} For all $r \geq 0$, $\x^*_r$ is a stationary point of $\mathrm{F}_{\lambda,\gamma_r}$, and $\mathrm{F}_{\lambda,\gamma_{r+1}}\br{\x^*_{r+1}} \leq \mathrm{F}_{\lambda,\gamma_r}\br{\x^*_r}$.
    \item\label{item:alg2_convergence_part2} The iterates $\x^*_r$ approach $\xoptset \cup \xambset$, where $\xoptset$ is the set of local minima of $\Ftwol$ and $\xambset$ is the set of ambiguous vectors. Namely, 
    $\lim_{r \rightarrow \infty} d \br{\x^*_r, \xoptset \cup \xambset} = 0.$
    \item\label{item:alg2_convergence_part3} Any partial limit of $\brc{\x^*_r}_{r=0}^\infty$ is either a local minimum of $\Ftwol$ or an ambiguous vector. 
\end{enumerate}
\end{lemma}

In accordance with our theoretical analysis, \cref{alg:homotopy} may indeed output an ambiguous vector that is not a local minimum. Empirically, when this occurs, the output  is invariably $s$\nbdash sparse for some $s < k$. This phenomenon is related to the promotion of sparse solutions by the lasso penalty. As described in \cref{appendix:computational_details}, we thus augment \cref{alg:homotopy} by a greedy post-processing step which outputs a $k$-sparse vector if it has a lower objective value. 

In summary, 
starting from $\br{\x^*_0, 0}$,
\cref{alg:homotopy} computes a discrete path $\br{\x^*_r, \gamma_r}$ in $\R[d] \times \brs{0,\infty}$, along which the objective $\mathrm{F}_{\lambda,\gamma} \ofx$
decreases monotonically. 
Each point along the path is a stationary point (and practically a local minimum) of the corresponding $\Ftwolg$. Hence, the endpoint 
is the result of local optimizations over smoothed objectives, and is less prone to be a poor local minimum of the nonsmooth $\Ftwol$.
Empirically, \cref{alg:homotopy} obtains superior solutions of \cref{pr:P2l}, compared to methods that directly optimize $\Ftwol$.

\section{Algorithm to compute the GSM penalty}
\label{sec:algorithm_gsm}
We now address the challenge of calculating the GSM penalty $\taukg \ofx$ and the weight vector $\wfunc \ofx$, defined in \cref{eq:def_taukg_gamma_finite,eq:def_wkg}. As these functions involve sums of $\binom{d}{k}$ terms, their na\"{i}ve calculation is computationally intractable. Moreover, due to the exponent, at large values of $\gamma$ these terms may be extremely large or small, 
and thus suffer an arithmetic overflow or underflow, potentially leading to meaningless results. Here we present a recursive algorithm that calculates $\taukg \ofx$ and $\wfunc \ofx$ accurately in $\mathcal{O}(kd)$ operations. The lemmas below are proven in \cref{appendix:proofs_properties_of_gsm_aux}.

Recently and independently of our work, in the context of multi-class classification with deep neural networks, a relaxation similar to ours was proposed by \cite{berrada2018smooth} as a smooth approximation of the top-$k$ classification error. They also proposed a recursive $\mathcal{O}\br{kd}$-time algorithm to compute functions similar to $\taukg\ofx$ and $\wfunc\ofx$. However, as they discuss in \cite[page 20]{berrada2018smooth}, their recursive scheme may suffer from numerical instabilities. 
In addition, their focus was on small values of $k$,
and all their experiments were done with $k=5$.

\subsection{Auxiliary GSM functions}
\label{sec:auxiliary_gsm_functions}
%
We first introduce two auxiliary functions and present some of their properties. For $0 \leq k \leq d$ and $\gamma \in \brs{-\infty,\infty}$, define $\mukg:\R[d]\rightarrow \R$ by
\eq[eq:def_mukg]{\ensuremath{
\mukg\ofz \eqdef \fourcase
{\min_{\abs{\Lambda}=k} \summ[i \in \Lambda] z_i}
{\gamma = -\infty}
{\frac{k}{d}\summ[i=1][d] z_i}
{\gamma = 0}
{\max_{\abs{\Lambda}=k} \summ[i \in \Lambda] z_i}
{\gamma = \infty}
{\frac{1}{\gamma} \log \br{ \frac{1}{\binom{d}{k}} \sum_{\Lambda : { \abs{\Lambda} = k }} \exp \br{ \gamma \sum_{i\in\Lambda} {z_i} } }}
{\text{otherwise}.}
}}
Next, for $\gamma \in \br{-\infty,\infty}$ we define $\gsm: \R[d] \rightarrow \R[d]$  to be the gradient of $\mukg \ofz$ w.r.t. $\z$, 
\eq[eq:def_gsm]{\ensuremath{
\gsmi{i} \ofz \eqdef 
\frac{ \sum_{\abs{\Lambda} = k,\ i \in \Lambda} \exp \br{ \gamma \sum_{j\in\Lambda} {z_j} } }
{ \sum_{ \abs{\Lambda} = k } \exp \br{ \gamma \sum_{j\in\Lambda} {z_j} } },\quad\ i=1,\ldots,d.
}}
The function $\gsm \ofz$ is a generalization of the 
\emph{soft-max} function, used extensively in multi-class classification. At $\gamma=0$, \cref{eq:def_gsm} reduces to $\theta^i_{k,0}(\z)=k/d$, consistent with the case $\gamma=0$ in \cref{eq:def_mukg}. 
We define $\gsmargs{k}{\pm \infty}\ofz$ by the corresponding limits, as described in the following lemma.
\begin{lemma}\label{thm:theta_limits_gamma}
    Let ${z_\br{1} \geq z_\br{2} \geq \cdots \geq z_\br{d}}$
    be the entries of $\z\in\R^d$ sorted in decreasing order. Let
    $\Lambda_a,\Lambda_b$ be the following index-sets, 
    \begin{gather}    
    \label{thm:theta_limits_idxset_def1}
    \begin{aligned}
    &\mbox{For } \gamma \app \infty: \quad &&\Lambda_a = \setst{i \in \brs{d}}{z_i > z_\br{k}}, \quad &&\Lambda_b = \setst{i \in \brs{d}}{z_i = z_\br{k}},
    \\ &\mbox{For } \gamma \app -\infty: \quad &&\Lambda_a = \setst{i \in \brs{d}}{z_i < z_\br{d-k+1}}, \quad &&\Lambda_b = \setst{i \in \brs{d}}{z_i = z_\br{d-k+1}}.
    \end{aligned}
    \end{gather}
    Then for $i=1,\ldots,d$,      
    \eq[eq:gsm_at_gamma_infty]{\ensuremath{
    \underset {\gamma \rightarrow \pm \infty } {\lim} \ \gsmi{i}\ofz = \threecase{1}{i \in \Lambda_a}{\frac{k-\abs{\Lambda_a}}{\abs{\Lambda_b} }}{i \in \Lambda_b}{0}{\text{otherwise,}}}}
\end{lemma}

By their definition, the functions $\mukg$ and $\gsm$ satisfy the following relations
\eq[eq:gsm_identity_gammaneg]{\ensuremath{
    \mukg\ofz = -\muargs{k}{-\gamma} \br{-\z},
    \qquad \gsm\ofz = \gsmargs{k}{-\gamma} \br{-\z}. 
    }}
The following lemma describes several non-trivial identities involving $\mukg \ofz$ and $\gsm \ofz$.
\begin{lemma}\label{thm:gsm_identities}
For any $\z \in \R[d]$, $0 \leq k \leq d$ and $\gamma \in \brs{-\infty,\infty}$,
\eq[eq:gsm_identity_dcomplement]{\ensuremath{
    \muargs{k}{\gamma}\ofz + \muargs{d-k}{-\gamma}\ofz & =  \summ[i=1][d]z_i,
    \\
    \gsmiargs{k}{\gamma}{i}\ofz + \gsmiargs{d-k}{-\gamma}{i}\ofz & = 1,\quad i=1,\ldots,d.} }
\end{lemma}

Finally, the functions $\mukg \ofz$ and $\gsm \ofz$ are related to $\taukg \ofx$ and $\wfunc \ofx$ as follows,  
\eq[eq:penalty_by_aux]{\ensuremath{
    \taukg \ofx = \muargs{d-k}{-\gamma} \br{\abs{\x}},\qquad
    \wfunc \ofx = \gsmargs{d-k}{-\gamma} \br{\abs{\x}}.
}}

\subsection{Calculating the auxiliary GSM functions}
%
Given \cref{eq:penalty_by_aux}, we now present a method to calculate the auxiliary GSM functions $\mukg \ofz$ and $\gsm \ofz$. In light of \cref{eq:gsm_identity_gammaneg,eq:gsm_identity_dcomplement}, it suffices to consider $\gamma \in \br{0,\infty}$ and $1 \leq k \leq \frac{d}{2}$.
A seemingly promising approach is to use the recursive formula for $k \geq 1$,
\begin{equation}
	\label{eq:simple_gsm_recursion}
t_k^i = \br{s_{k-1} - t^i_{k-1}} \exp \br{\gamma z_i},\mbox{ for } i=1,\ldots,d,
	\quad \mbox{and \ \  }
	s_k = \frac{1}{k} \sum_{i=1}^d t_k^i, 
\end{equation}
with recursion base
$s_0 = 1$, $t_0^i = 0$ for $i=1,\ldots,d$. 
Then, calculate $\mukg \ofz$ and $\gsm \ofz$ by
\eq{\ensuremath{\mukg \ofz = \frac{1}{\gamma}
\log \Bigbr{\frac{1}{\binom{d}{k}} s_k}
    \quad \mbox{ and } \quad
    \gsmi{i} \ofz = \frac{t_k^i}{s_k},\quad i=1,\ldots,d.
    }}
However, intermediate values in this simple recursion tend to be extremely large or small.
The resulting arithmetic overflow and underflow, in turn, often lead to meaningless
results.     
 

To avoid arithmetic overflow, we perform calculations in a logarithmic representation. 
We present below a recursive procedure
to compute $\mukg \ofz$ and several auxiliary quantities in $\mathcal{O}\br{kd}$ operations. 
Next, we use these quantities and calculate $\gsm \ofz$ by a separate 
recursion, described in \cref{sec:calculating_theta}. 
The lemmas below
are proven in \cref{appendix:proofs_computing_gsm}.



\myparagraph{Calculating $\mukg \ofz$}
\label{sec:calculating_mu}
%
Recall that $z_{\br{i}}$ is the $i$th largest entry of $\z\in\R^d$. We augment this by defining $z_{\br{0}} = \infty$ and $z_{\br{i}} = -\infty$ for $i>d$.
Next, for $q \in \brc{0,1,\ldots}$ and $\gamma \in \br{0,\infty}$, define the function $s_{q,\gamma} : \R[d] \rightarrow \R$ by
\eq[eq:def_aqg]{\ensuremath{
    s_{q,\gamma} \ofz \eqdef \threecase{1}{q=0}{\sum_{\abs{\Lambda} = q } \exp \br{\gamma \br{\sum_{i\in\Lambda}z_i - \summ[i=1][q] z_{\br{i}}}}}
    {1 \leq q \leq d}{0}{\text{otherwise.}}}}
The term $\summ[i=1][q] z_{\br{i}}$ is subtracted to scale all the exponents such that the largest one equals 1, hence avoiding overflow. Note that $\mukg \ofz$ can be expressed in terms of $s_{k,\gamma} \ofz$ by
\eq[eq:expressing_mu_by_a]{\ensuremath{
\mukg \ofz = \frac{1}{\gamma} \log \br{\frac{1}{\binom{d}{k}} s_{k,\gamma} \ofz} + \summ[i=1][k] z_{\br{i}}.}}
We now formulate a recurrence relation for $s_{k,\gamma} \ofz$. To this end, we extend the definitions of $z_{\br{i}}$ and $s_{k,\gamma} \ofz$ to subvectors $\br{z_r,z_{r+1},\ldots,z_d}$. Define for $1 \leq r \leq d+1$ and $i,q \geq 0$,
\eq[eq:def_zbrr]{\ensuremath{
z_{\br{i}}^{\br{r}} \eqdef \threecase{\infty}{i = 0}{\text{The $i$-th largest entry of $\br{z_r,z_{r+1},\ldots,z_d}$}}{1 \leq i \leq d-r+1}{-\infty}{i > d-r+1}}}
and similarly, define $s_{q,\gamma}^{\br{r}} \ofz$ as the function $s_{q,\gamma} \br{\argdot}$ applied to the subvector $\br{z_r,\ldots,z_d}$,
\eq[eq:def_aqgr]{\ensuremath{
s_{q,\gamma}^{\br{r}} \ofz \eqdef \threecase{1}{q=0}{\sum_{\Lambda \subseteq \brc{r,\ldots,d},\ \abs{\Lambda} = q } \exp \br{\gamma \br{\sum_{i\in\Lambda}z_i - \sum_{i=1,\ldots,q}z_{\br{i}}^{\br{r}}}}}{1 \leq q \leq d-r+1}{0}{q > d-r+1.}}}
The following lemma plays a key role in our recursive calculation of $\mukg \ofz$.
\begin{lemma}\label{thm:recursion_aqg}
The quantities $s_{q,\gamma}^{\br{r}} \ofz$ satisfy the following recurrence relation for $r=1,\ldots,d$,
\eq[eq:recursion_aqgr]{\ensuremath{
    s_{q,\gamma}^{\br{r}} \ofz = \threecase{\twofloors{s_{q,\gamma}^{\br{r+1}} \ofz \cdot \exp\br{-\gamma \subplus{z_r - z_{\br{q}}^{\br{r+1}} } } +}{s_{q-1,\gamma}^{\br{r+1}} \ofz \cdot \exp\br{-\gamma \subminus{z_r - z_{\br{q}}^{\br{r+1}}} }}} {1 \leq q \leq d-r }
    {1} {q = {0,d-r+1}} {0} {q > d-r+1.}}}
\end{lemma}
As described in \cref{alg:calc_mu_consecutive}, the order of the recursion in \cref{eq:recursion_aqgr} is from $r=d+1$ down to $r=1$, and for each $r$, from $q=0$ up to $q=\min \brc{d-r+1,k}$.

Using floating-point arithmetic, the recursion \cref{eq:recursion_aqgr} may still lead to an overflow due to the large number of summands in $s_{q,\gamma} \ofz$, Eq.~\cref{eq:def_aqg}.
We therefore switch to a logarithmic representation. For $q \geq 0$ and $1 \leq r \leq d+1$, define
\eq[eq:def_bqgr]{\ensuremath{
b_{q,\gamma}^{\br{r}} \ofz \eqdef \twocase{\log \br{\frac{1}{\binom{d-r+1}{q} } s_{q,\gamma}^{\br{r}} \ofz }}{0 \leq q \leq d-r+1}{0}{\text{otherwise.}}}}
For $r=1$, we denote $b_{q,\gamma} \ofz \eqdef b_{q,\gamma}^{\br{1}} \ofz$.
Note that by \cref{eq:expressing_mu_by_a},
\eq[eq:expressing_mu_by_b]{\ensuremath{
    \mu_{k,\gamma} \ofz = \frac{1}{\gamma} b_{k,\gamma} \ofz + \summ[i=1][k] z_{\br{i}}.}}
By \cref{eq:def_aqgr}, $s_{q,\gamma}^{\br{r}} \ofz$ is the sum of $ \binom{d-r+1}{q}$ exponential terms whose maximum equals $1$. Therefore, the average $\frac{1}{\binom{d-r+1}{q}} s_{q,\gamma}^{\br{r}} \ofz$ is bounded in $\brs{\frac{1}{\binom{d-r+1}{q}}, 1}$.  
This, in turn, implies that $b_{q,\gamma}^\br{r}\ofz$ is bounded in $\brs{- q \log \br{d-r+1}, 0}$, and thus does not overflow.

Reformulating \cref{eq:recursion_aqgr}, we have the following recursive formula for $b_{q,\gamma}^{\br{r}} \ofz$:
\eq[eq:recursion_bqgr]{\ensuremath{
b_{q,\gamma}^{\br{r}} \ofz = \twocase{\twofloors { \log \Biggl[ \frac{d-r-q+1}{d-r+1} \exp \br{b_{q,\gamma}^{\br{r+1}} \ofz - \gamma \subplus{z_r - z_{\br{q}}^{\br{r+1}} }} \ + \Biggr. }{ \Biggl. \frac{q}{d-r+1} \exp \br{ b_{q-1,\gamma}^{\br{r+1}} \ofz -\gamma \subminus{z_r - z_{\br{q}}^{\br{r+1}}} } \Biggr] }  } {\threefloors{1 \leq r \leq d}{\mbox{and}}{1 \leq q \leq d-r}}
{0} {\text{$q \geq d-r+1$.}}}}
The recursion base is $b_{q,\gamma}^{\br{r}} \ofz =0$ for (a) $r=d+1$, (b) $q=0$, or (c) $q \geq d-r+1$.

\Cref{eq:recursion_bqgr} is still numerically unsafe when the term inside the logarithm is close to 1, due to possible loss of significance ---  a well-known phenomenon in numerical analysis \cite{kincaidcheney1996}. 
Specifically, for $\abs{x}\ll1$, $\log\br{1+x} \approx x$, $\exp \br{x} \approx 1+x$, and $1+x$ may contain few significant digits of the original $x$, thus leading to a high relative error.
To overcome this problem, we use the functions $\logonep \br{x} = \log\br{1+x}$ and $\expmone \br{x} = \exp\br{x}-1$,   
implemented in the standard libraries of most computing languages. For $x$ near zero, they yield a more accurate result.

Instead of calculating the recursion step by \cref{eq:recursion_bqgr}, we thus consider the two following equivalent formulas, derived from \cref{eq:recursion_bqgr} by algebraic manipulations,
\begin{subequations}\label{eq:stable_b}
    \begin{equation}\label{eq:stable_b_forma}
    b_{q,\gamma}^{\br{r}} \ofz = \logonep \Bigbrs{\frac{d-r-q+1}{d-r+1} \expmone \br{b_{q,\gamma}^{\br{r+1}} \ofz - b_{q-1,\gamma}^{\br{r+1}} \ofz - \xi} } + b_{q-1,\gamma}^{\br{r+1}} \ofz - \subminus{\xi},
    \end{equation}
    \begin{equation}\label{eq:stable_b_formb}
    b_{q,\gamma}^{\br{r}} \ofz = \logonep \Bigbrs{\frac{q}{d-r+1} \expmone \br{b_{q-1,\gamma}^{\br{r+1}} \ofz - b_{q,\gamma}^{\br{r+1}} \ofz + \xi} } + b_{q,\gamma}^{\br{r+1}} \ofz - \subplus{ \xi }
    \end{equation}
\end{subequations}
with $\xi = \gamma (z_r - z_{\br{q}}^{\br{r+1}})$. 
We use \cref{eq:stable_b_forma} as the recursion step when
\eq[eq:stable_b_condition]{\ensuremath{
b_{q,\gamma}^{\br{r+1}} \ofz - b_{q-1,\gamma}^{\br{r+1}} \ofz - \xi \leq 0,}}
and use \cref{eq:stable_b_formb} otherwise. This guarantees that in both cases, the argument of $\expmone \br{\argdot}$ is non-positive, and the argument of $\logonep \br{\argdot}$ is in $\brs{-\max \brc{\frac{q}{d-r+1}, \frac{d-r-q+1}{d-r+1}}, 0 }$. This, in turn, ensures that no overflow occurs, and that if one of the $\expmone\br{\argdot}$ terms is small enough to underflow, its influence on the result is anyway negligible. 

In summary, to calculate $\mukg\ofz$, we call \cref{alg:calc_mu_consecutive} with parameter $s = k$. 
\Cref{alg:calc_mu_consecutive} applies the recursion \cref{eq:recursion_bqgr}, using \cref{eq:stable_b_forma} or \cref{eq:stable_b_formb}, depending on the condition~\cref{eq:stable_b_condition}.
Overall, this takes $\mathcal{O}\br{kd}$ operations and $\mathcal{O}\br{k}$ additional memory.

\begin{algorithm}[t]
\caption{Calculate $\muargs{k}{\gamma} \ofz$ recursively}
\label{alg:calc_mu_consecutive}
\begin{algorithmic}[1]
    \Input $\z = \br{z_1,\ldots,z_d} \in \R[d]$, $\ k,s \in \Z\ $ s.t. $1 \leq k \leq s \leq d$, $\ \gamma \in \br{0,\infty}$     
    \Output $\muargs{k}{\gamma} \ofz$, $\brc{b_{q,\gamma} \ofz}_{q=0}^s$ and $\brc{z_\br{q}}_{q=0}^s$
    \Variables $\bvec = \br{b_0,\ldots,b_s},\ \tilde{\bvec} = \br{\tilde{b}_0,\ldots,\tilde{b}_s},\ \vvec = \br{v_0,\ldots,v_s},\ \tilde{\vvec} = \br{\tilde{v}_0,\ldots,\tilde{v}_s}$
    \State Initialize 
       $b_0 \assign 0$, 
       $v_0, \tilde{v}_0 \assign \infty$
    \For{$r \assign d,d-1,\ldots,1$} 
        \For{$q \assign 1,2,\ldots,\min \brc{s,d-r}$}
            \State $v_q \assign \max \brc{ \min \brc{z_r, \tilde{v}_{q-1}}, \tilde{v}_{q} }$
            \State $\xi \assign \gamma \br{z_r - \tilde{v}_q}$
            \State $\eta \assign \tilde{b}_{q} - \tilde{b}_{q-1} - \xi$
            \If{$\eta \leq 0$}
                $\ b_q \assign \logonep \br{\frac{d-r-q+1}{d-r+1} \expmone \br{\eta}} + \tilde{b}_{q-1} - \subminus{\xi}$           
            \Else 
                $\ b_q \assign \logonep \br{\frac{q}{d-r+1} \expmone \br{-\eta}} + \tilde{b}_q - \subplus{\xi}$
            \EndIf
        \EndFor
        \If{$s \geq d-r+1$} 
            \State $v_{d-r+1} \assign \min \brc {z_r, \tilde{v}_{d-r}}$
            \State $b_{d-r+1} \assign 0$
        \EndIf
        \State Set $\tilde{\bvec} \assign \bvec$, $\tilde{\vvec} \assign \vvec$
        \Comment{Here $b_q = b_{q,\gamma}^{r} \ofz$, $v_q = z_{\br{q}}^{\br{r}}$ for $q=0,\ldots,\min\brc{s,d-r+1}$}
    \EndFor
    \State\Return $b_{q,\gamma} \ofz = \tilde{b}_q$, $\ z_\br{q} = \tilde{v}_q$ for $q=0,\ldots,s\ $ and 
    $\ \muargs{k}{\gamma} \ofz = \frac{1}{\gamma} b_{k,\gamma} \ofz + \sum_{q=1}^k v_q$
\end{algorithmic}   
\end{algorithm}

\subsection{Accuracy 
and runtime evaluation}
We evaluated the accuracy of our numerical scheme by calculating $\mukg\ofz$ and $\gsm\ofz$ on randomly generated vectors $\z \in \R[d]$ using \cref{alg:calc_gsm_main}. The calculated values
were compared to the reference values  $\mu_{\textup{Ref}}$ and $\boldtheta_{\textup{Ref}}$, 
computed by \cref{alg:calc_gsm_main} using 128-bit quadruple-precision arithmetic\footnote{The numerical evaluation code is available at \url{https://github.com/tal-amir/gsm}.}.

{Two tests settings were considered: (1) $d=1000$, $k\in\brc{10:10:1000}$; and (2) $d=100000$, $k \in \brc{10,50,100,200}$, with $a:b:c$ denoting ${a, a+b, a+2b, a+3b, \ldots, c}$. The following 18 values were used for $\gamma$: $10^{-20},10^{-10}, 10^{-5}, 10^{-2}, 0.2:0.2:1, 2:2:10, 10^2, 10^5, 10^{10}, 10^{20}$. Two types of random vector $\z$ were used: (a) $z_i$ i.i.d. $\mathcal{U}\br{0,1}$ and (b) $z_i$ is the absolute value of i.i.d. $\mathcal{N}\br{0,1}$. For each combination of $d,k,\gamma$ and random vector type, 200 instances were tested.}

We considered the following error measures for $\mukg$ and
$\gsm$: 
$\abs{\frac{{\mu} - \mu_{\textup{Ref}}}{\mu_{\textup{Ref}}}}$ and $\frac{1}{k}\norm{{\boldtheta} - \boldtheta_{\textup{Ref}}}_{\infty}$. 
\Cref{tab:gsm_accuracy} summarizes
the largest obtained errors
of \cref{alg:calc_gsm_main}, 
invoked either with
64-bit double-precision or with 32-bit single-precision arithmetic. 
%

\setlength\extrarowheight{3pt}

\begin{table}[htbp]
    \begin{center}
    \begin{tabulary}{\textwidth}{|l|cc|cc|}
        \hline
        \multicolumn{1}{|c|}{\multirow{2}{*}{Precision}}          & \multicolumn{2}{c|}{$\mukg\ofz$}     & \multicolumn{2}{c|}{$\gsm\ofz$}   \\ \cline{2-5}
                          & $d=1,000$  & $d=100,000$    & $d=1,000$     & $d=100,000$    \\ \hline
        Double (64 bit)   & 4.5e-15    & 1.2e-13        & 2.1e-14       & 2e-12          \\ 
        Single (32 bit)   & 2e-6       & 1e-4           & 9e-6          & 4e-4           \\ \hline            
    \end{tabulary}
    \end{center}
    \caption{Numerical accuracy of \cref{alg:calc_gsm_main} \label{tab:gsm_accuracy}}
    \small Numerical errors for $\mukg\ofz$ and $\gsm\ofz$, given by $\abs{\frac{{\mu} - \mu_{\textup{Ref}}}{\mu_{\textup{Ref}}}}$ and $\frac{1}{k}\norm{{\boldtheta} - \boldtheta_{\textup{Ref}}}_{\infty}$ respectively. 
\end{table}

For $d\gg 1$, the accuracy may be further enhanced by replacing the consecutive scheme of \cref{alg:calc_mu_consecutive} with one that separately calculates $b_{q,\gamma}\br{\argdot}$ for multiple subvectors of $\z$ and merges the results. This is analogous to the numerical error of calculating the sum of a $d$-dimensional vector, where 
a pairwise summation has a much smaller accumulation of errors compared
to consecutive summation. 

\Cref{tab:gsm_times} presents average run-times of \cref{alg:calc_gsm_main}, implemented in C, on a standard PC.  
In accordance to the theoretical analysis, the times grow roughly linearly with $k$ and $d$.

\begin{table}[h]
    \begin{center}
    \begin{tabulary}{\textwidth}{|l|cccc|}
        \hline
                  & $d=1,000$ & $d=10,000$ & $d=100,000$ & $d=1,000,000$ \\ \hline
        $k=10$    & 9e-3      & 6e-2       & 7e-1        & 6           \\ 
        $k=100$   & 6e-2      & 6e-1       & 6           & 56            \\ 
        $k=500$   & 5e-1      & 3          & 28          & 273           \\ \hline            
    \end{tabulary}
    \end{center}
    \caption{Average runtimes [sec.] over 50 realizations of \cref{alg:calc_gsm_main}. \label{tab:gsm_times}}
\end{table}



\FloatBarrier
\section{Numerical experiments}
\label{sec:numerical_experiments}
%
%
\subsection{GSM vs. direct optimization of the trimmed lasso}
\label{sec:gsm_vs_direct_tls}
First, we illustrate the advantage of our GSM-based approach in comparison to DC-programming and to ADMM, which directly optimize the trimmed lasso penalized \cref{pr:P2l}. We implemented all methods in Matlab.\footnote{Our Matlab code is available at \url{https://github.com/tal-amir/sparse-approximation-gsm}.}
For DC-Programming and ADMM we followed the description in \cite{bertsimas2017trimmed}; see \cref{appendix:computational_details}.

We considered the following simulation setup: In each realization, we generate a random matrix $A$ of size $60 \times 300$ whose entries are i.i.d. $\mathcal{N}\br{0,1}$, followed by column normalization. For $k \in \brc{10,20,25}$, we generate a random $k$\nbdash sparse signal $\x_0$ whose nonzero coordinates are i.i.d. $\mathcal{N}\br{0,1}$. We observe $\y=A\x_0+\e$, where $\e \in \mathbb{R}^{60}$ is a random noise vector whose entries are i.i.d. $\mathcal{N}\br{0,1}$, normalized such that $\norm{\e}_2 = 0.01 \cdot \norm{A\x_0}_2$. We considered several values of $\lambda$ in the interval $[10^{-5},1]\cdot \lambdabar$, 
where $\lambdabar$ is defined in \cref{eq:lambdabar}. 
For each obtained solution $\xhat$, 
we compute its normalized objective $\Ftwol\br{\xhat} / \Ftwol\br{\x_0}$. Values near or below $1$ indicate solutions close to the global minimizer, as follows from \cref{thm:sparse_reconstruction_p2}, whereas values significantly above $1$ indicate highly suboptimal solutions, typically with a very different support from that of $\x_0$.

\Cref{fig:plambda_tests_1} shows, for two different values of $\lambda$, 
the normalized objectives obtained by the three methods for 100 random realizations. Points above the diagonal represent instances for which GSM found solutions with lower objective values than the competing method. As seen in the figure, small values of $k$ lead to relatively easy problems, where all methods succeed.
At larger values of $k$, 
a significant number of instances are in the top left quadrant, which represent instances where DC-programming and ADMM found highly-suboptimal solutions in comparison to those found by GSM. 
Similar results were obtained for other values of $\lambda$. 

\captionsetup[subfigure]{labelformat=empty}
\begin{figure}[t]
    \centering
    \subfloat[$\lambda = 0.005 \cdot \lambdabar$]{\includegraphics[width=0.35\textwidth]{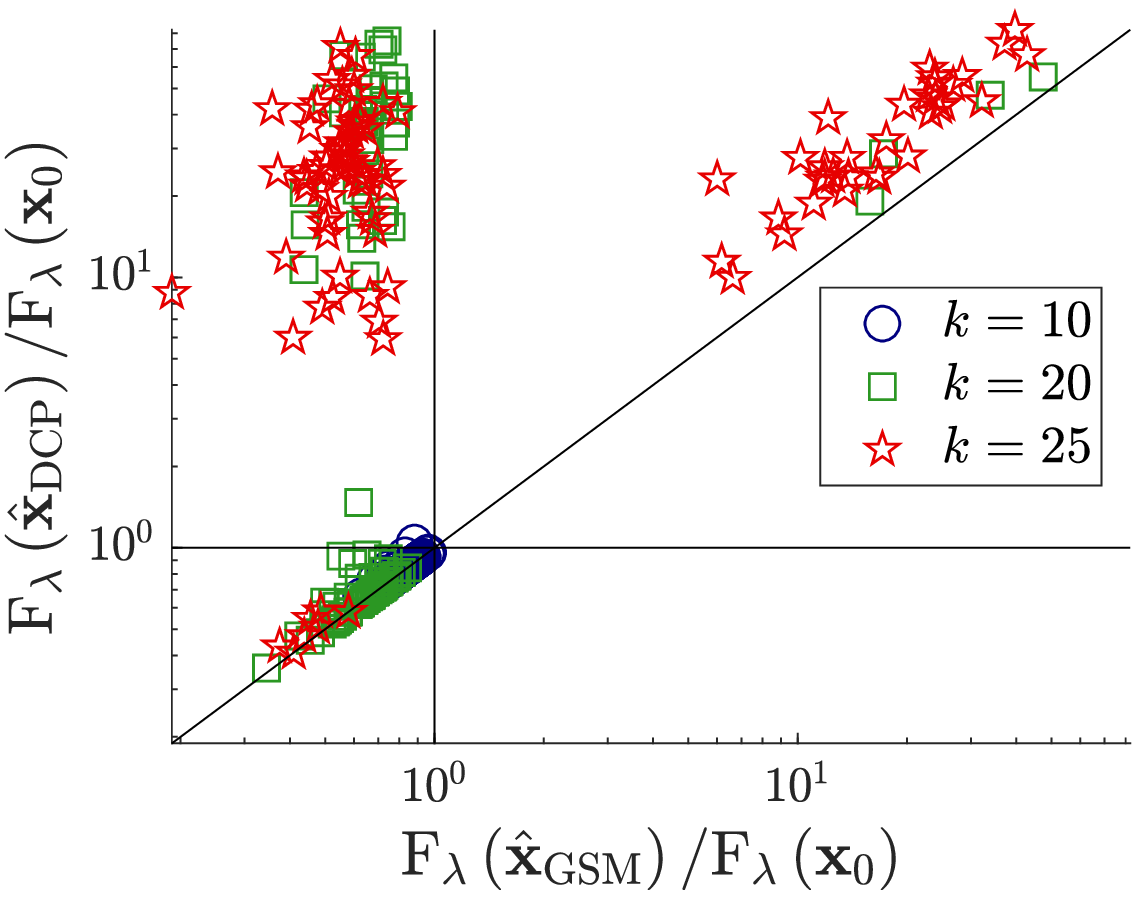}}
    \subfighspace
    \subfloat[$\lambda = 0.005 \cdot \lambdabar$]{\includegraphics[width=0.35\textwidth]{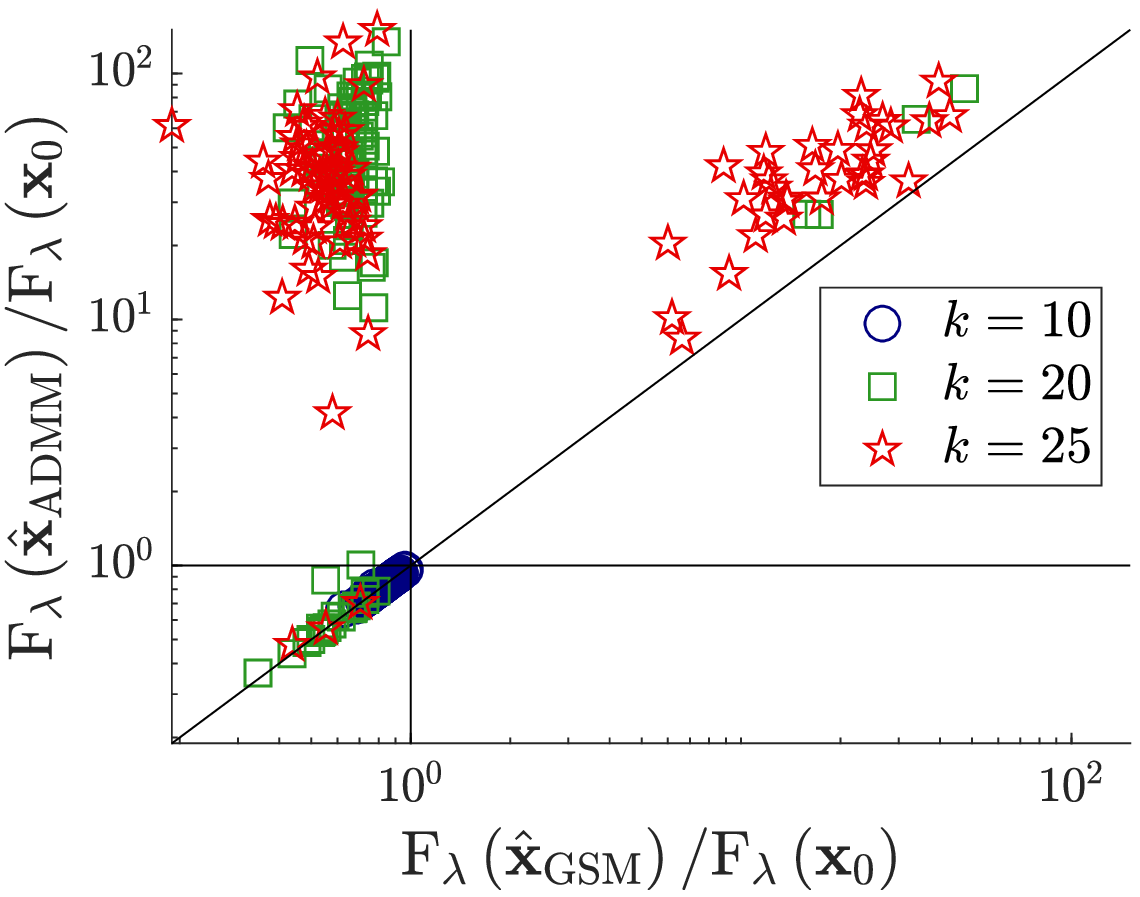}}
    \\
    \subfloat[$\lambda = 0.5 \cdot \lambdabar$]{\includegraphics[width=0.35\textwidth]{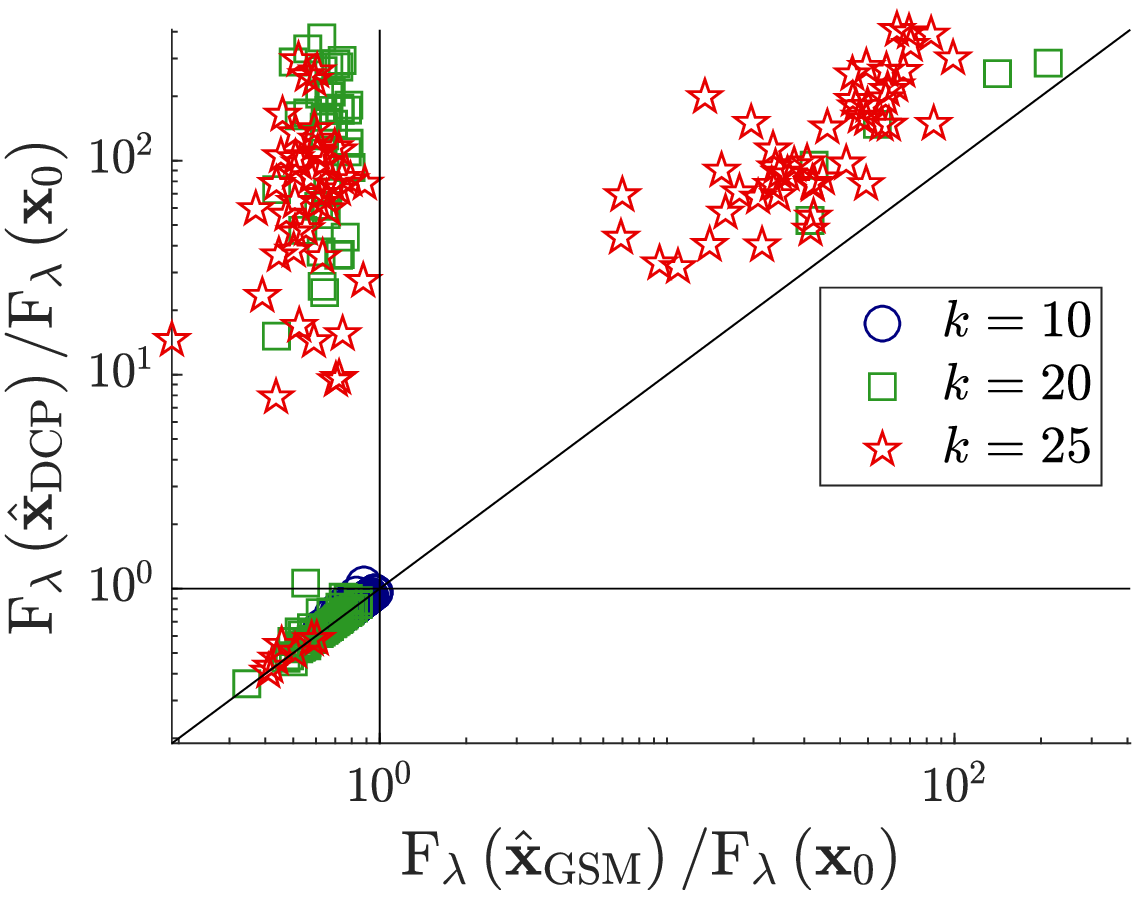}}
    \subfighspace
    \subfloat[$\lambda = 0.5 \cdot \lambdabar$]{\includegraphics[width=0.35\textwidth]{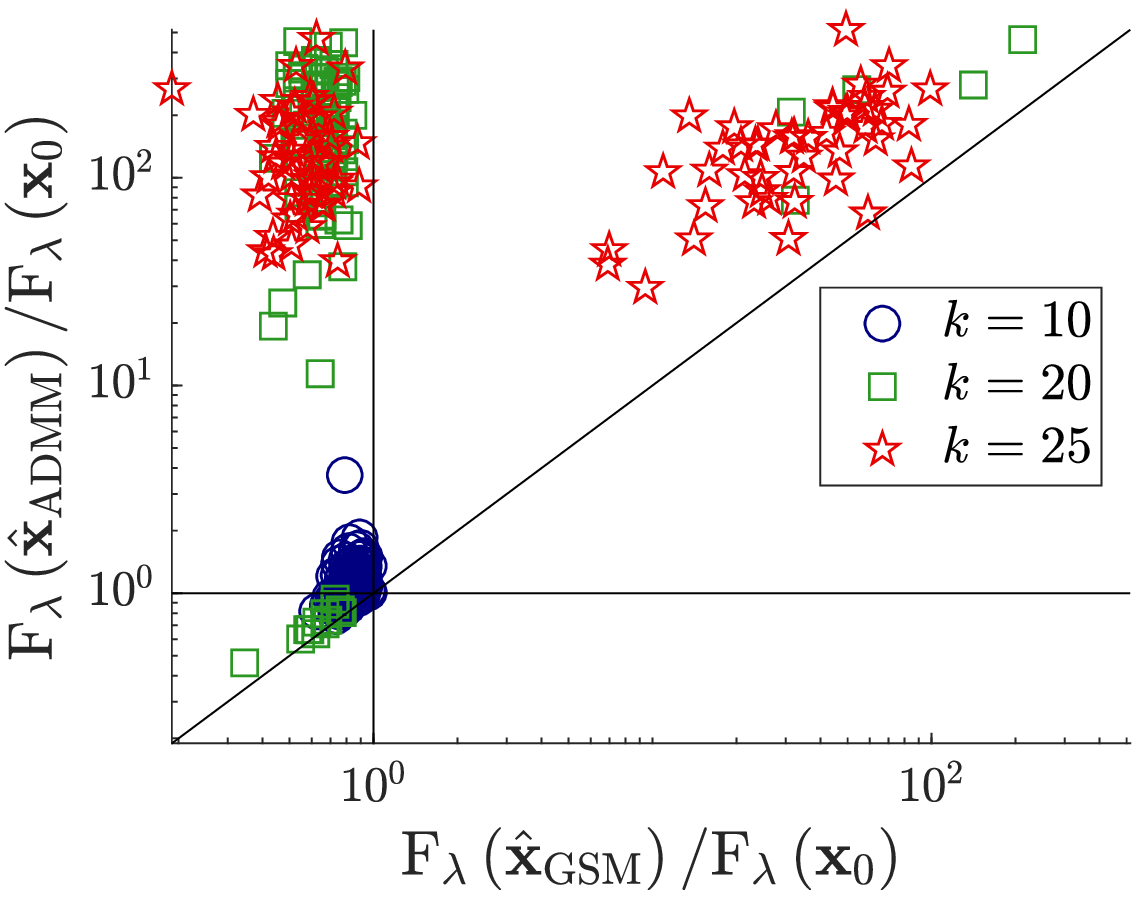}}
    \caption{Objective values of solutions to \cref{pr:P2l} obtained by DC programming (left) and by ADMM (right), compared with the objective of the solution obtained by our GSM method. Each objective value is normalized by that of the original $\x_0$. Each point represents one realization. Points above the diagonal represent instances where GSM found better solutions. Normalized objectives much higher than $1$ indicate highly suboptimal solutions. }
    \label{fig:plambda_tests_1}    
\end{figure}

\subsection{Sparse signal recovery}
\label{sec:sparse_recovery_simulation}
%
Next, we compare the performance of GSM to several popular methods in sparse signal recovery. We observe $\y=A\x_0+\e$, where the 
$k$\nbdash sparse signal $\x_0 \in \R[d]$, matrix $A \in \R[n \times d]$ and noise vector $\e \in \R[n]$ were generated as described below.


We considered two types of random matrices $A$: \emph{Uncorrelated} matrices, whose entries are i.i.d. $\mathcal{N}(0,1)$, and \emph{correlated} matrices, whose rows are drawn from a $d$-dimensional Gaussian distribution $\mathcal{N}\br{\zerovec,\Sigma}$, with covariance matrix $\Sigma$ given by $\sigma_{i,j} = \rho^{\abs{i-j}}$, $\rho = 0.8$, similarly to \cite{bertsimas2016best}. In both cases, the columns of $A$ were normalized to have unit $\ell_2$\nbdash norm.
Three types of signals $\x_0$ were used: (i) \emph{Gaussian}, whose $k$ nonzero entries are at randomly chosen indices and their values are i.i.d. $\mathcal N(0,1)$; (ii) \emph{equispaced linear}, whose $k$ nonzeros are at equispaced indices,
their magnitudes are chosen from $\setst{1 + \frac{i-1}{k-1}29}{i=1,\ldots,k}$ randomly without repetition, and their signs are i.i.d. $\pm 1$; 
(iii) \emph{equispaced $\pm 1$},
similar to (ii) only that the nonzero entries are i.i.d. $\pm 1$.
Similar simulation designs were considered, e.g., in \cite{buhlmann2013correlated}. 

The vector $\e \in \R[n]$ is drawn from $\mathcal N({\bf 0},\sigma^2 I_n)$ with variance
$\sigma^2 \eqdef \nu^2 \mathbb{E}_{\x} \brs{\norm{A\x}_2^2}/n$,
where $\nu \in \brs{0,1}$ is a noise-strength parameter, and the expectation $\mathbb{E}_{\x} \brs{\norm{A\x}_2^2}$ is over the distribution of $\x_0$, estimated empirically from 2000 randomly drawn signals. 
With these definitions, the signal-to-noise ratio (SNR) is given by $\mbox{SNR} = \mathbb{E}_{\x} \brs{\|A\x\|_2^2} / \mathbb{E}_{\e}\brs{\|\e\|_2^2} = 1/\nu^2$. 


\myparagraph{Evaluated methods}
%
We compared the following 8 methods: 
1-2) Two variants of our {GSM}, with residual norm power $p=2$ or $p=1$, as in \cref{pr:P2l,pr:P1l} respectively;
3) Iterative support detection ({ISD}) \cite{wang2010sparse}; 
4-5) $\ell_p$-minimization by IRLS or by IRL1; 
6) Minimization of the trimmed lasso by DC programming \cite[algorithm 1]{bertsimas2017trimmed}; 
7) Least-squares OMP; 
8) Basis pursuit denoising \cite{chen2001atomic}. 
For ISD,  we used the code provided by its authors. For all other methods, 
we used our own Matlab implementation. 
As we assume that the sparsity level $k$ is given, the output of each method was post-processed by solving a least-squares problem on its $k$ largest-magnitude entries and setting the other entries to zero --- thereby ensuring that the evaluated solution is $k$\nbdash sparse and is optimal on its support.
Further technical details appear in 
\cref{appendix:computational_details}. 

\myparagraph{Performance evaluation}
%
A solution $\xhat$ is evaluated by the following 3 quality measures: 
(i) \emph{normalized objective value}; (ii) \emph{relative recovery error},
and (iii) {\em support precision}, given by  
\begin{equation}
\text{NormObj}\br{\xhat} \eqdef \frac{\norm{A\xhat - \y}_2}{\norm{A\x_0 - \y}_2},
\ 
\text{RecErr}\br{\xhat} \eqdef \frac{\norm{\xhat - \x_0}_1}{\norm{\x_0}_1},
\ 
\text{SuppPrec}\br{\xhat} \eqdef \frac{\abs{ \text{supp}\br{\xhat} \cap \text{supp}\br{\x_0} } }{k}
	\nonumber
\end{equation}
where $\text{supp}\br{\x} = \setst{i \in \brs{d}}{x_i \neq 0}$. An optimization is deemed \emph{successful} if $\text{NormObj}\br{\xhat} \leq 1$.
Such solutions $\xhat$
are guaranteed by \cref{thm:sparse_reconstruction_p1} to be close to the ground-truth $\x_0$, provided that $\ripmin$ is not too small. In contrast, solutions with $\text{NormObj}\br{\xhat} \gg 1$ are highly suboptimal, and are often poor estimates of $\x_0$. 
We consider a recovery successful if
$\text{RecErr}\br{\xhat} \leq \max\{2\nu,10^{-3}\}$. The additional threshold of $10^{-3}$
is needed to accommodate methods
that have early stopping criteria. The results are not sensitive to these specific thresholds.

%

\myparagraph{Recovery performance as a function of sparsity level}
For each $k \in \brc{16,18,20,\ldots,48}$ we generated 200 random instances of $A$,$\x_0$ and $\e$, with $A \in \R[100\times 800]$ and  essentially zero noise ($\nu = 10^{-6}$). The results of the eight methods appear in \Cref{fig:success_vs_sparsity_noiseless_1}.
As seen in the plots, GSM achieved the highest optimization and recovery success rates, with the closest competitor being 
ISD. The improved performance of ISD over some of the other methods is in accordance with \cite{wang2010sparse}.
Empirically, the power-2 variant of GSM was slightly better than the power-1 variant. 
GSM also achieved the best average support precision, except at large values of $k$, where 
the recovery success rates of all methods were near zero. 

\begin{figure}[t]
    \centering
    \includegraphics[width=0.28\textwidth]{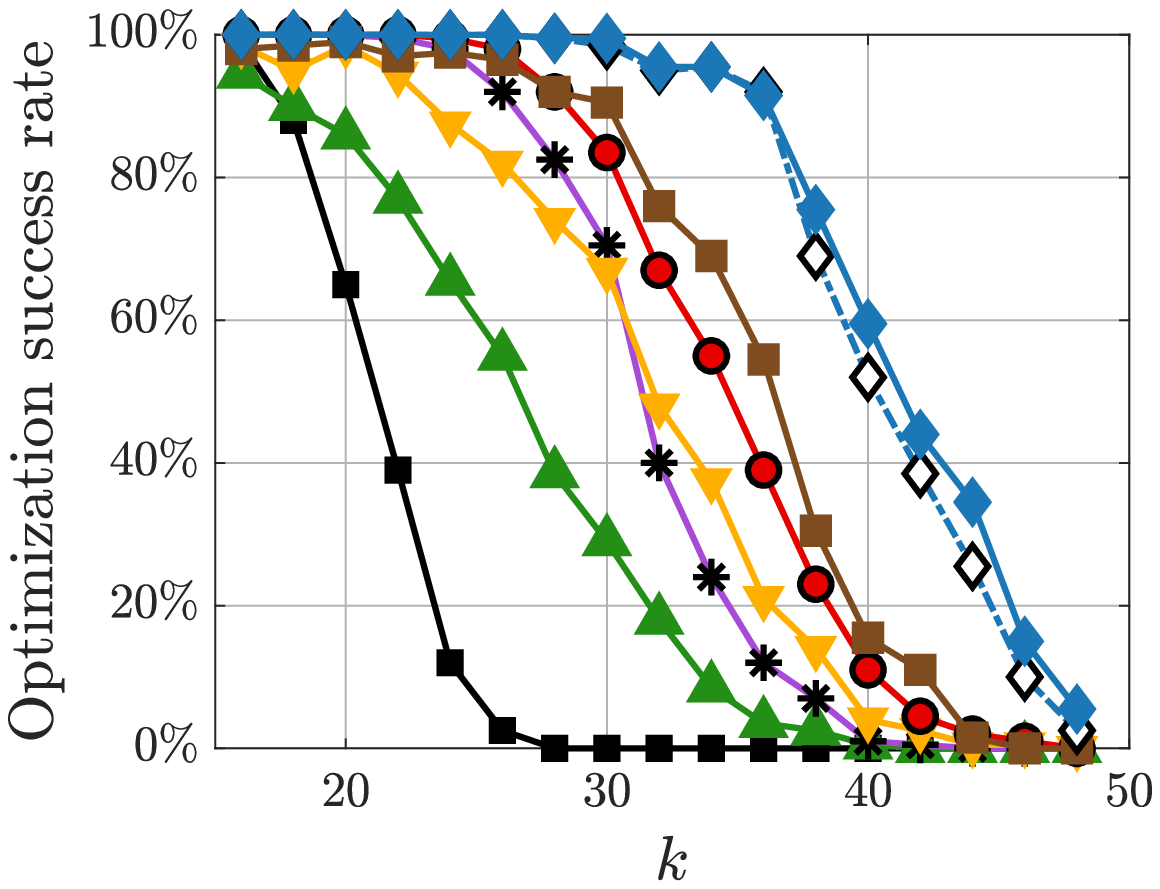}
    \subfighspace
    \includegraphics[width=0.28\textwidth]{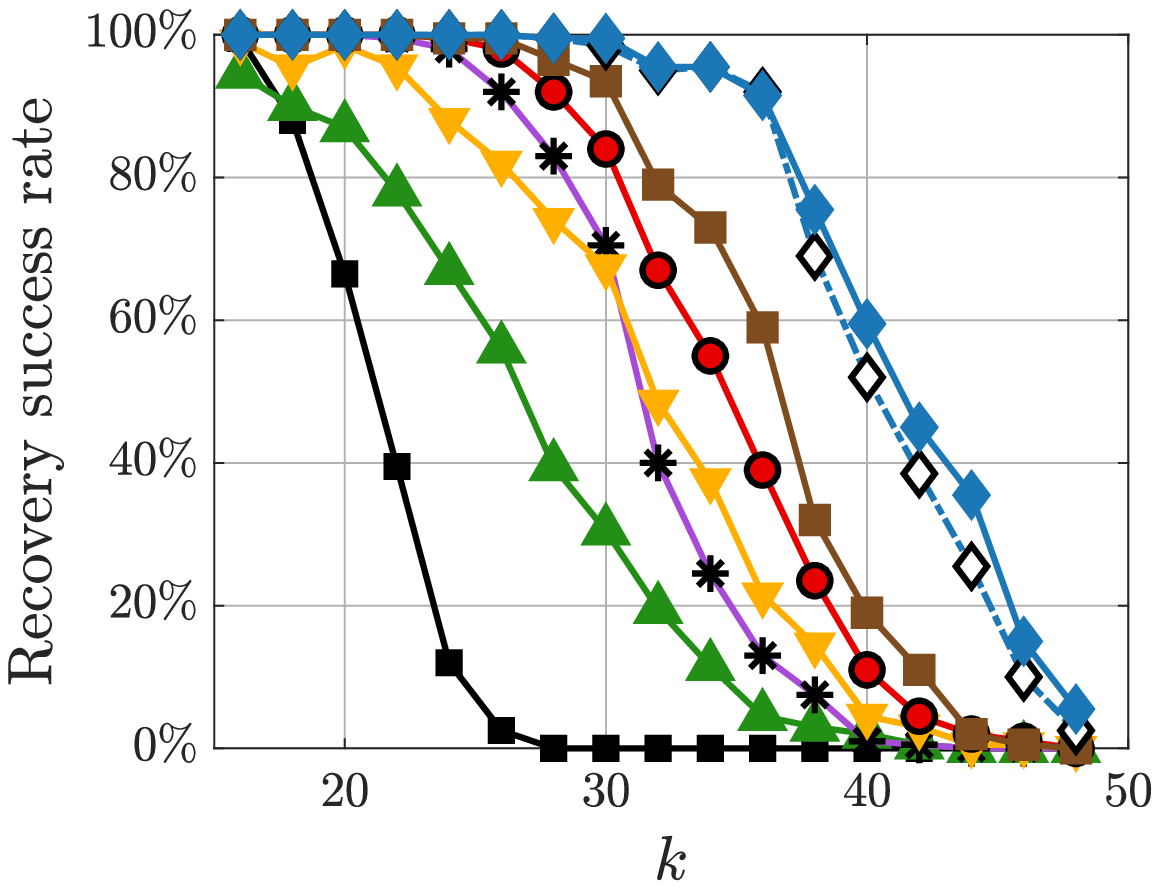}
    \subfighspace
    \includegraphics[width=0.28\textwidth]{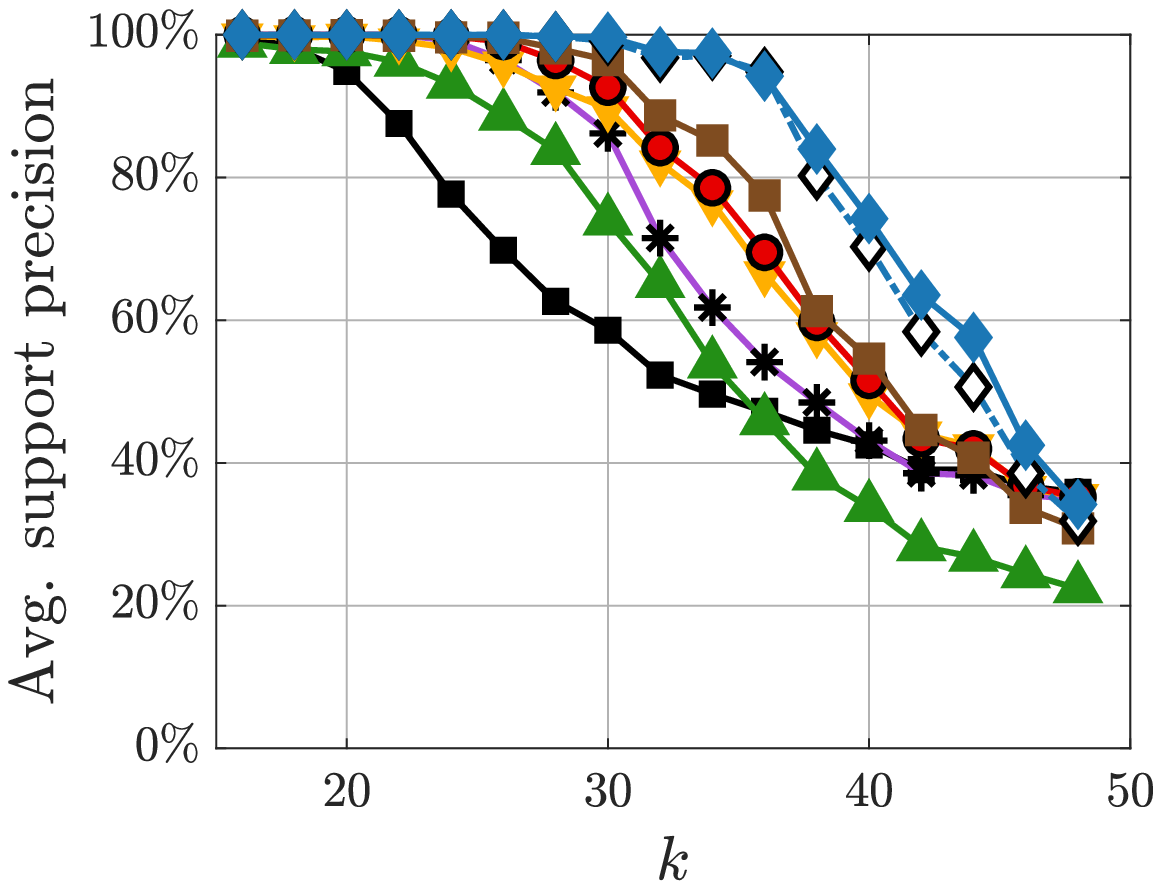}
    \newline
    \includegraphics[width=0.28\textwidth]{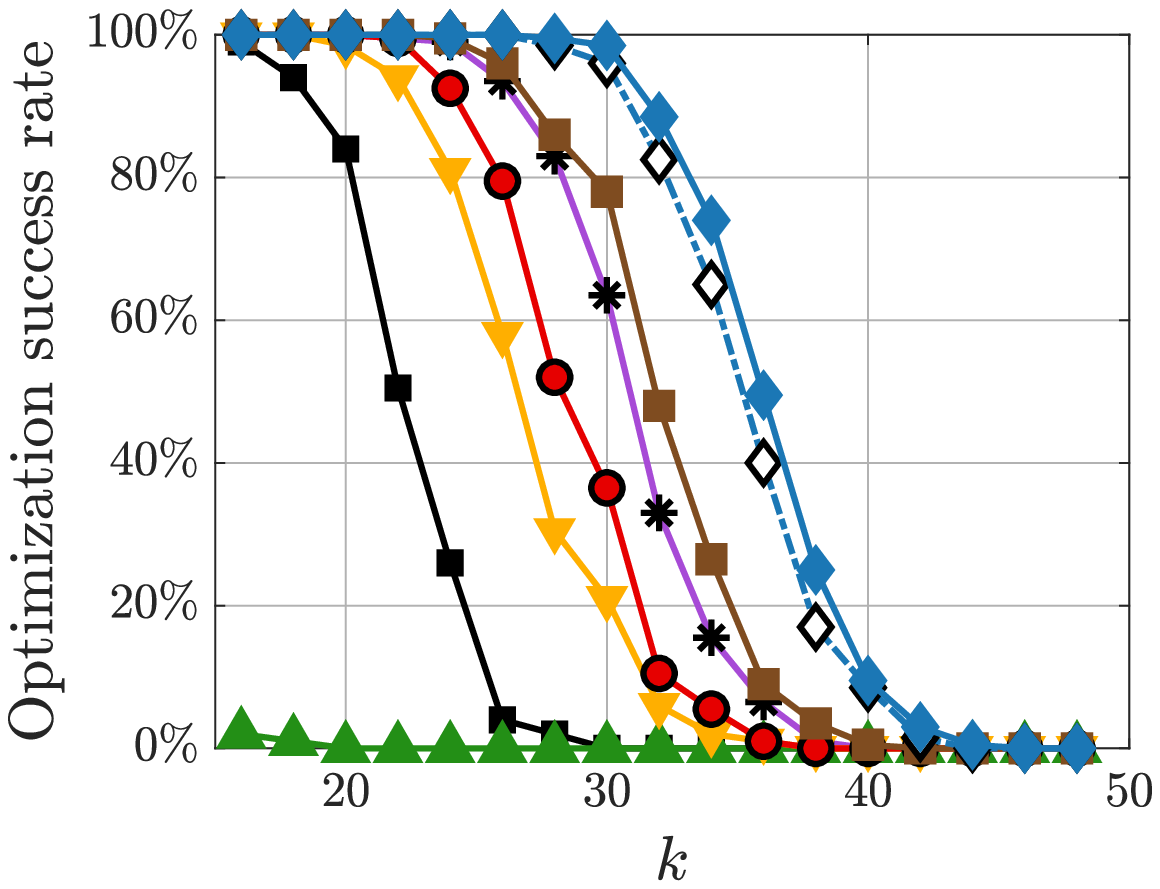}
    \subfighspace
    \includegraphics[width=0.28\textwidth]{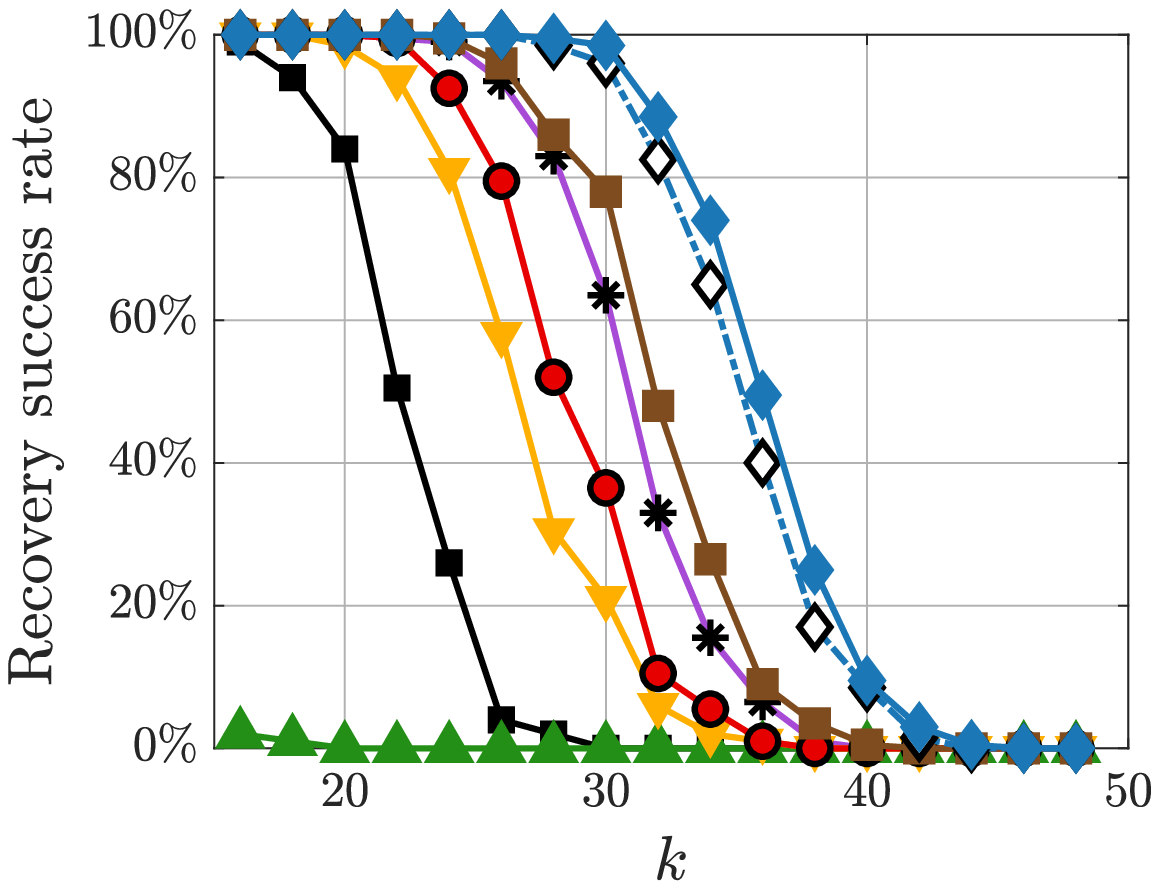}
    \subfighspace
    \includegraphics[width=0.28\textwidth]{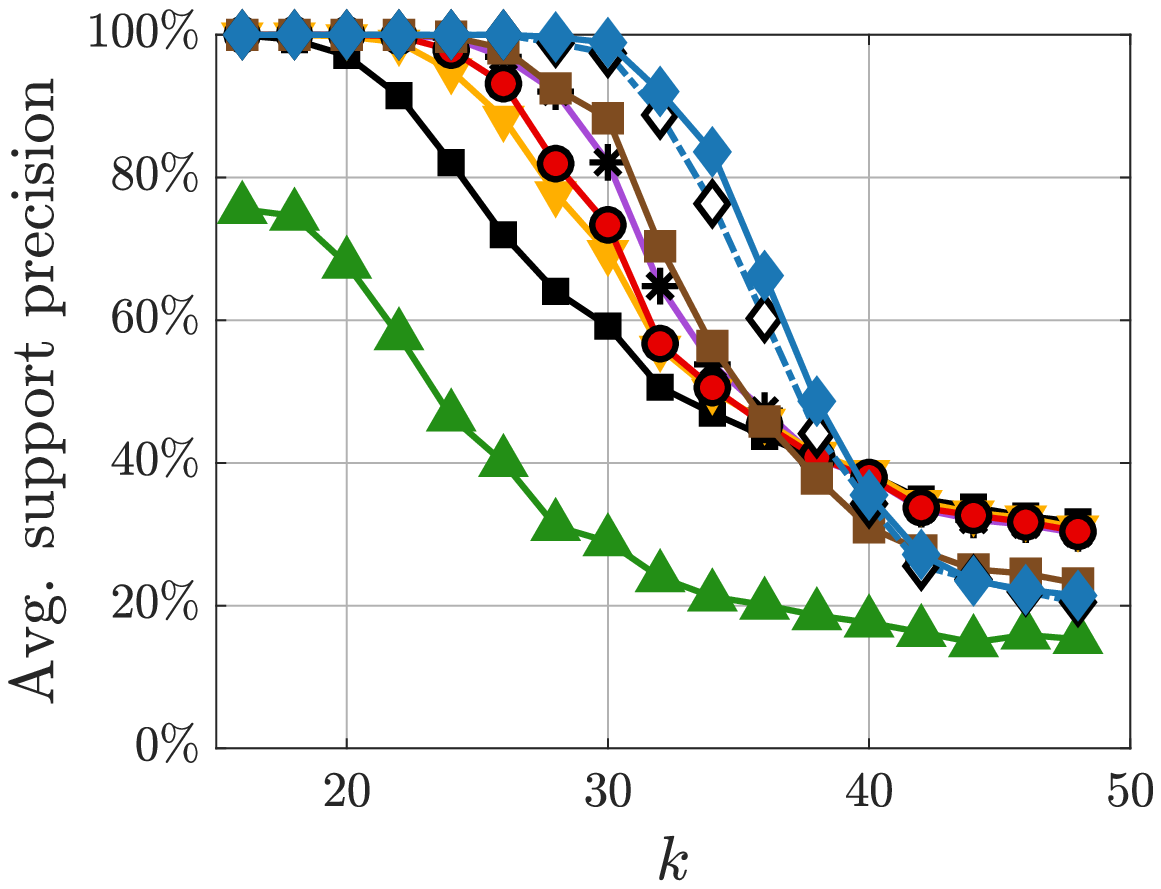}
	\vspace{5pt}
	\newline
    \includegraphics[width=0.65\textwidth]{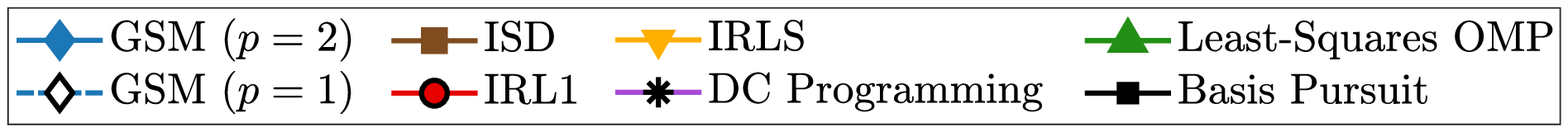}
    \caption{Recovery of $k$-sparse signals in a near-zero noise setting ($\nu = 10^{-6}$). Matrix size: $100 \times 800$. Top row: Gaussian signal, uncorrelated matrix. Bottom row: Equispaced linear signal, correlated matrix. }
    \label{fig:success_vs_sparsity_noiseless_1}    
\end{figure}

\begin{figure}[t]
    \centering
    \includegraphics[width=0.28\textwidth]{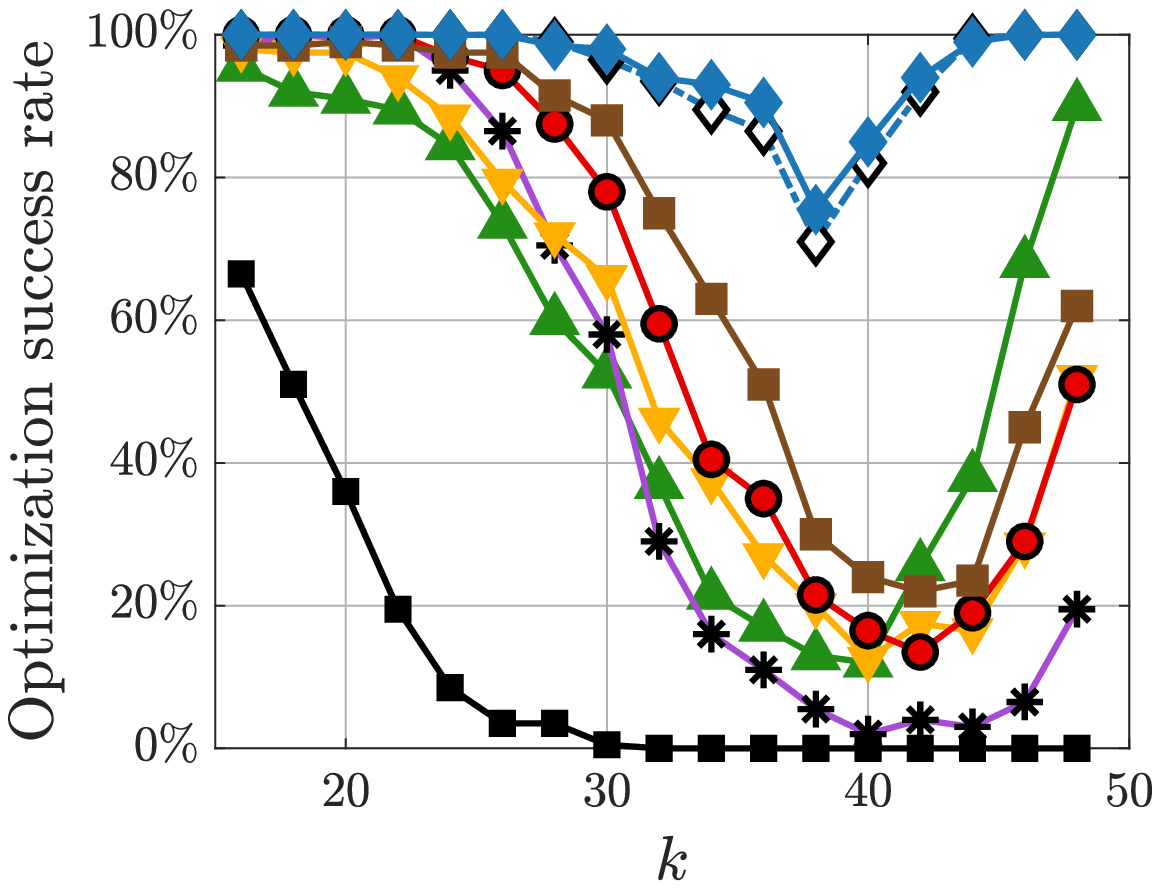}
    \subfighspace
    \includegraphics[width=0.28\textwidth]{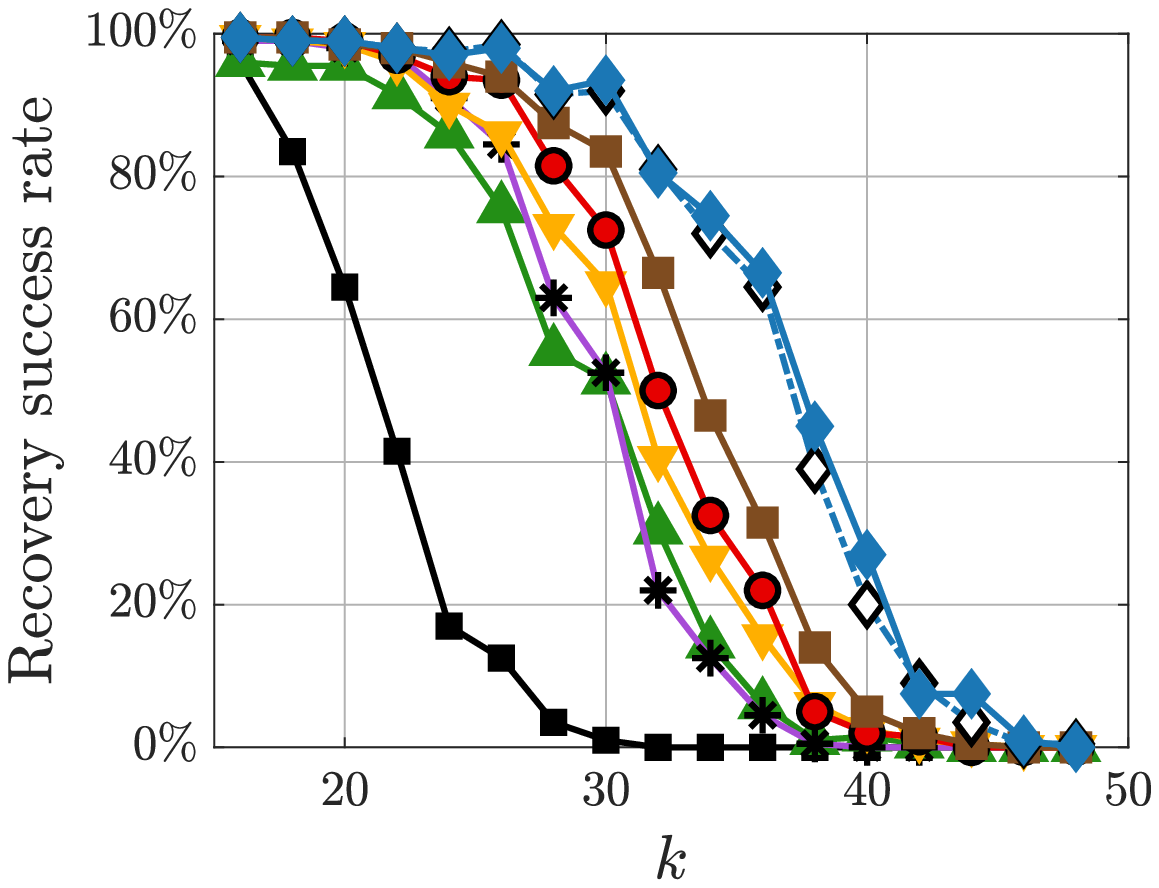}
    \subfighspace
    \includegraphics[width=0.28\textwidth]{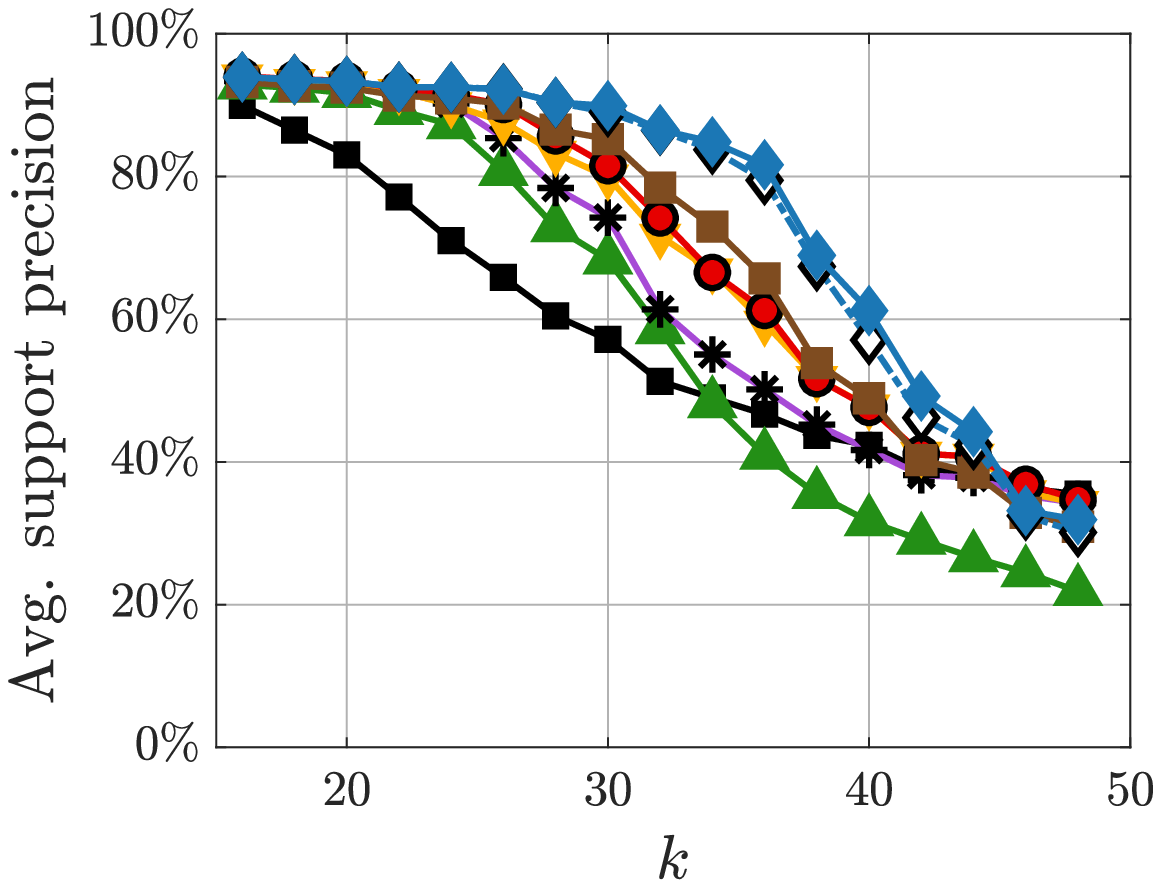}
    \newline
    \includegraphics[width=0.28\textwidth]{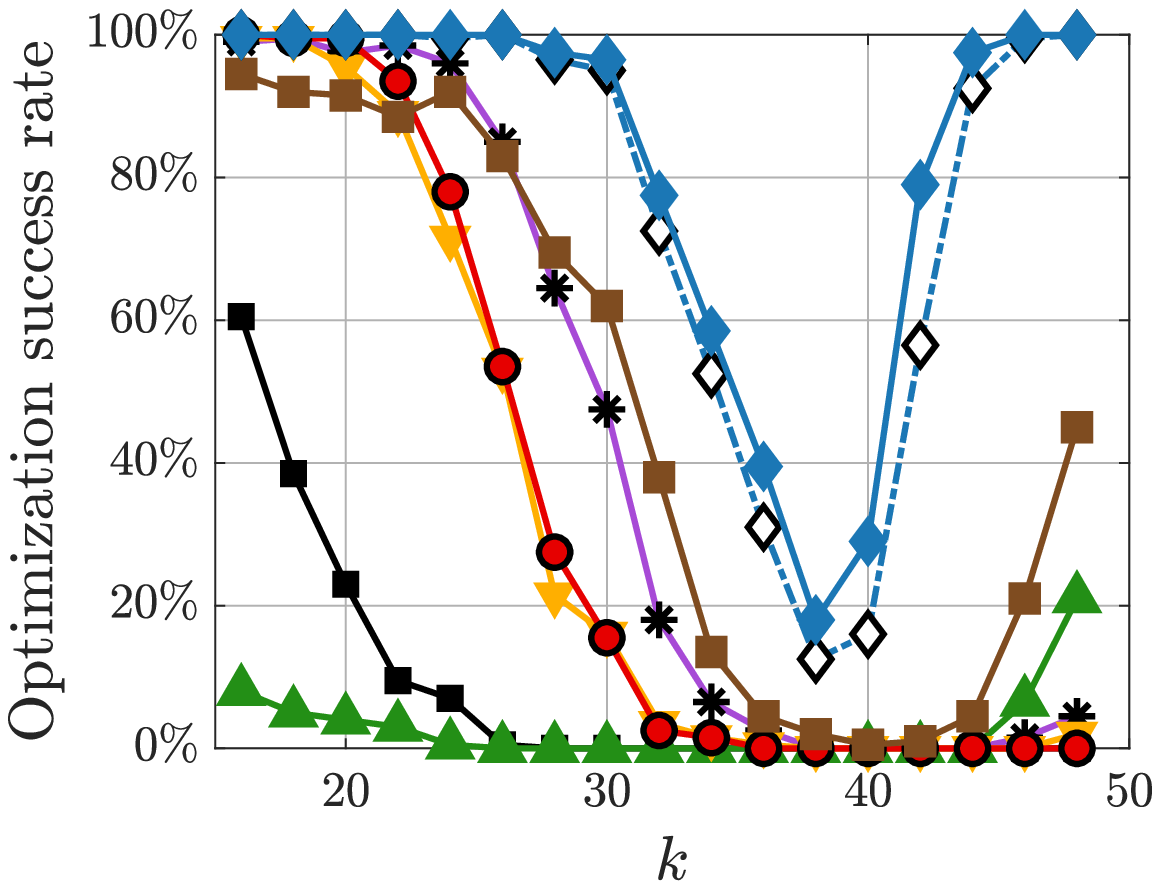}
    \subfighspace
    \includegraphics[width=0.28\textwidth]{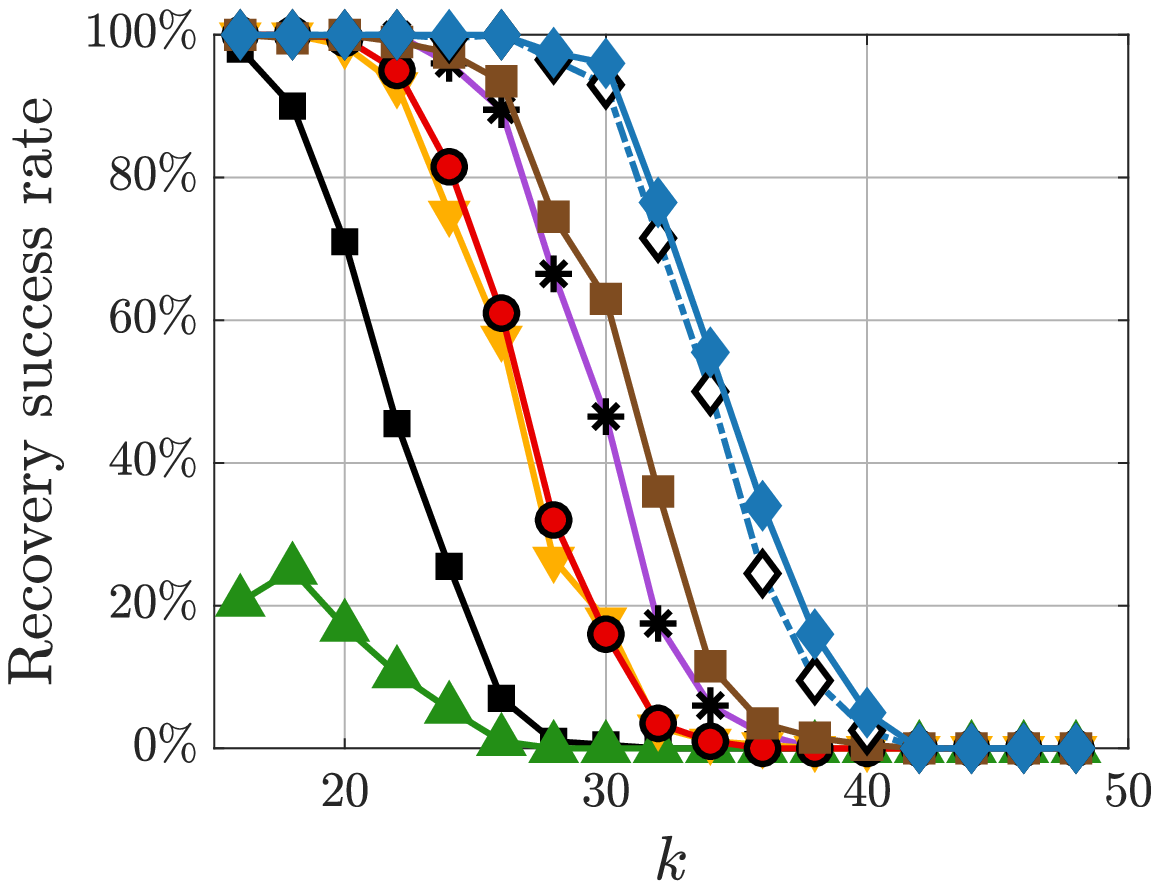}
    \subfighspace
    \includegraphics[width=0.28\textwidth]{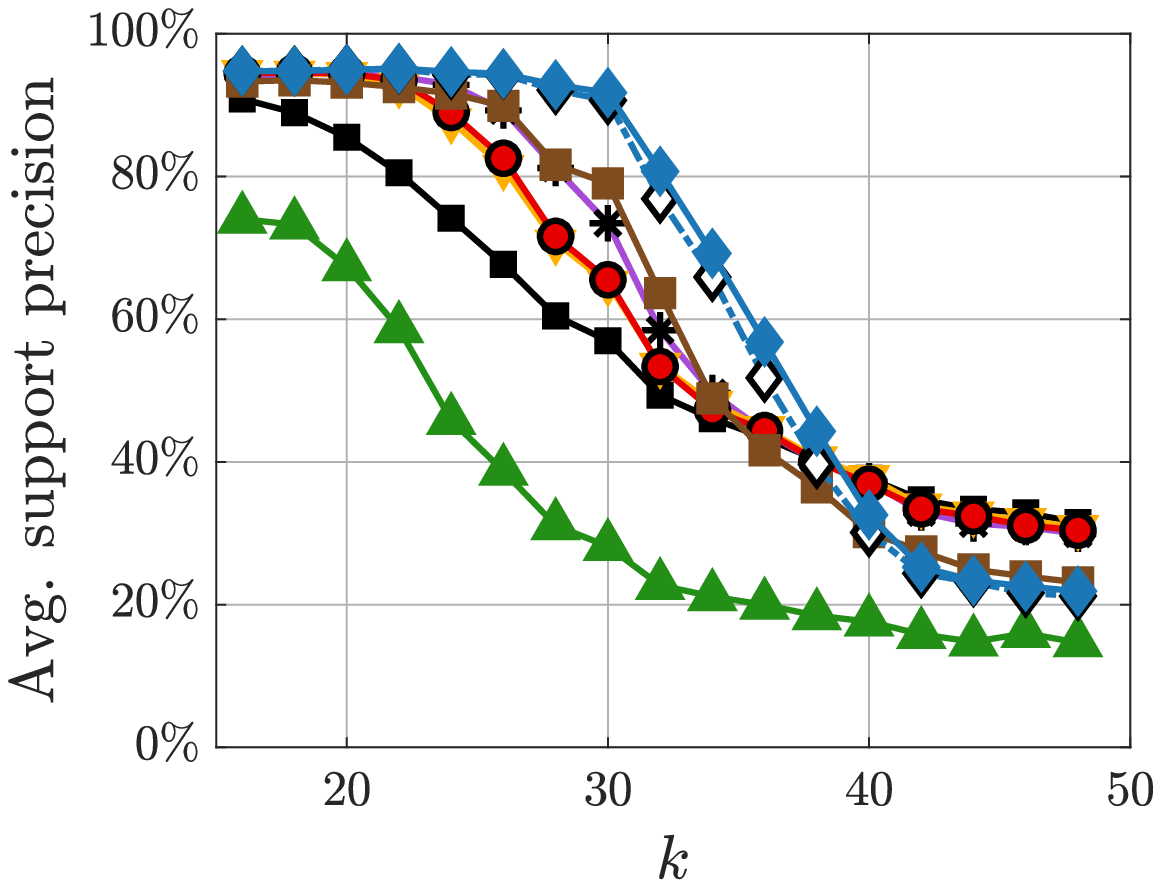}
	\vspace{5pt}
	\newline
    \includegraphics[width=0.65\textwidth]{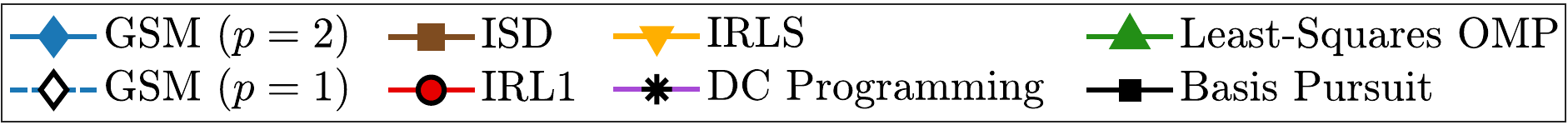}
    \caption{Recovery of $k$-sparse signals in a $\nu=5\%$ noise setting. Top row: Gaussian signal, uncorrelated matrix. Bottom row: Equispaced linear signal, correlated matrix.} 
    \label{fig:success_vs_sparsity_noisy_1}    
\end{figure}

\Cref{fig:success_vs_sparsity_noisy_1} shows the results under 5\% noise ($\nu = 0.05$). Even though noise degraded the performance of all methods, GSM achieved the highest accuracy. In contrast to the noiseless case, here at high values of $k$,
some methods obtain a small residual $\|A\x-\y\|_2$ by an incorrect set of columns, whose linear combination overfits the noisy vector $\y$ -- hence the notable difference between the optimization and recovery success rates. %
Further results appear in \cref{sec:further_numerical_results}. Similar qualitative results (not shown) were obtained with $50\times 1000$ matrices.


\myparagraph{Number of required measurements}
Next, we compare the number of measurements $n$ required for successful recovery in a noiseless setting by GSM and by the next best competitor, ISD, under the same 
setup as in \cite[Figures 3-5]{wang2010sparse}. Specifically, for $d=600$, 
we generate a $k$-sparse vector $\x_0$ whose non-zero entries are  i.i.d. $\mathcal{N}\br{0,1}$,
where $k\in\{8,40,150\}$. Then, for various values of $n$, 
we estimate $\x_0$ from $\y=A\x_0$, where
$A \in \R[n\times d]$ has i.i.d. entries $\mathcal{N}\br{0,1}$ without column normalization. A recovery is considered successful if $\|\hat\x - \x_0\|_1/\|\x_0\|_1 < 10^{-4}$. \Cref{fig:isd} illustrates the recovery success rates over 100 instances for each $n$, with GSM able to successfully recover a sparse signal from fewer measurements than ISD.

\definecolor{color_gsm}{rgb}{0.1059,0.4667,0.7098}
\definecolor{color_isd}{rgb}{0.4980,0.2980,0.1255}

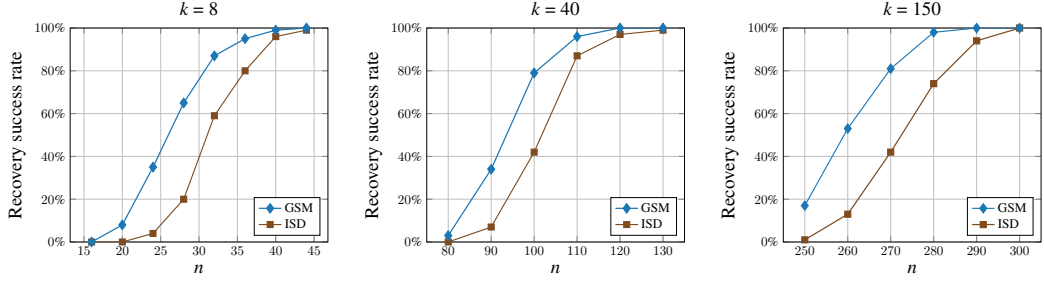
\begin{figure}[t]
\begin{center}
\begin{tikzpicture}[scale=0.5]
\begin{axis}[reverse legend, ymin=0,ymax=1, grid=both, legend cell align={left}, legend pos=south east, title = {\Large \textit{k} = 8}, 
    xlabel = {\Large \textit{n}}, ylabel = {\Large Recovery success rate},
    yticklabel={\pgfmathparse{\tick*100}\pgfmathprintnumber{\pgfmathresult}\%},
    point meta={y*100}]
\addplot [color=color_isd, mark=square*, mark size=2.25pt, line width=0.3mm] table [x=n, y=success_isd] {dat/isd_comparison_k_8.dat};
\addplot [color=color_gsm, mark=diamond*, mark size=3.5pt, line width=0.33mm] table [x=n, y=success_gsm] {dat/isd_comparison_k_8.dat};
\legend{ISD,GSM}
\end{axis}
\end{tikzpicture}
\subfighspace
\begin{tikzpicture}[scale=0.5]
\begin{axis}[reverse legend, ymin=0,ymax=1,grid=major, legend cell align={left}, legend pos=south east, title = {\Large \textit{k} = 40}, 
    xlabel = {\Large \textit{n}}, ylabel = {\Large Recovery success rate},
    yticklabel={\pgfmathparse{\tick*100}\pgfmathprintnumber{\pgfmathresult}\%},
    point meta={y*100}]
\addplot [color=color_isd, mark=square*, mark size=2.25pt, line width=0.3mm] table [x=n, y=success_isd] {dat/isd_comparison_k_40.dat};
\addplot [color=color_gsm, mark=diamond*, mark size=3.5pt, line width=0.33mm] table [x=n, y=success_gsm] {dat/isd_comparison_k_40.dat};
\legend{ISD,GSM}
\end{axis}
\end{tikzpicture}
\subfighspace
\begin{tikzpicture}[scale=0.5]
\begin{axis}[reverse legend, ymin=0,ymax=1,grid=major, legend cell align={left}, legend pos=south east, title = {\Large \textit{k} = 150}, 
    xlabel = {\Large \textit{n}}, ylabel = {\Large Recovery success rate},
    yticklabel={\pgfmathparse{\tick*100}\pgfmathprintnumber{\pgfmathresult}\%},
    point meta={y*100}]
\addplot [color=color_isd, mark=square*, mark size=2.25pt, line width=0.33mm] table [x=n, y=success_isd] {dat/isd_comparison_k_150.dat};
\addplot [color=color_gsm, mark=diamond*, mark size=3.5pt, line width=0.3mm] table [x=n, y=success_gsm] {dat/isd_comparison_k_150.dat};
\legend{ISD,GSM}
\end{axis}
\end{tikzpicture}
\caption{Recovery success rate of GSM and ISD for a $k$-sparse Gaussian signal in $\R[d]$, with $d=600$, as a function of the number of measurements $n$. }
\label{fig:isd}
\end{center}
\end{figure}

\subsection{Comparison with MIP-based methods}\label{sec:comparison_mip}
Finally, we compare the power-2 variant of GSM with two recent mixed integer programming based approaches,
considered state-of-the-art in solving
\eqref{pr:P0}: (1) The cutting plane method  \cite{bertsimas2020sparse}, and (2) the coordinate descent and local combinatorial search method  \cite{hazimeh2020fast}.
We did not run the first method ourselves but rather, as described below, compare our results to those reported in their work, under identical settings. 
We  ran the second method using the authors' L0learn R package. We tested it both with and without local combinatorial search, denoted by CD+Comb and CD, respectively.

First, we replicate the two settings of Figure 2(right) 
and Figure 4(left) in \cite{bertsimas2020sparse}. 
Specifically, for varying values of $n$, we generate an $n \times d$ matrix $A$ with uncorrelated i.i.d. entries $\mathcal{N}\br{0,1}$ without column normalization. 
We then generate a $k$-sparse vector $\x_0 \in \R[d]$ with $k=10$, whose nonzero entries are chosen randomly from $\pm 1$ and placed at random indices. 
We observe $\y = A\x_0 + \e$ where $\e$ has i.i.d. $\mathcal{N}\br{0,1}$ entries, normalized such that $\textup{SNR} = 1/\nu^2 = \norm{A\x_0}^2_2 / \norm{\e}^2_2$. We considered the following two settings as in 
\cite{bertsimas2020sparse}: (a) $d=15000$, SNR=400; and (b) $d=5000$, SNR=9. 
Both the cutting plane method and our GSM received as input the correct value of $k$. In contrast,  
 the coordinate descent method, returns a set of solutions with different sparsity levels.
 We gave their method a slight advantage by choosing the solution whose support $\hat{S}$ is the closest to the ground-truth support $S_0$ in terms of the F-score: $2 \frac{\abs{\hat{S} \cap S_0}}{\abs{\hat{S}} + \abs{S_0}}$.

As in \cite{bertsimas2020sparse}, we measure the quality of a solution $\xhat$ by its support precision, $\abs{\hat{S}\cap S_0}/k$. 
\Cref{fig:bertsimas_tests} compares the performance of GSM, CD and CD+Comb, averaged over 200 random instances for each value of $n$. 
To achieve a precision comparable to GSM, CD+Comb required about $20\%$
more observations. 
In terms of run-time, our method, implemented in Matlab, 
took on average less than 3 minutes to complete a single run.
The CD/CD+Comb method, partly implemented in C++, is significantly faster and took less than 1 second per run.

Given the NP-hardness of \eqref{pr:P0}, in their original paper, \cite{bertsimas2020sparse}
did not run their exact cutting plane method till it found the optimal solution, but rather capped each run to at most 10 minutes. With this time limit, 
their performance was much worse than that of GSM. 
For example, in setting (a)
they achieved an average support precision smaller than 30\% at $n=100$. In comparison, our GSM achieved an 88\% precision.

\begin{figure}[t]
    \centering
    \subfloat[$d = 15000$, $\textup{SNR} = 400$]{\includegraphics[width=0.3\textwidth]{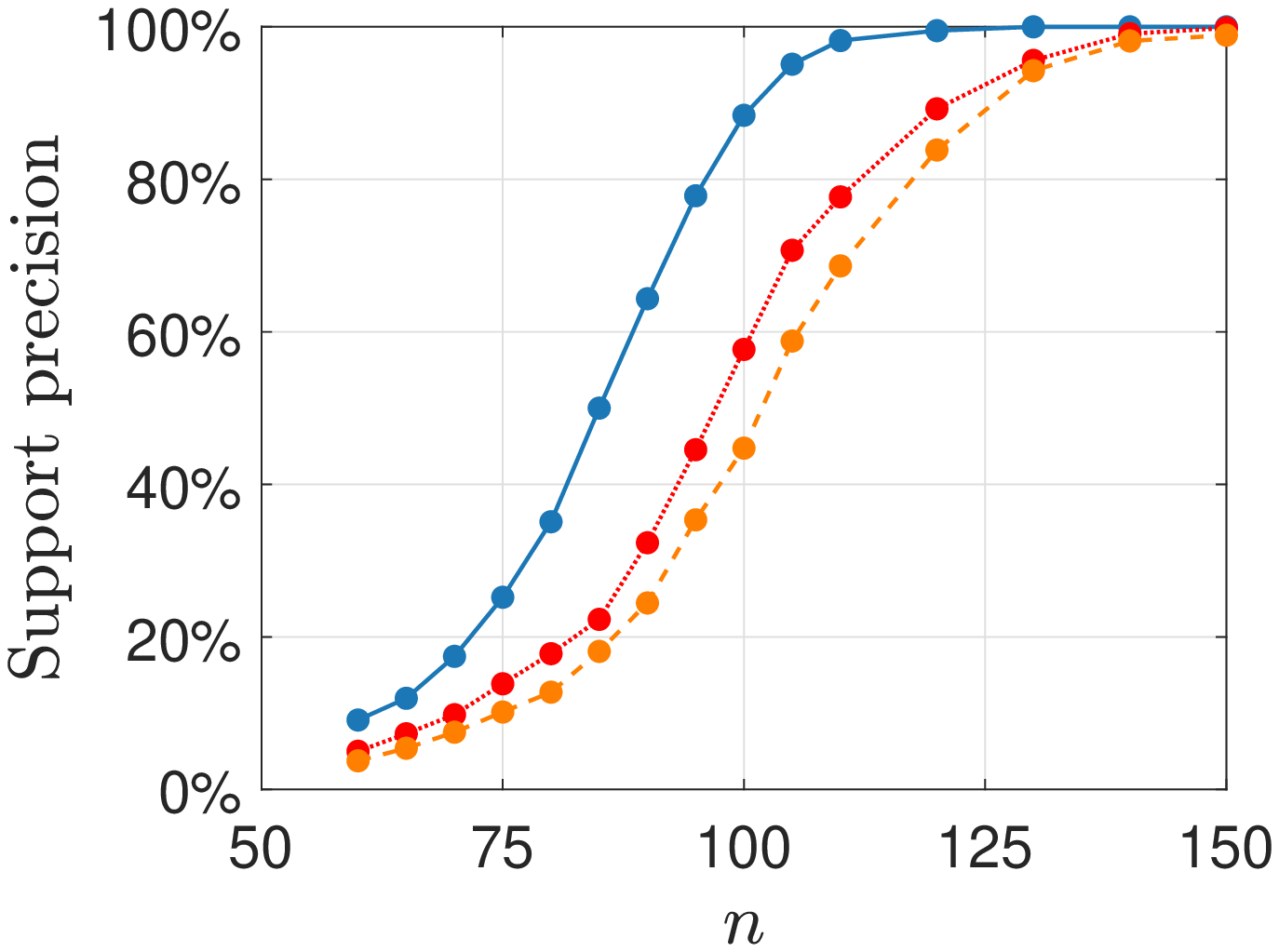}}
    \subfighspace
    \subfloat[$d = 5000$, $\textup{SNR} = 9$]{\includegraphics[width=0.3\textwidth]{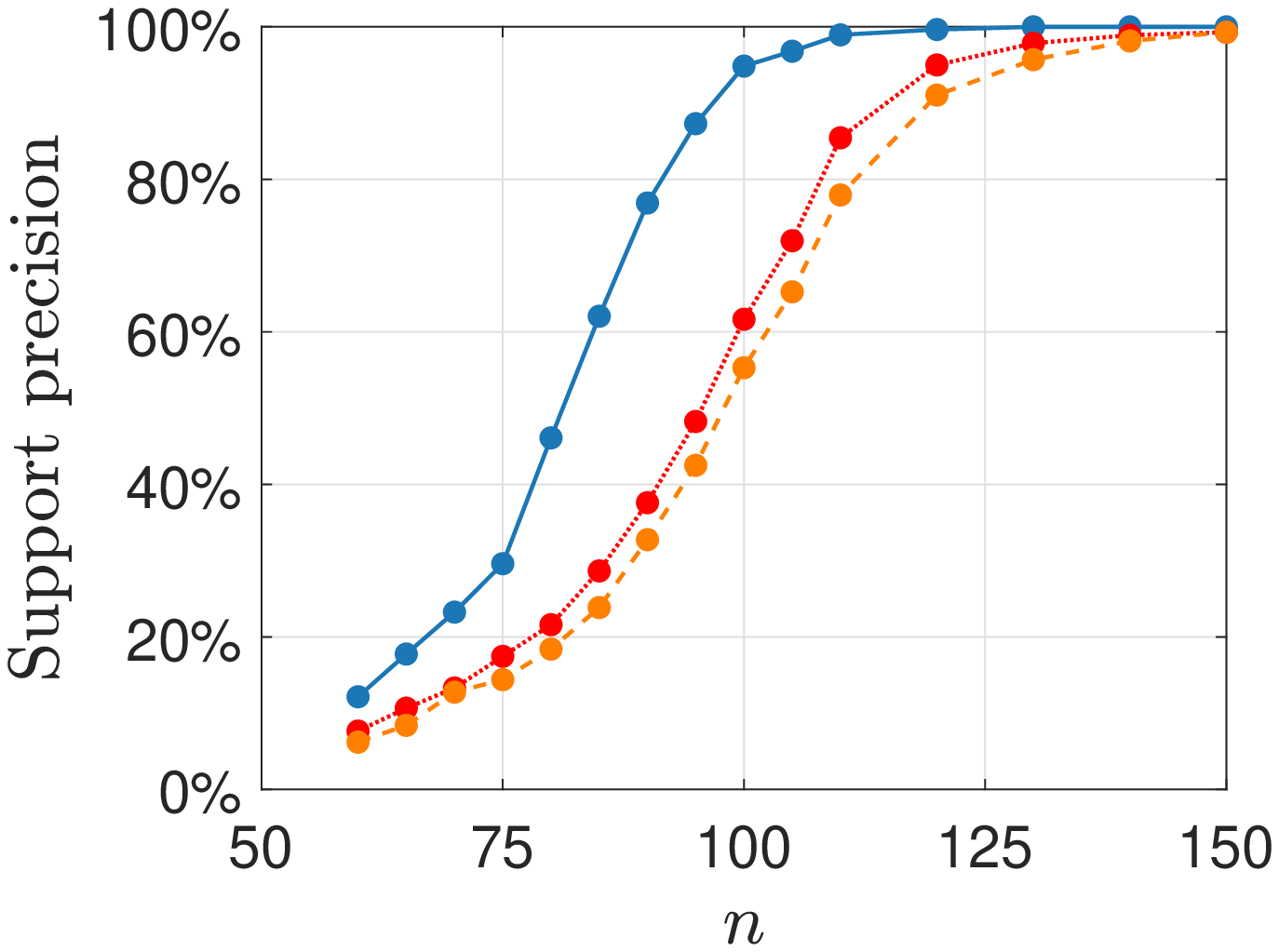}}
	\vspace{5pt}
	\newline
    \includegraphics[width=0.3\textwidth]{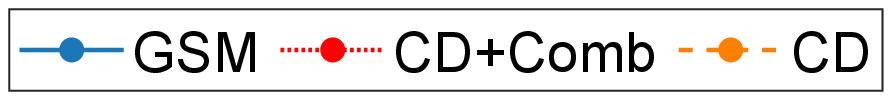}
    \caption{Support precision in recovering a $k$-sparse signal, $k=10$, with an uncorrelated matrix ($\rho = 0$). Settings (a) and (b) correspond to Figure 2(right) and Figure 4(left) in \cite{bertsimas2020sparse}. }
    \label{fig:bertsimas_tests}    
\end{figure}

Finally, 
to highlight the performance achievable by GSM versus that of CD/CD+Comb, we considered the following quite challenging setup: For $d=20000$, $k=50$ and varying values of $n$, a $n \times d$ Gaussian matrix $A$ was generated with correlated columns, with coefficient $\rho = 0.5$. 
The observed vector $\y = A\x_0 + \e$ was generated with random Gaussian noise $\e$.
Here we followed the definition of SNR of \cite{hazimeh2020fast}, and normalized $\e$ such that $\textup{SNR} = 1/\nu^2 = \mathbb{E}_{A} \brs{\norm{A\x_0}^2_2} / \mathbb{E}_{\e} \brs{\norm{\e}^2_2}$. 
Three different settings were tested: (a) $\x_0$ equispaced $\pm 1$, SNR=400; (b) $\x_0$ equispaced $\pm 1$, SNR=9; and (c) $\x_0$ equispaced linear, SNR=9.

In these simulations, the true sparsity level $k$ was unknown. 
As in \cite{hazimeh2020fast}, we also observe a separate 
random design validation set of the same size, $\tilde{\y}=\tilde{A} \x_0+\tilde{\e}$.
We ran our GSM method for input sparsity $k\in[1,60]$, and choose the solution with smallest 
prediction error on the validation set. For CD/CD+Comb, we passed the ground-truth $k$ as an upper limit on the support size, and similarly took the solution with smallest validation prediction error. 

Similar to \cite{hazimeh2020fast}, we also calculated the false positive rate  $\frac{\abs{\hat{S} \cap S_0^c}}{\abs{\hat{S}}}$ and the expected prediction error  $\frac{\mathbb{E}_{A}\brs{\norm{A\xhat - \y}_2^2}}{\mathbb{E}_{A}\brs{\norm{\y}_2^2}}$. 
\Cref{fig:hazimeh_tests} shows the results, averaged over 100 instances for each $n$. As seen in the individual plots, GSM achieved superior recovery. The runtime of our method was approximately 10 minutes for each combination of $n,k$.


\begin{figure}[t]  
    \centering
    \subfloat{\includegraphics[width=0.3\textwidth]{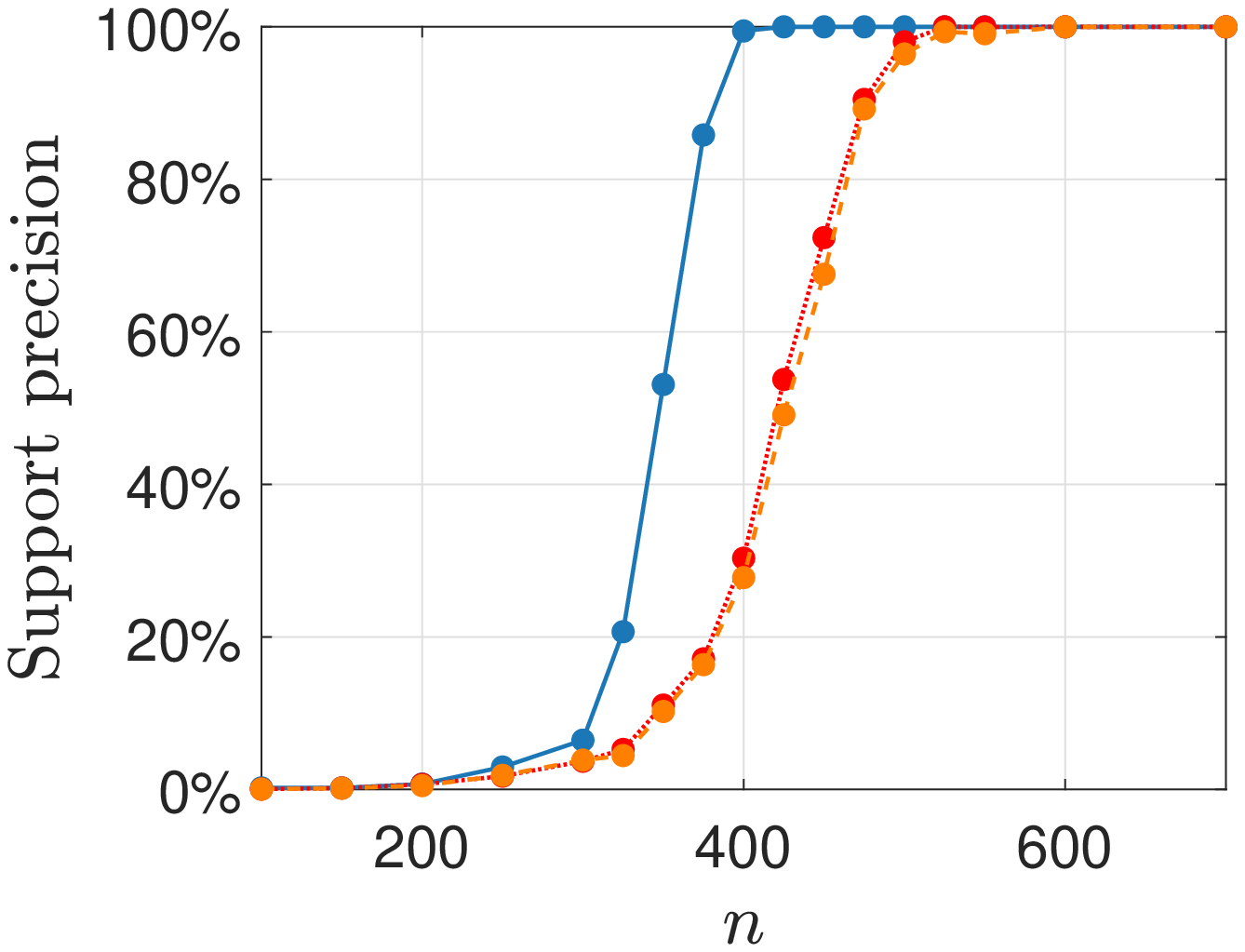}}
    \subfighspace
    \subfloat{\includegraphics[width=0.3\textwidth]{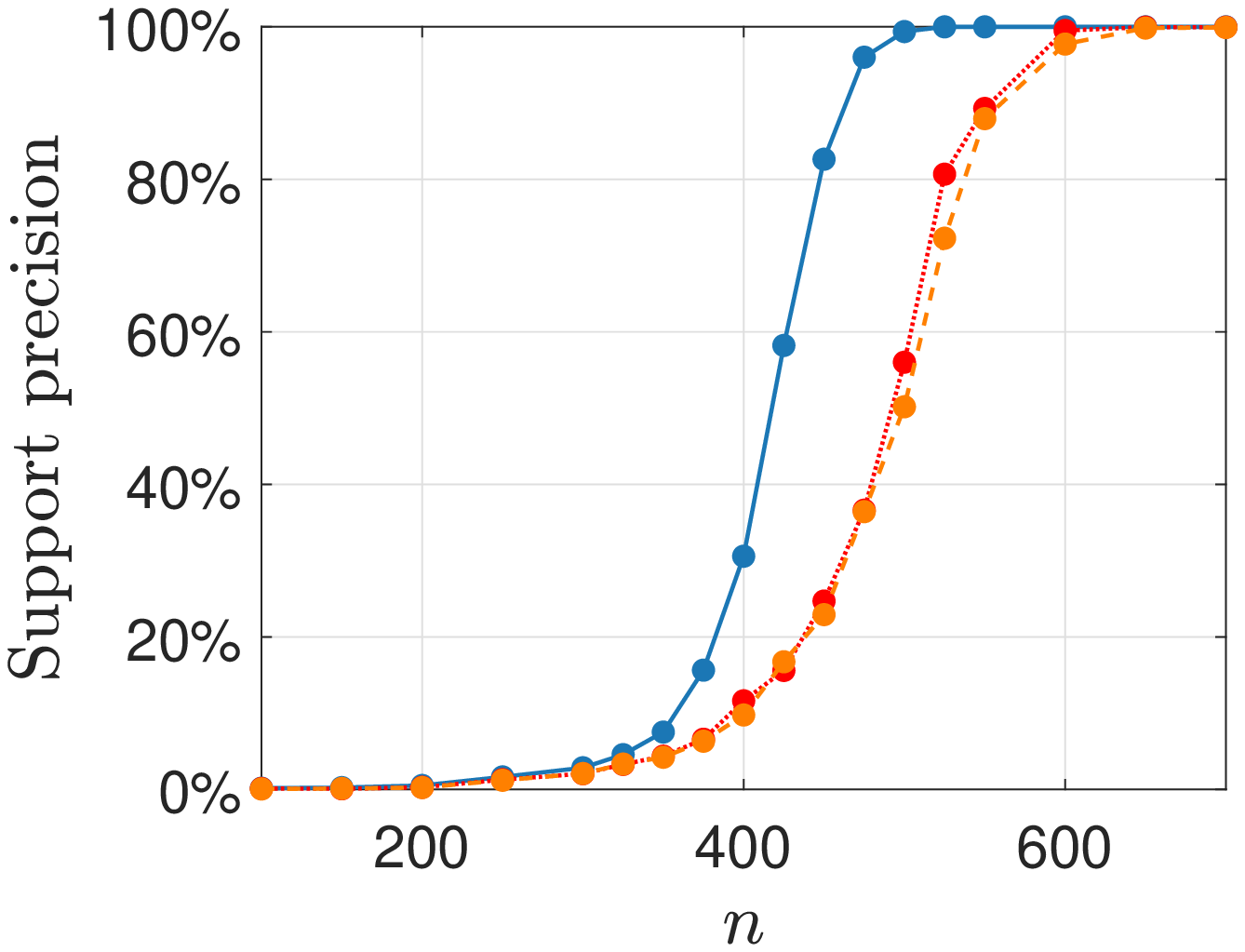}}
    \subfighspace
    \subfloat{\includegraphics[width=0.3\textwidth]{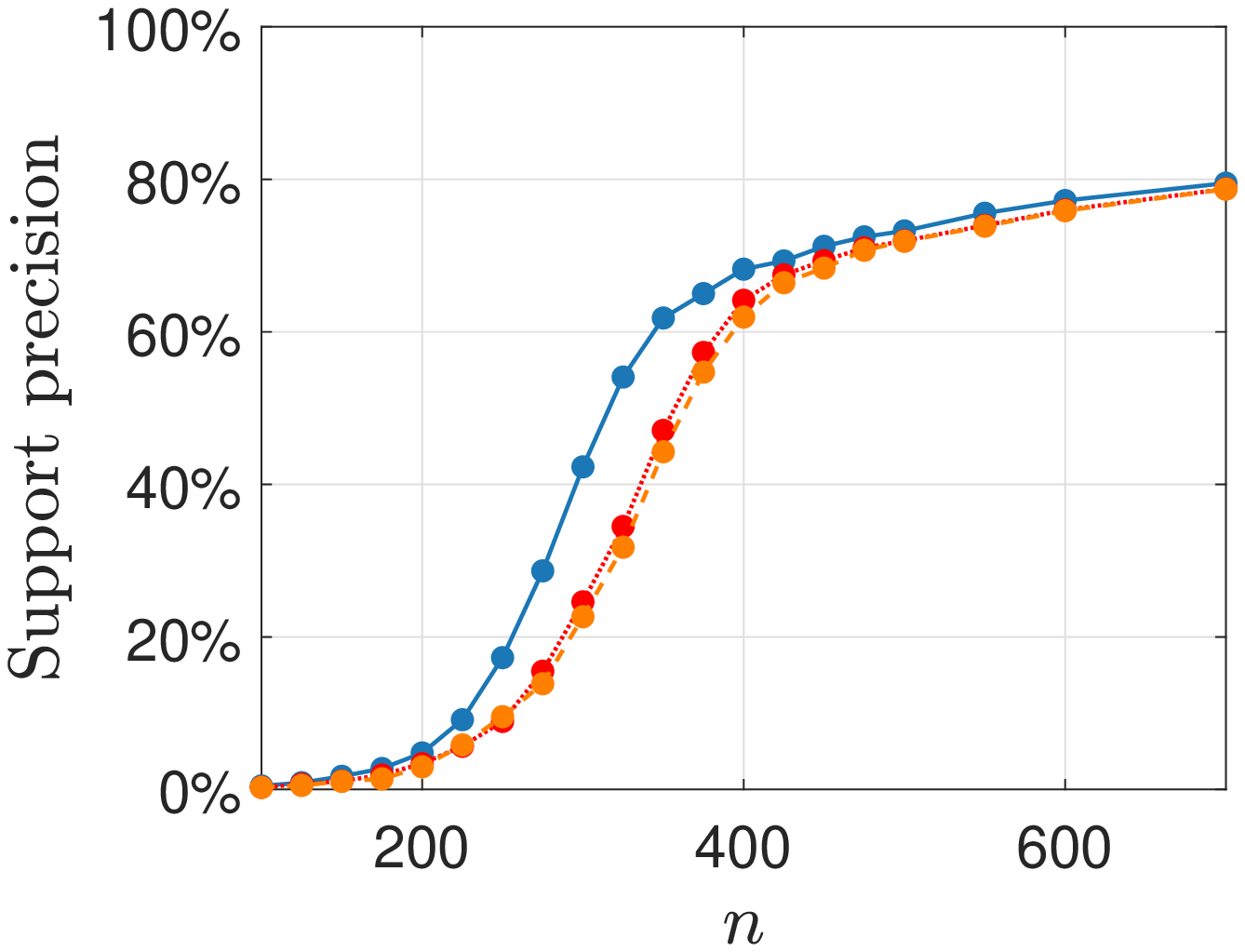}}
    \subfighspace
    \\
    \subfloat{\includegraphics[width=0.3\textwidth]{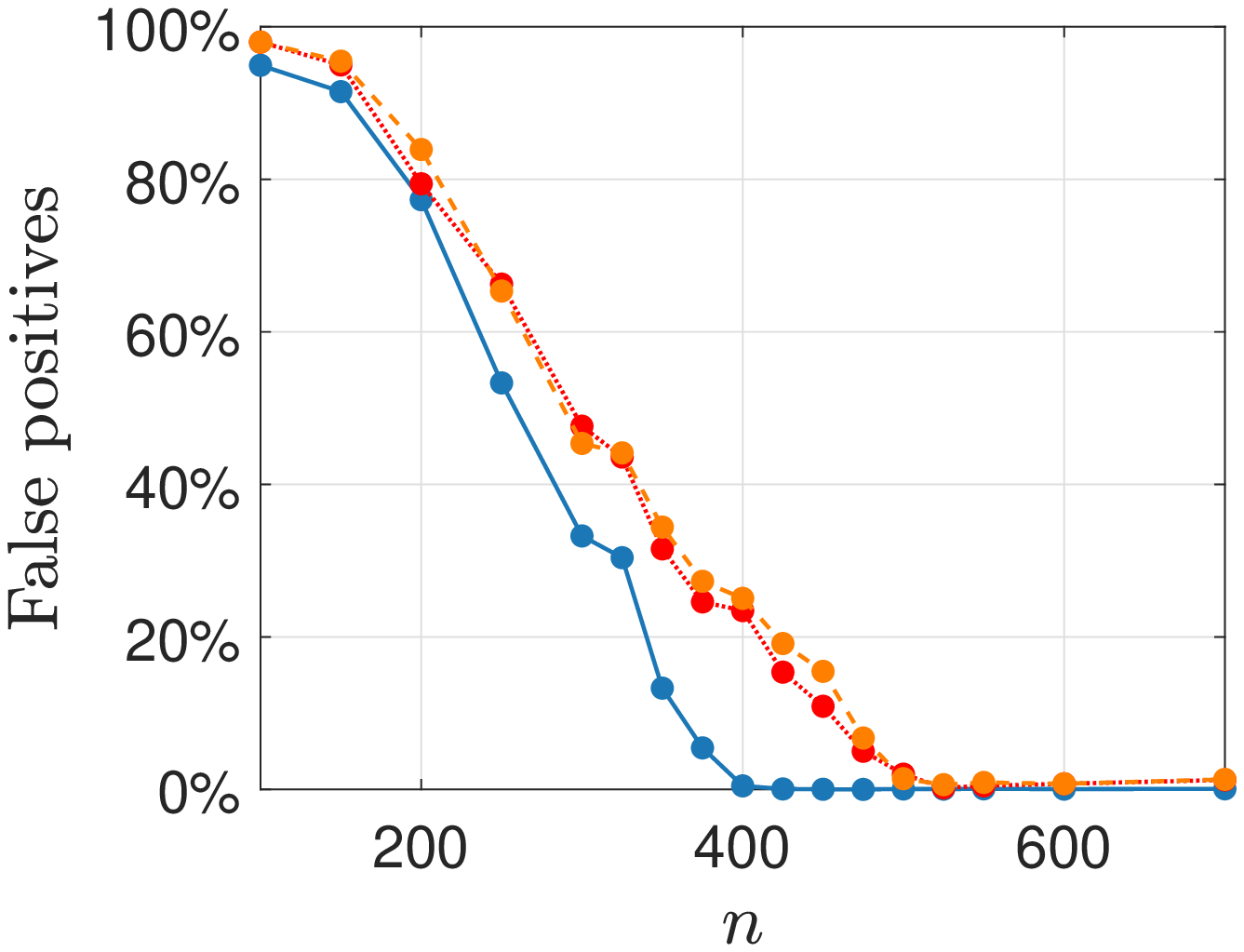}}
    \subfighspace
    \subfloat{\includegraphics[width=0.3\textwidth]{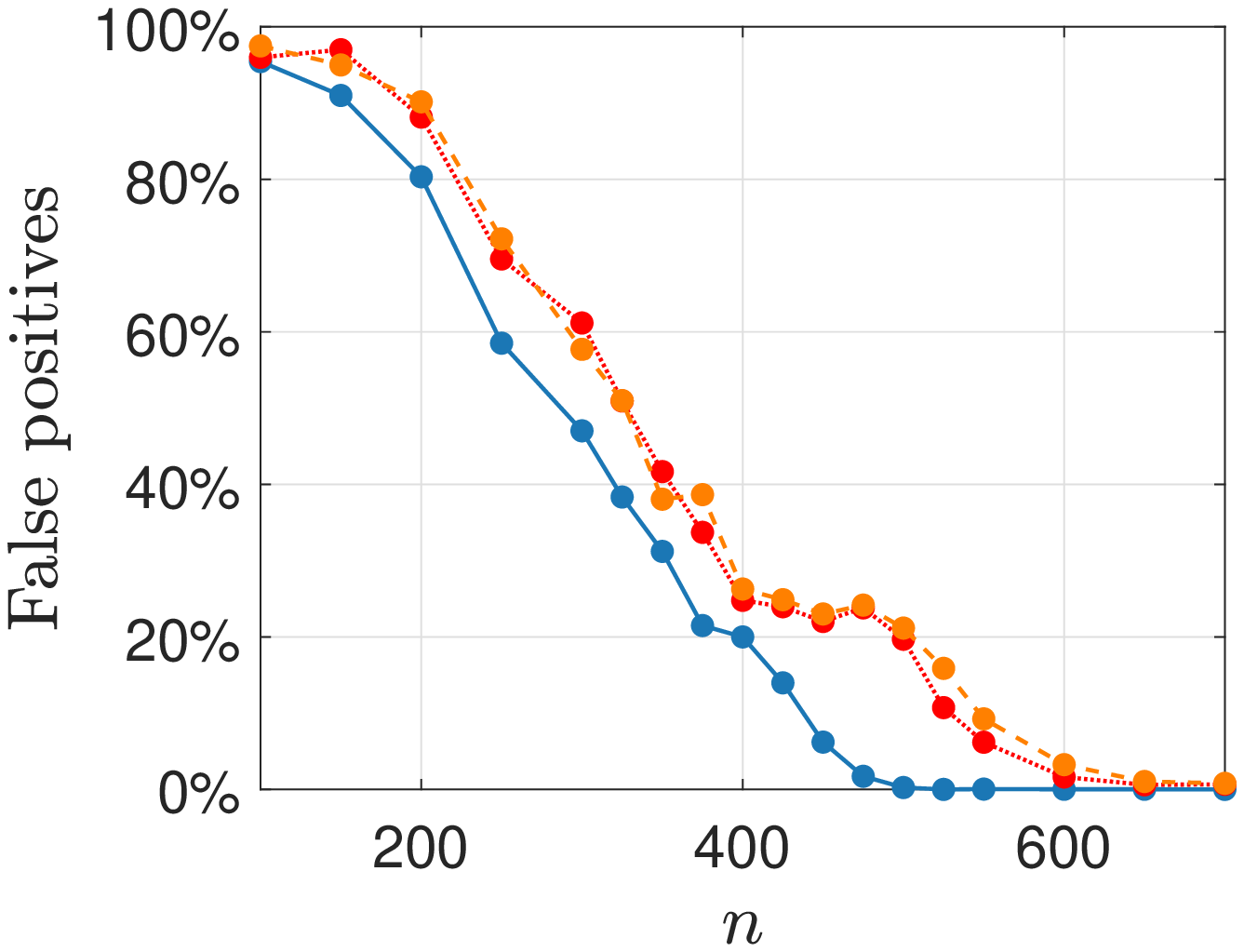}}
    \subfighspace
    \subfloat{\includegraphics[width=0.3\textwidth]{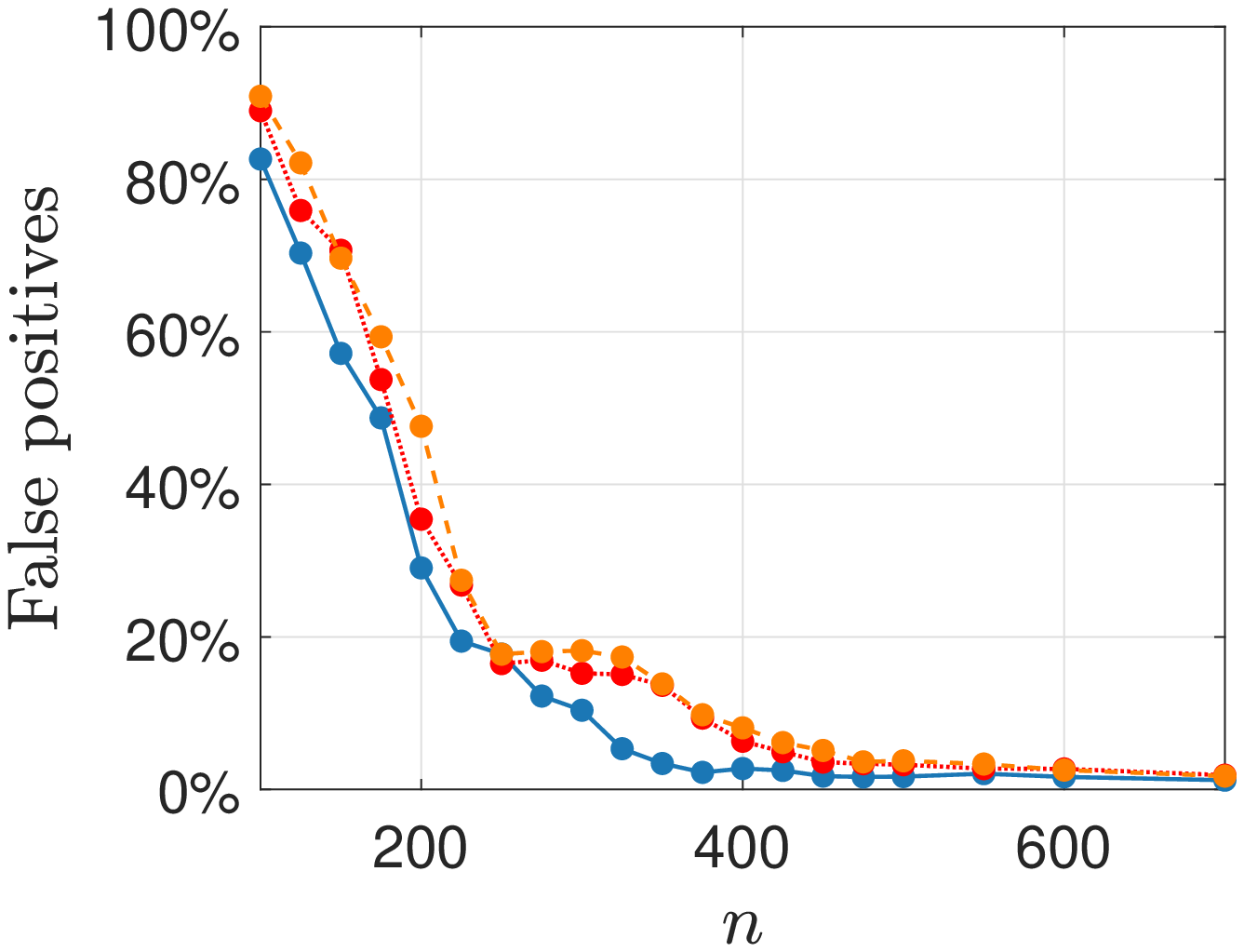}}
    \subfighspace
    \\
    \subfloat[{\footnotesize $\textup{SNR} = 400$, entries $\pm 1$}]{\includegraphics[width=0.3\textwidth]{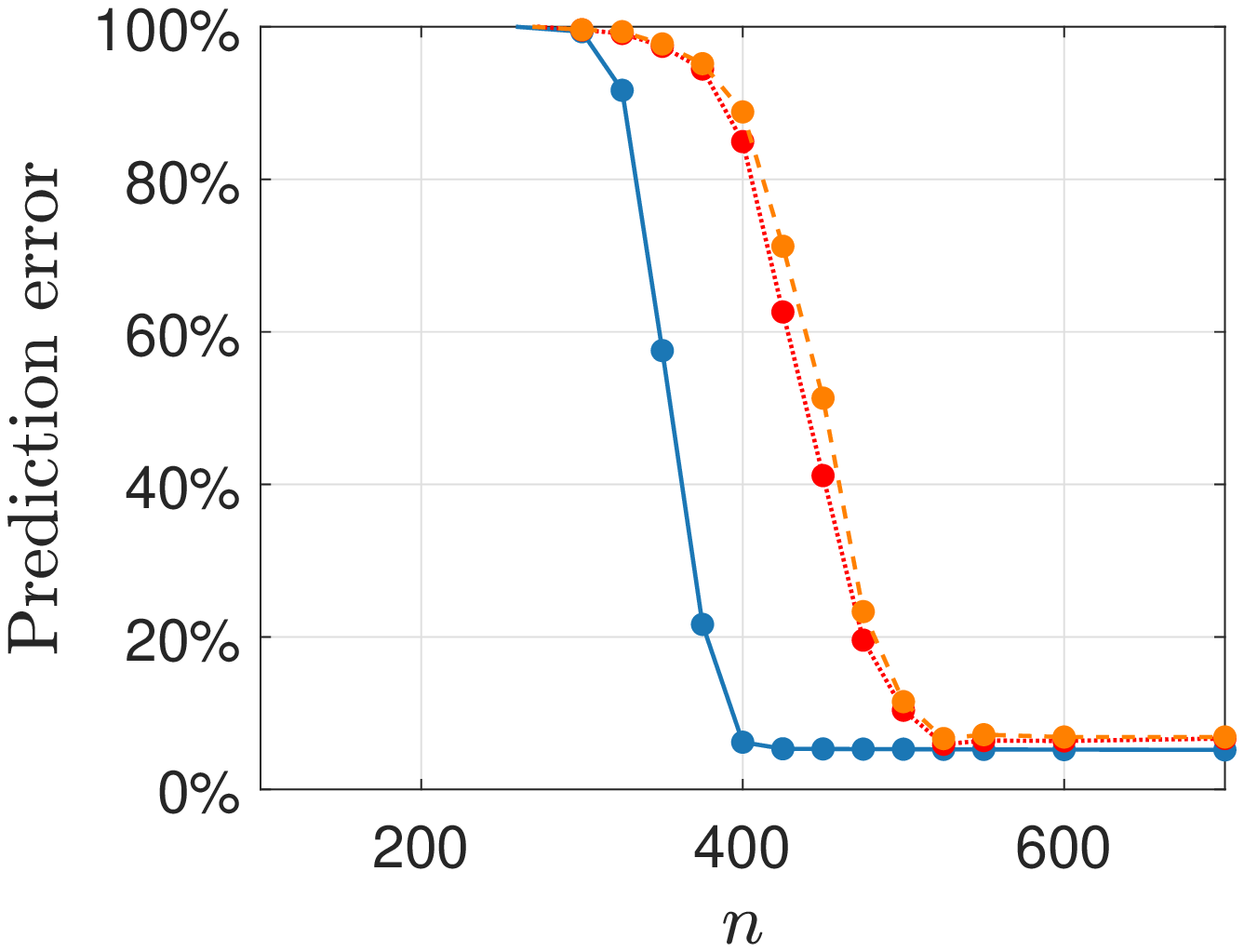}}
    \subfighspace
    \subfloat[{\footnotesize $\textup{SNR} = 9$, entries $\pm 1$}]{\includegraphics[width=0.3\textwidth]{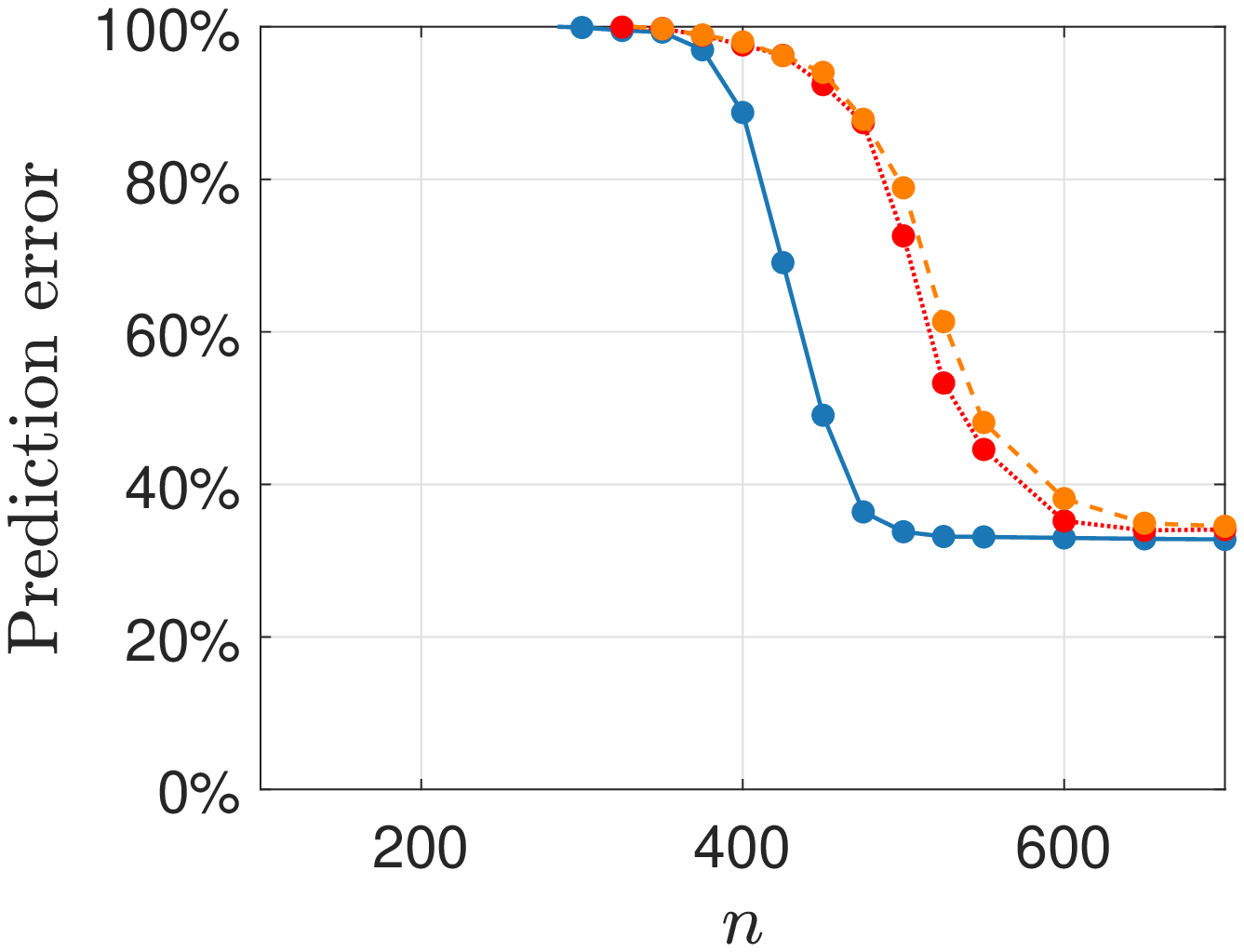}}
    \subfighspace
    \subfloat[{\footnotesize $\textup{SNR} = 9$, linear magnitudes}]{\includegraphics[width=0.3\textwidth]{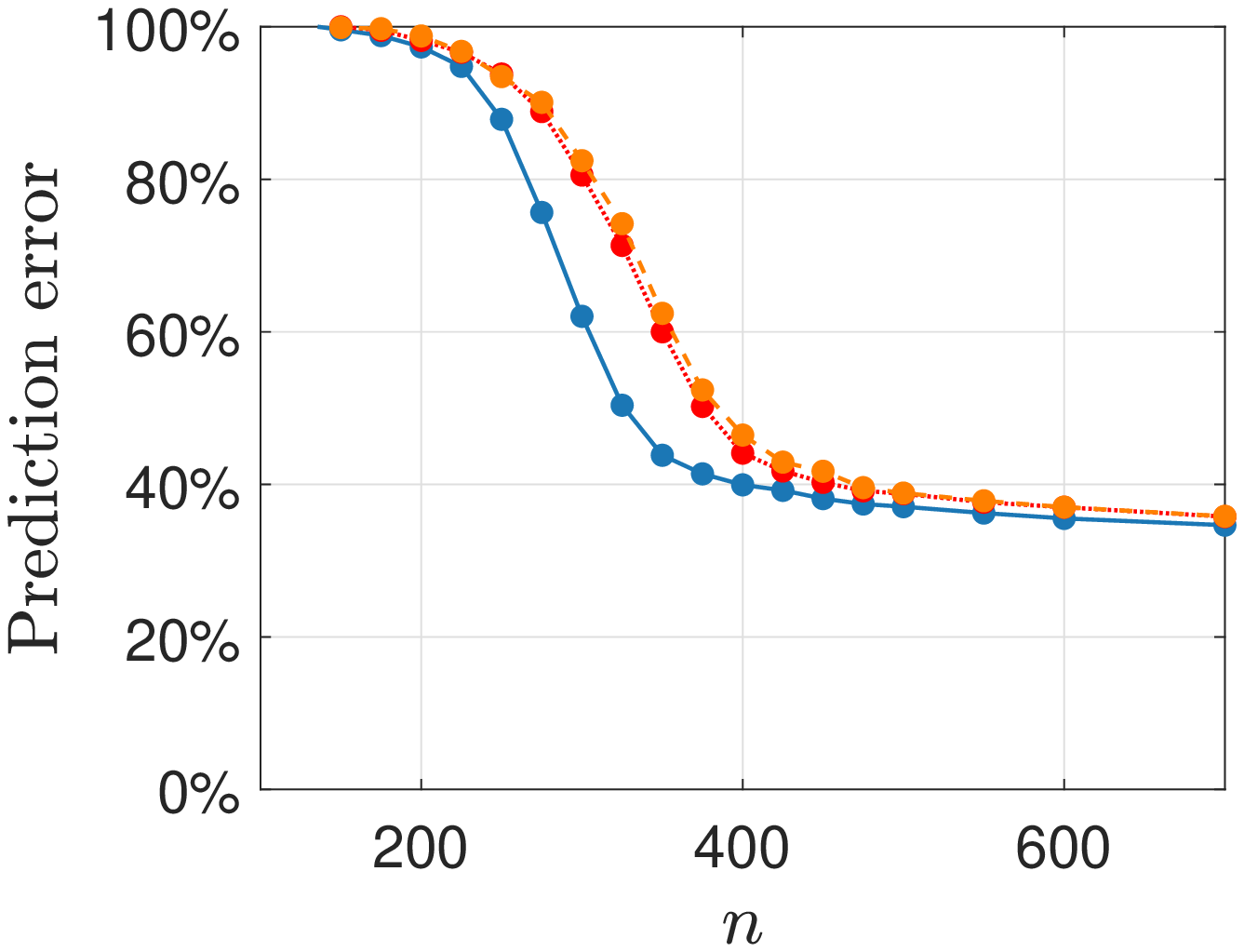}}
    \subfighspace
	\vspace{5pt}
	\newline
    \includegraphics[width=0.3\textwidth]{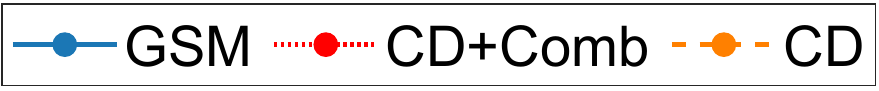}
    \caption{Sparse signal recovery with $d=20000$, $k=50$ and correlation parameter $\rho = 0.5$. Top to bottom: Support precision, false positive rate, prediction error. The vectors in the left and center columns are with nonzero entries $\pm1$. In the right column, the nonzeros are with linear magnitudes.}
    \label{fig:hazimeh_tests}    
\end{figure}

\myparagraph{Summary}
%
%
Via several simulations, we illustrated the competitive performance of GSM in solving \cref{pr:P0} and recovering sparse signals. This improved performance, nonetheless, comes at a computational cost. 
Each run of GSM typically considered up to 30 different values of $\lambda$. For each $\lambda$, solving \cref{pr:P2l} by homotopy 
required several hundreds of values of $\gamma$. For each $\gamma$, solving \cref{pr:P2lg} often took three MM iterations. Hence, 
one run of GSM solved several thousands 
of weighted $\ell_1$ problems~\cref{pr:P2lw} or \cref{pr:P1lw}. 
We remark that optimizations of \cref{pr:P2l} over different values of $\lambda$ can be done in parallel. Further speedups may be possible by using specialized 
weighted $\ell_1$ solvers that can be warm-started at the solutions of previous problem instances.

\section*{Acknowledgments}
We thank Arian Maleki for interesting discussions. BN is the incumbent of the William Petschek professorial chair of mathematics. 
Part of this work was done while BN was on sabbatical at the Institute for Advanced Study
at Princeton. BN would like to thank the IAS and the Charles Simonyi endowment for
their generous support. Finally, we thank the associate editor and the reviewers for their comments, which greatly improved our manuscript.

\bibliographystyle{plain}
\bibliography{sparse_approx_gsm_arxiv}

\appendix
\begin{appendices}

\section[Appendix A. Further numerical results]{Further numerical results}
\label{sec:further_numerical_results}

To further illustrate the improved performance of GSM, \cref{fig:cdf_noisy_1} 
shows, for selected values of $k$, the cumulative distribution of the 
normalized recovery error
$$\mbox{G}(t)=\mbox{P}\brs{\|\hat\x-\x_0\|_1/\|\x_0\|_1 \leq t}$$ 
under 5\% noise ($\nu=0.05$). These results show that the better success rates  
of GSM are not sensitive to the specific threshold that defines a successful 
recovery.

\begin{figure}[t]
    \centering
    \begin{minipage}[c]{0.7\textwidth}
    \subfloat[$k=30$]{\includegraphics[scale=0.4]{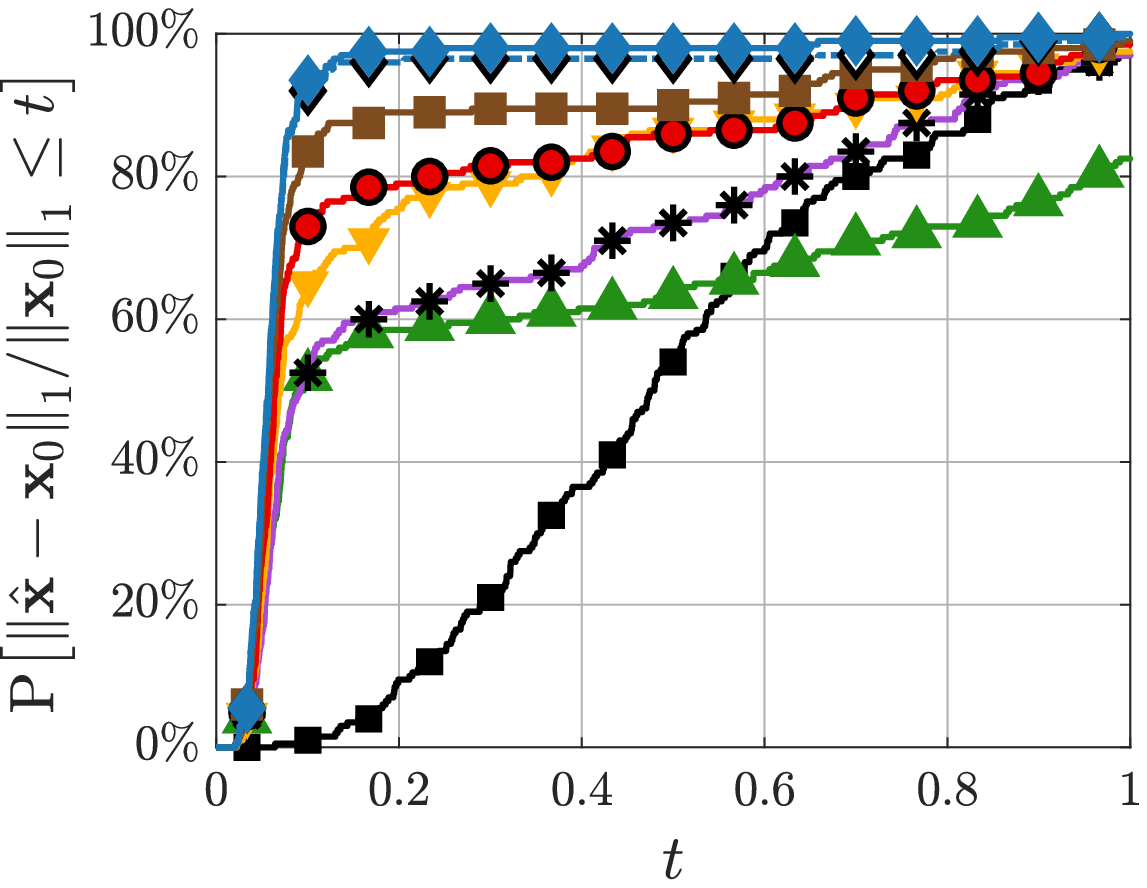}}
    \subfighspace
    \subfloat[$k=32$]{\includegraphics[scale=0.4]{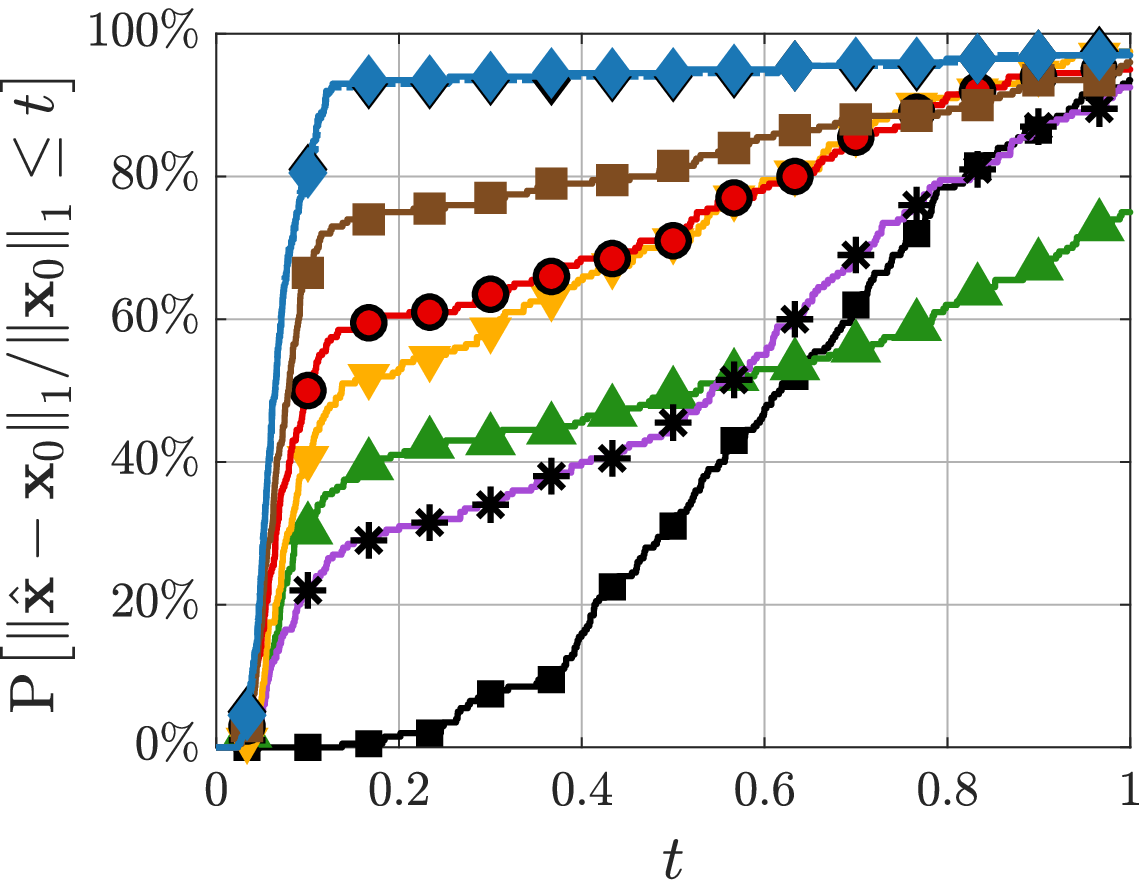}}
    \\
    \subfloat[$k=34$]{\includegraphics[scale=0.4]{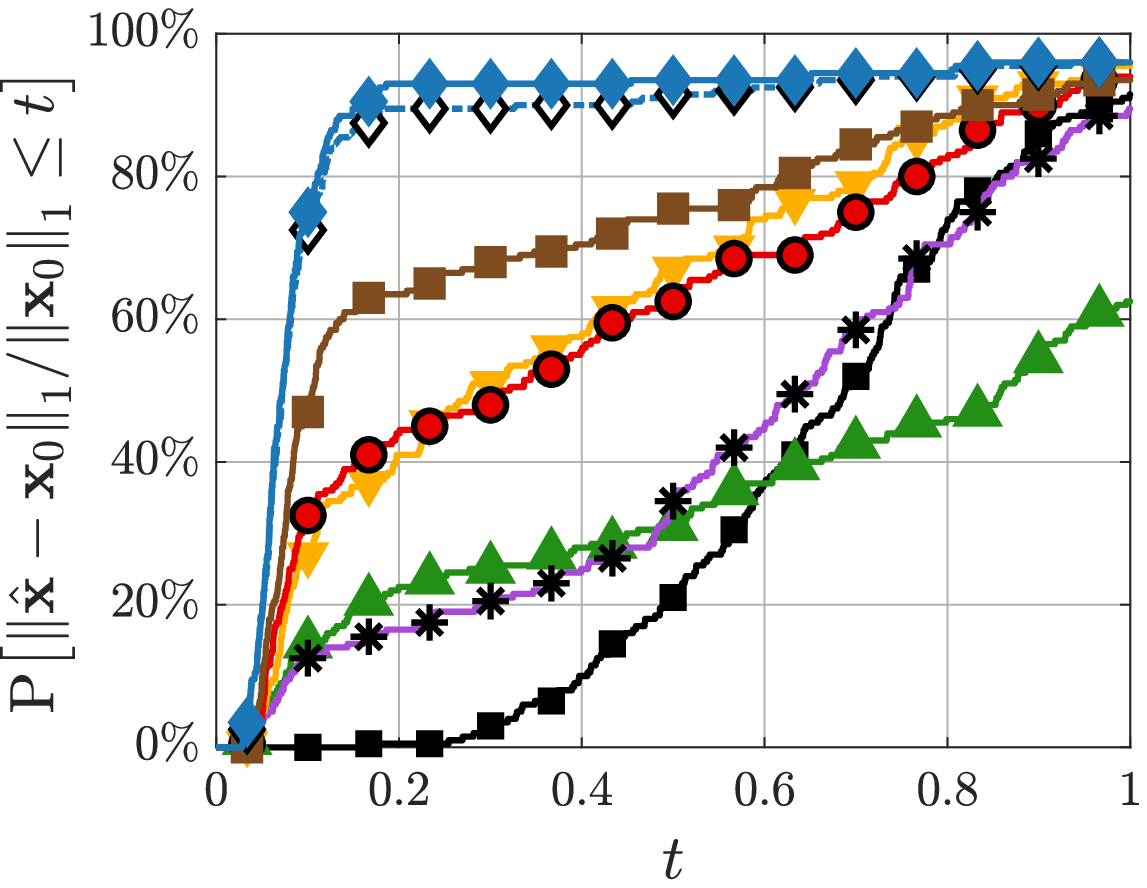}}
    \subfighspace
    \subfloat[$k=36$]{\includegraphics[scale=0.4]{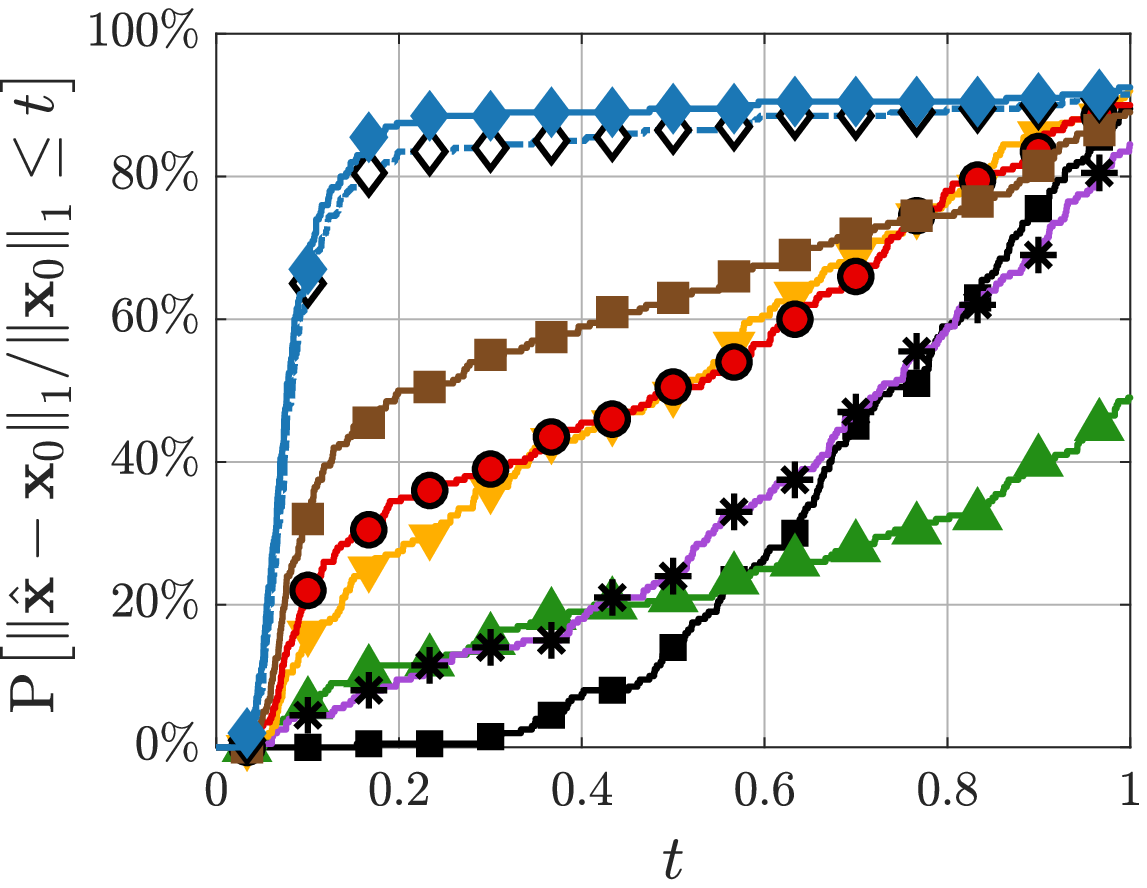}}
    %
    \end{minipage}
    \begin{minipage}[c]{0.25\textwidth}
    \subfloat{\includegraphics[scale=0.5]{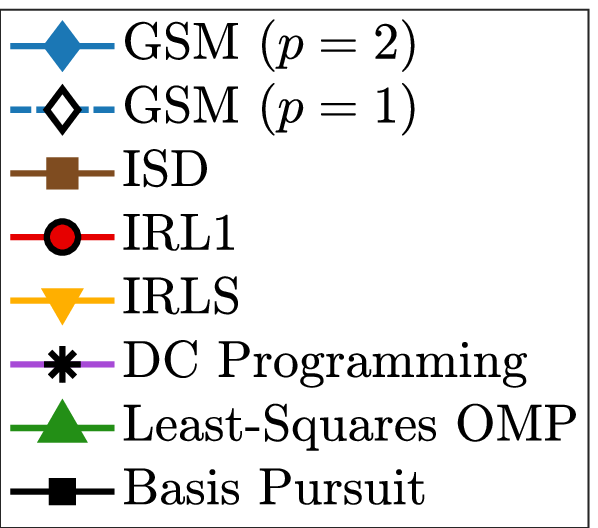}}
    \end{minipage}
    \caption{Cumulative distribution of the relative recovery error 
    $\norm{\xhat-\x_0}_1 / \norm{\x_0}_1$ under 5\% noise ($\nu = 0.05$). 
    Gaussian signal, uncorrelated $100 \times 800$ matrix.} 
    \label{fig:cdf_noisy_1}    
\end{figure}

\section[Appendix B. Proofs]{Proofs} 
\label{appendix:proofs}

\subsection{Theory for the trimmed lasso}
\label{sec:proofs_theory_tls_appendix}

%
%

In our proofs, we shall use the following inequality 
\eq[eq:beta_inequality]{\ensuremath{\norm{A\vvec}_2 \leq \lambdalarge 
\norm{\vvec}_1 \qquad \forall \vvec \in \R[d],}}
where $\lambdalarge$ is as in \cref{eq:def_lambda_b}. This inequality follows 
directly from the triangle inequality.

\begin{proof}[Proof of \cref{thm:large_lambda_p2l}]
    \label{proof:large_lambda_p2l}
    Let $\xstar$ be a local minimum of $\Ftwol$. Let $\Lambda \subset \brs{d}$ 
    be an index set of the $d-k$ smallest-magnitude entries of $\xstar$, 
    breaking ties arbitrarily. Define $\w \in \R[d]$ by 
    \eq{\ensuremath{
        w_i = \left\{
        \begin{array}{lc}
        1 & i\in \Lambda\\
        0 & \mbox{otherwise.}
        \end{array}
        \right.
    }}    
    Let $\Ftwolw \ofx \eqdef \frac{1}{2}\norm{A\x-\y}_2^2 + \lambda 
    \inprod{\w}{\abs{\x}}.$ 
    Then $\Ftwol \br{\xstar} = \Ftwolw \br{\xstar}$. Moreover, for any $\x \in 
    \R[d]$, $\Ftwol \br{\x} \leq \Ftwolw \br{\x}$. 
    Hence, the assumption that $\xstar$ is a local minimum of $\Ftwol$ implies 
    that $\xstar$ is also a local minimum of $\Ftwolw$. Since $\Ftwolw$ is 
    convex, then  $\xstar$ is also a global minimum.
    
    Next, we claim that $\norm{A\xstar-\y}_2 \leq \norm{\y}_2$. Suppose 
    otherwise by contradiction.
    Note that $\inprod{\w}{\abs{\xstar}} \geq 0 = \inprod{\w}{\abs{\zerovec}}$. 
    Thus,
    \begin{equation*}
        \begin{split}
            \Ftwolw\br{\xstar}
            &= \tfrac{1}{2}\norm{A\xstar - \y}_2^2 + 
            \lambda\inprod{\w}{\abs{\xstar}}
            \\ &> \tfrac{1}{2}\norm{\y}_2^2 + \lambda\inprod{\w}{\abs{\xstar}}
            \geq \tfrac{1}{2}\norm{\y}_2^2 + \lambda\inprod{\w}{\abs{\zerovec}}
            = \Ftwolw\br{\zerovec},
        \end{split}
    \end{equation*}
    contradicting the fact that $\xstar$ is a global minimum of $\Ftwolw$.
    %
    
    For a function $f: \R[d] \rightarrow \R$, denote by $\partial f \ofx$ the 
    subdifferential of $f$ at $\x$. We now calculate $\partial \Ftwolw 
    \ofxstar$. The term $\frac{1}{2} \norm{A\x-\y}^2_2$ is differentiable 
    everywhere, and its subdifferential at $\x$ is the singleton containing its 
    gradient $\brc{A^T\br{A\x-\y}}$. Next, define for $\x\in\R[d]$,
        \begin{equation}
        V(\x) = \setst
            {\vvec\in \R[d]}
            {\norm{\vvec}_\infty \leq 1, \mbox{ and } \forall i \mbox{ with } 
            x_i\neq 0, v_i=\sign{x_i} }.
                \label{eq:def_V_x}
    \end{equation}
     It can be shown that the subdifferential of $\inprod{\w}{\abs{\x}}$ at 
     $\x$ is given by the following set
    \eq{\ensuremath{
    \partial \inprod{\w}{\abs{\x}} = \setst{\vvec \hadamard \w}{\vvec \in 
    V(\x)},}}
    where $\hadamard$ denotes the Hadamard (entrywise) product. Therefore, by 
    the Moreau-Rockafellar theorem on the additivity of subdifferentials 
    \cite[Prop. 4.2.4, pg. 232]{bertsekas2003convanal}\cite[Thm. 
    2.9]{balder2010subdiff},
    \eq{\ensuremath{\partial \Ftwolw \ofx = \setst{A^T\br{A\x-\y} + \lambda 
    \vvec \hadamard \w}{\vvec \in V(\x)}.
    }}
    Since $\xstar$ is a global minimum of $\Ftwolw$, $\partial \Ftwolw 
    \ofxstar$ contains the zero vector. Thus, there exists $\vvec \in 
    V(\xstar)$ such that 
$
A^T\br{A\xstar-\y} + \lambda \vvec \hadamard \w = \zerovec.$
    
    Suppose 
    that $\xstar$ is not $k$-sparse. Then there exists some index $i$ such that 
    $w_i=1$ and $x^*_i \neq 0$. At that index $i$, 
    $\inprod{\mb{a}_i}{A\xstar-\y} + \lambda \cdot \sign{x^*_i} = 0$.
    Since $\norm{A\xstar-\y}_2 \leq \norm{\y}_2$,
    \eq{\ensuremath{
        \lambda = \abs{\inprod{\mb{a}_i}{A\xstar-\y}} \leq \norm{\mb{a}_i}_2 
        \norm{A\xstar-\y}_2 
        \leq \norm{\mb{a}_i}_2 \norm{\y}_2
        \leq \max_{j=1,\ldots,d} \norm{\mb{a}_j}_2 \norm{\y}_2
            = \bar{\lambda}. 
    }}
    Hence, for any $\lambda>\bar{\lambda}$ any local minimum of $\Ftwol$ must 
    be $k$-sparse. 
\end{proof}

\begin{proof}[Proof of \cref{thm:large_lambda_p1l}]
    \label{proof:large_lambda_p1l}
    Let $\xstar$ be a local minimum of $\Fonel$, where $\lambda > 
    \lambdalarge$. Suppose by contradiction that $\xstar$ is not $k$\nbdash 
    sparse. Let $\vvec \eqdef \projk \br{\xstar}-\xstar \neq \zerovec$. Then for
    any $t \in \brs{0,1}$,
    \eq{\ensuremath{
        \Fonel \br{\xstar + t \vvec} - \Fonel \ofxstar = \norm{A \br{\xstar + t 
        \vvec} - \y}_2 -  \norm{A \xstar - \y}_2
        + \lambda \brs{ \tauk \br{\xstar + t \vvec} - \tauk \ofxstar}.
    }}
    By the triangle inequality, 
    \eq{\ensuremath{\norm{A \br{\xstar + t \vvec} - \y}_2 -  \norm{A \xstar - 
    \y}_2 \leq t \norm{A\vvec}_2.}}
    As for the second term, ${\xstar + t \vvec} = {\br{1-t}\xstar + t \projk 
    \br{\xstar}}$.
    This vector coincides with $\xstar$ on the $k$ largest-magnitude entries of 
    $\xstar$, and its remaining entries decrease linearly in magnitude with 
    $\br{1-t}$. Therefore,
    $\tauk \br{ \xstar + t \vvec } = (1-t)\tauk (\xstar).$
    Furthermore, since $\tauk(\xstar)=\norm{\vvec}_1$
    \eq{\ensuremath{
    \Fonel \br{\xstar + t \vvec} - \Fonel \ofxstar
            \leq t \norm{A \vvec}_2 + \lambda \brs{ \br{1-t} \tauk \br{ \xstar 
            } - \tauk \ofxstar}
        = t \left(\norm{A \vvec}_2 - \lambda \norm{\vvec}_1 \right).
    }}
    By 
    Eqs.~\cref{eq:beta_inequality},~\cref{eq:def_lambda_b} and the fact that 
    $\norm{\vvec}_1\br{\lambdalarge - \lambda} < 0$, for all $t \in 
    \left(0,1\right]$,
    \eq{\ensuremath{\Fonel \br{\xstar + t \vvec} - \Fonel\ofxstar  \leq t 
    \norm{\vvec}_1 \br{\lambdalarge - \lambda} < 0.}}
    This contradicts the assumption that $\xstar$ is a local minimum of 
    $\Fonel$.
\end{proof}

To prove \cref{thm:small_lambda_p1l}, we first state and prove the following 
auxiliary lemma.

\begin{lemma}\label{thm:small_lambda_p1l_aux} Suppose the matrix $A \in \R[n 
\times d]$ is of rank $n$, so that $\lambdasmall>0$. Let $\w \in \brs{0,1}^d$ 
such that $\sum_{i=1}^d w_i = d-k$. Let $\xstar$ be a global minimum of 
\eq{\ensuremath{\Fonelw \ofx \eqdef \norm{A\x-\y}_2 + \lambda 
\inprod{\w}{\abs{\x}},}}
where $0 < \lambda < \lambdasmall$. Then $A \xstar = \y$.
\end{lemma}

\begin{proof}\label{proof:small_lambda_p1l_aux}
Suppose by contradiction that $A\xstar \neq \y$. Let $\partial \Fonelw \ofx$ 
denote the subdifferential of $\Fonelw$ at $\x$. Since $A \xstar \neq \y$, the 
term $\norm{A \x - \y}_2$ is differentiable at $\x = \xstar$, and its 
subdifferential at $\xstar$ is a singleton consisting of its gradient, namely
\eq{\ensuremath{\evalat{\partial \norm{A \x - \y}_2}[\x = \xstar] = \brc{ A^T 
\frac{A\xstar-\y}{\norm{A\xstar - \y}_2} }.}}
The subdifferential of $\inprod{\w}{\abs{\x}}$ at $\x$ is given by
$\partial \inprod{\w}{\abs{\x}} = \setst{\vvec \hadamard \w}{\vvec \in V(\x) }$,
where $V(\x)$ is given by \cref{eq:def_V_x}. Therefore, by the 
Moreau-Rockafellar theorem, 
\eq{\ensuremath{\partial \Fonelw \ofxstar = \setst{A^T 
\frac{A\xstar-\y}{\norm{A\xstar - \y}_2} + \lambda \vvec \hadamard \w}{\vvec 
\in V(\xstar)}}.
}
    Since $\xstar$ is a global minimum of $\Fonelw$, $\partial \Fonelw 
    \ofxstar$ contains the zero vector. Thus, there exists some
    $\vvec\in V(\xstar)$ such that 
\eq{\ensuremath{
    A^T \frac{A\xstar-\y}{\norm{A\xstar - \y}_2} + \lambda \vvec \hadamard \w = 
    \zerovec.
    }}
This implies that
\eq{\ensuremath{
\lambda \norm{\vvec \hadamard \w}_2 = \norm{ A^T 
\frac{A\xstar-\y}{\norm{A\xstar - \y}_2}}_2 \geq \sigma_n \br{A}.
}}
Since $\norm{\vvec}_\infty \leq 1$, $\w \in \brs{0,1}^d$ and $\sum_{i=1}^d w_i 
= d-k$, we have:
\eq{\ensuremath{
    \norm{\vvec \hadamard \w}_2
    \leq \norm{\w}_2 
    = \sqrt{\summ[i=1][d] w_i^2}
    \leq \sqrt{\summ[i=1][d] w_i}
    = \sqrt{d-k}.
}}
Therefore $\sigma_n \br{A} 
\leq \lambda \sqrt{d-k},$ 
which contradicts the assumption that $\lambda < \lambdasmall=\frac{\sigma_n 
\br{A}}{\sqrt{d-k}}$.
\end{proof}

\begin{proof}[Proof of \cref{thm:small_lambda_p1l}]
    \label{proof:small_lambda1}
    Let $\xstar$ be a local minimum of $\Fonel$ with $0 < \lambda < 
    \lambdasmall$. Let $\w \in \R[d]$ such that $w_i=1$ at indices $i$ 
    corresponding to the $d-k$ smallest-magnitude entries of $\xstar$, breaking 
    ties arbitrarily, and $w_i=0$ elsewhere. By similar arguments to those in 
    the proof of \cref{thm:large_lambda_p2l}, $\xstar$ is a global minimum of 
    $\Fonelw$. Therefore, by \cref{thm:small_lambda_p1l_aux}, $A\xstar=\y$.  
\end{proof}

\begin{proof}[Proof of \cref{thm:bad_local_minma_everywhere}]
We prove the theorem for the power-2 objective $\Ftwol$. The proof for $\Fonel$ 
is similar.
Given a set $\Lambda$ of size $k$, we define
the following vector $\w \in \R[d]$
\begin{equation*}
    w_i = \twocase{0}{i \in \Lambda}{1}{i \notin \Lambda.}
\end{equation*}
Let $\Ftwolw \ofx \eqdef f\br{\x} + \lambda \inprod{\w}{\abs{\x}}$, where 
$f\br{\x} = \frac{1}{2}\norm{A\x-\y}_2^2$.
As in the proof of \cref{thm:large_lambda_p2l}, with
$V(\x)$ defined in \cref{eq:def_V_x}, the subdifferential of $\Ftwolw\br{\x}$ 
at $\x$ is given by
\begin{equation}
    \partial \Ftwolw \ofx = \setst{\nabla f \br{\x} + \lambda \vvec \hadamard 
    \w}{\vvec \in V(\x)}.
\end{equation}

Let $\xtilde$ be a global minimizer of $\norm{A\x-\y}_2$ over all vectors $\x 
\in \R[d]$ supported in $\Lambda$. Namely,
\begin{equation}
    \xtilde = \argmin_{\x}\ \norm{A\x-\y}_2 \quad \mbox{s.t.} \quad \forall i 
    \notin \Lambda \ x_i = 0.
\end{equation}
Assume that $\|\xtilde\|_0=k$. We shall now show that $\xtilde$ is a global 
minimizer of $\Ftwolw$ for any $\lambda\geq \lambdabar$. Define the vector 
$\tilde{\vvec} \in \R[d]$ by
\begin{equation*}
    \tilde{v}_i = \twocase{\sign\br{x_i}}{i \in \Lambda}{- \tfrac{1}{\lambda} 
    \nabla f \br{\xtilde}_i}{i \notin \Lambda.}
\end{equation*}
To show that $\tilde{\vvec}$ belongs to $V\br{\xtilde}$, it suffices to show 
that $\abs{\tfrac{1}{\lambda} \nabla f \br{\xtilde}_i} \leq 1$ for all $i 
\notin \Lambda$. 
Let $i \notin \Lambda$ and recall that ${\bf a}_i$ is the $i$-th column of $A$. 
Then
$$\abs{\nabla f \br{\xtilde}_i} = \abs{\inprod{{\bf a}_i}{A\xtilde-\y}} \leq 
\norm{{\bf a}_i}_2 \norm{A\xtilde-\y}_2
\overset{\text{(a)}}{\leq} \norm{{\bf a}_i}_2 \norm{\y}_2 
\overset{\text{(b)}}{\leq} \lambdabar \overset{\text{(c)}}{\leq} \lambda,$$
where (a) holds by our choice of $\xtilde$, (b) holds by the definition of 
$\lambdabar$, and (c) holds by the theorem assumption.
By definition,  
$\nabla f \br{\xtilde} + \lambda \tilde{\vvec} \hadamard \w \in \partial 
\Ftwolw\br{\xtilde}$.

We now show that $\nabla f \br{\xtilde} + \lambda \tilde{\vvec} \hadamard \w = 
\zerovec$. Since $\xtilde$ minimizes $f$ over all vectors supported in 
$\Lambda$, $\nabla f \br{\xtilde}_i = 0$ for all $i \in \Lambda$. Also recall 
that $w_i = 0$ for all $i \in \Lambda$. Therefore, for all $i \in \Lambda$,
\begin{equation*}
    \nabla f \br{\xtilde}_i + \lambda \tilde{v}_i w_i = 0.
\end{equation*}
For $i \notin \Lambda$, by the definitions
of $w_i$ and $\tilde v_i$, 
\begin{equation*}
    \nabla f \br{\xtilde}_i + \lambda \tilde{v}_i w_i = \nabla f \br{\xtilde}_i 
    - \lambda \tfrac{1}{\lambda} \nabla f \br{\xtilde}_i \cdot 1 = 0.
\end{equation*}
Therefore, $\zerovec \in \partial \Ftwolw\br{\xtilde}$, implying that $\xtilde$ 
is a global minimizer of $\Ftwolw$.

To complete our proof, we need to show that $\xtilde$ is a local minimum of 
$\Ftwol$. 
To this end, first note that since
$\|\x\|_0=k$, then 
 $\Ftwol\br{\xtilde} = \Ftwolw\br{\xtilde}$.
Let $x_{\min} = \min_{i\in\Lambda} |x_i|$.
By our assumptions $x_{\min}>0$. For any $\x$ with $\|\x-\tilde x\|_{\infty}< 
x_{\min}/2$, its top $k$ coordinates are the set $\Lambda$. Hence, for such 
vectors $\Ftwol\br{\x} = \Ftwolw\br{\x}$. 
Thus, there exists a small neighborhood of $\tilde{\x}$, where 
$\Ftwol(\tilde{\x}) \leq \Ftwol(\x)$
and so $\tilde{\x}$ is a local minimum
of $\Ftwol$.
\end{proof}

We now prove our sparse recovery guarantee for the power\nbdash 2 objective 
\cref{pr:P2l}. The proof for the power\nbdash 1 case follows thereafter.
\begin{proof}[Proof of \cref{thm:sparse_reconstruction_p2}]
    \label{proof:sparse_reconstruction_p2}    
    We start with the case that $\x_0$ is $k$\nbdash sparse, so that $\tauk 
    \br{\x_0}=0$. Then the assumption that $\Ftwol \ofxhat \leq \Ftwol 
    \br{\projk \br{\x_0}}$ translates to
    \eq[eq:proof_sparse_reconstruction_p2_eq_a]{\ensuremath{\frac{1}{2}\norm{A\xhat-\y}_2^2
     + \lambda \tauk \ofxhat \leq  \frac{1}{2} \norm{\e}_2^2.}}
    Since $\lambda \tauk \ofxhat \geq 0$ and $\y = A \x_0 + \e$, it follows 
    from \cref{eq:proof_sparse_reconstruction_p2_eq_a} that
    \eq{\ensuremath{\norm{\e}_2 
        \geq \norm{A\xhat-\y}_2 
        = \norm{A\br{\xhat-\x_0}-\e}_2.}}
    Denote $\xhat_r \eqdef \xhat - \projk\br{\xhat}$. Then $\tauk \ofxhat = 
    \norm{\xhat_r}_1$ and by the triangle inequlity,
    \eq{\ensuremath{\norm{\e}_2 
        \geq \norm{A\br{\projk\ofxhat-\x_0} + A\xhat_r - \e}_2
        \geq \norm{A\br{\projk\ofxhat-\x_0}}_2 - \norm{A\xhat_r}_2 - 
        \norm{\e}_2.}}
    Note that $\projk\ofxhat-\x_0$ is $2k$\nbdash sparse. Thus, using 
    \cref{eq:def_RIP_condition,eq:beta_inequality},
    \eq{\ensuremath{2 \norm{\e}_2
        \geq \ripmin\norm{\projk\ofxhat-\x_0}_1 - \lambdalarge\norm{\xhat_r}_1.
        }}        
    By \cref{eq:proof_sparse_reconstruction_p2_eq_a}, $\lambda \tauk\ofxhat = 
    \lambda \norm{\xhat_r}_1 \leq \frac{1}{2}\norm{\e}_2^2$, and thus
    \eq[eq:proof_sparse_reconstruction_p2_eq_b]{\ensuremath{\norm{\projk\br{\xhat}-\x_0}_1
        \leq \frac{2}{\ripmin}\norm{\e}_2 + \frac{1}{2\lambda} 
        \frac{\lambdalarge}{\ripmin}\norm{\e}_2^2.}}
    This proves part 1 of the theorem for the case that $\x_0$ is $k$\nbdash 
    sparse.

    Next, suppose that $\x_0$ is an arbitrary vector in $\R[d]$. Denote 
    $\xtilde_0 \eqdef \projk\br{\x_0}$ and $\tilde{\e} \eqdef \e + A \br{\x_0 - 
    \xtilde_0}$. Then $\y = A \x_0 + \e = A \xtilde_0 + \tilde{\e}$, where 
    $\xtilde_0$ is $k$\nbdash sparse.
    By the triangle inequality, 
    \eq[eq:proof_sparse_reconstruction_p2_eq_c]{\ensuremath{ 
    \norm{\projk\br{\xhat} - \x_0}_1 
    \leq \norm{\x_0 - \xtilde_0}_1  + \norm{\projk\br{\xhat} - \xtilde_0}_1  
    = \tauk \br{\x_0} + \norm{\projk\br{\xhat} - \xtilde_0}_1.
    }}
    Similarly, using \cref{eq:beta_inequality},
    \eq[eq:proof_sparse_reconstruction_p2_eq_d]{\ensuremath{
    \norm{\tilde{\e}}_2 \leq \norm{\e}_2 + \norm{A\br{\x_0-\xtilde_0}}_2 \leq 
    \norm{\e}_2 + \lambdalarge\norm{\x_0-\xtilde_0}_1 = \norm{\e}_2 + 
    \lambdalarge \tauk\br{\x_0}.
    }}
    The assumption that $\Ftwol \ofxhat \leq \Ftwol \br{\projk \br{\x_0}}$, 
    combined with the definition of $\tilde{\e}$, implies that $\Ftwol \ofxhat 
    \leq \Ftwol \br{\xtilde_0} = \frac{1}{2}\norm{\tilde{\e}}_2^2$. Since 
    $\xtilde_0$ is $k$-sparse, and $\y = A \xtilde_0 + \tilde{\e}$, then by 
    \cref{eq:proof_sparse_reconstruction_p2_eq_b},
    \eq{\ensuremath{ \norm{\projk\br{\xhat}-\xtilde_0}_1 
    \leq \frac{2}{\ripmin}\norm{\tilde{\e}}_2 + 
    \frac{1}{2\lambda}\frac{\lambdalarge}{\ripmin}\norm{\tilde{\e}}_2^2.}}
    Combining this inequality with 
    \cref{eq:proof_sparse_reconstruction_p2_eq_c} gives
    \eq{\ensuremath{ \norm{\projk\br{\xhat}-\x_0}_1
    \leq \tauk \br{\x_0} + \frac{2}{\ripmin}\norm{\tilde{\e}}_2 + 
    \frac{1}{2\lambda}\frac{\lambdalarge}{\ripmin}\norm{\tilde{\e}}_2^2.}}
    Finally, inserting \cref{eq:proof_sparse_reconstruction_p2_eq_d} to the 
    above yields Eq.~\cref{eq:recovery_error_bound_1}, and thus part 1 of the 
    theorem is proven.
    
    We now prove part 2. Suppose that $\xhat$ is $k$\nbdash sparse. By the 
    triangle inequality,
    \eq{\ensuremath{\norm{\xhat-\x_0}_1
    \leq \norm{\x_0-\projk\br{\x_0}}_1  + \norm{\xhat-\projk\br{\x_0}}_1 = 
    \tauk \br{\x_0}
     + \norm{\xhat-\projk\br{\x_0}}_1.}} 
    By \cref{eq:def_RIP_condition}, the second term is bounded by 
    $\norm{\xhat-\projk\br{\x_0}}_1 \leq 
    \frac{1}{\ripmin}\norm{A\br{\xhat-\projk\br{\x_0}}}_2$.
    Thus, 
    \eq{\ensuremath{\norm{\xhat-\x_0}_1
        \leq &\tauk \br{\x_0} + 
        \frac{1}{\ripmin}\norm{A\br{\xhat-\projk\br{\x_0}}}_2
         = \tauk \br{\x_0} + \frac{1}{\ripmin}\norm{A\xhat - \y - 
         \br{A\projk\br{\x_0} - \y}}_2
         \\ \leq &\tauk \br{\x_0} + \frac{1}{\ripmin}\norm{A\xhat - \y}_2 + 
         \frac{1}{\ripmin}\norm{A\projk\br{\x_0} - \y}_2.}}
     Since $\tauk\br{\xhat} = \tauk\br{\projk\br{\x_0}} = 0$, the assumption 
     that $\Ftwol \ofxhat \leq \Ftwol\br{\projk\br{\x_0}}$ implies that 
     $\norm{A\xhat - \y}_2 \leq \norm{A\projk\br{\x_0} - \y}_2$. Therefore,
     \eq[eq:proof_sparse_reconstruction_p2_eq_e]{\ensuremath{\norm{\xhat-\x_0}_1
      \leq \tauk \br{\x_0} + \frac{2}{\ripmin}\norm{A\projk\br{\x_0} - \y}_2.}}
     Using the triangle inequality and \cref{eq:beta_inequality},
     \eq{\ensuremath{\norm{A\projk\br{\x_0} - \y}_2
     \leq \norm{A\x_0-\y}_2 + \norm{A\br{\x_0 - \projk\br{\x_0}}}_2
     \leq \norm{A\x_0-\y}_2 + \lambdalarge \norm{\x_0 - \projk\br{\x_0}}_1.}}
     Since $\norm{\projk\br{\x_0}-\x_0}_1 = \tauk\br{\x_0}$ and 
     $\norm{A\x_0-\y}_2 = \norm{\e}_2$, then 
     $\norm{A\projk\br{\x_0} - \y}_2 
        \leq \norm{\e}_2 + \lambdalarge \tauk\br{\x_0}.$     
     Inserting this into \cref{eq:proof_sparse_reconstruction_p2_eq_e} yields 
     \cref{eq:recovery_error_bound_2}, which proves part 2 of the theorem.
\end{proof}

\begin{proof}[Proof of \cref{thm:sparse_reconstruction_p1}]
    \label{proof:sparse_reconstruction_p1}
We first prove parts 1 and 2 of the theorem when $\x_0$ is $k$\nbdash sparse. 
In this case, the vector $\projk\ofxhat-\x_0$ is at most $2k$\nbdash sparse. By 
\cref{eq:def_RIP_condition} and the triangle inequality,
\eq{\ensuremath{
    \ripmin\norm{\projk \br{\xhat} - \x_0}_1
    &\leq  \norm{A \br{\projk \br{\xhat} - \x_0}}_2 \leq  \norm{A \cdot \projk 
    \br{\xhat} - \y}_2 + \norm{A \x_0 - \y}_2
    \\ &\leq \norm{A\xhat-\y}_2 +\norm{A\br{\projk(\xhat)-\xhat}}_2 + 
    \norm{\e}_2.
}}
Using~\cref{eq:beta_inequality},
$\norm{A\br{\projk \br{\xhat}-\xhat}}_2 \leq \lambdalarge \norm{\projk 
\br{\xhat}-\xhat}_1 = \lambdalarge \tauk \br{\xhat}$. Thus,
\eq[eq:thm_sparse_reconstruction_main_inequality]
{\ensuremath{
    \ripmin\norm{\projk \br{\xhat} - \x_0}_1
    \leq \norm{A \xhat - \y}_2 + \lambdalarge \tauk \br{\xhat} + \norm{\e}_2.} }
Since $\x_0$ is $k$\nbdash sparse, 
$\projk\br{\x_0}=\x_0$. 
The assumption that $\Fonel \br{\xhat} \leq \Fonel \br{\projk\br{\x_0}}$ thus 
reads as
 \eq[eq:thm_sparse_reconstruction_optimality_of_xhat]{\ensuremath{
         \norm{A \xhat - \y}_2 + \lambda \tauk \br{\xhat} \leq &\norm{A \projk 
         \br{\x_0} - \y}_2 + \lambda \tauk \br{\projk \br{\x_0}}
          = \norm{A \x_0 - \y}_2 = \norm{\e}_2.} }  
Therefore,  
\eq[eq:thm_sparse_reconstruction_upper_bound_on_P_beta]{\ensuremath{
        \norm{A \xhat - \y}_2 + \lambdalarge \tauk \br{\xhat} 
        \leq  \max\brc{1,\tfrac{\lambdalarge}{\lambda}} \br{ \norm{A \xhat - 
        \y}_2 + \lambda \tauk \br{\xhat} }
        \leq \max\brc{1,\tfrac{\lambdalarge}{\lambda}} \norm{\e}_2,
}}
Inserting \cref{eq:thm_sparse_reconstruction_upper_bound_on_P_beta} into 
\cref{eq:thm_sparse_reconstruction_main_inequality} 
proves part 1 of the \lcnamecref{thm:sparse_reconstruction_p1} when $\x_0$ is 
$k$\nbdash sparse, since  
\eq[eq:thm_sparse_reconstruction_main_inequality_bound_when_x0_sparse]{\ensuremath{
        \norm{\projk \br{\xhat} - \x_0}_1 \leq \tfrac{1}{\ripmin} \br{1 + 
        \max\brc{1,\tfrac{\lambdalarge}{\lambda}}} \norm{\e}_2.
}}

Next, suppose that $\xhat$ is also $k$\nbdash sparse. Here 
$\projk\ofxhat=\xhat$ and $\tauk \ofxhat = 0$. Combining 
\cref{eq:thm_sparse_reconstruction_main_inequality} and 
\cref{eq:thm_sparse_reconstruction_optimality_of_xhat} proves part 2 of the 
theorem when $\x_0$ is $k$\nbdash sparse, since
\eq[eq:x_x0_e]{\ensuremath{
    \norm{\xhat - \x_0}_1 \leq 
    \tfrac{1}{\ripmin} \br{ \norm{A \xhat - \y}_2 + \norm{\e}_2} 
    \leq  \tfrac{2}{\ripmin} \norm{\e}_2.} }

We now generalize the proof to an arbitrary $\x_0$. Denote $\xtilde_0 = 
\projk\br{\x_0}$ and $\tilde{\e} = \e + A \br{\x_0 - \xtilde_0}$. Then $\y = A 
\x_0 + \e = A \xtilde_0 + \tilde{\e}$.
By the triangle inequality,
\eq[eq:thm_sparse_reconstruction_arbitrary_x0_a]{\ensuremath{
        \norm{\projk \br{\xhat} - \x_0}_1 \leq &\norm{\projk \br{\xhat} - 
        \xtilde_0}_1 + \norm{\x_0 - \xtilde_0}_1
         = \norm{\projk \br{\xhat} - \xtilde_0}_1 + \tauk \br{\x_0}.
}}
Similarly, by the triangle inequality and \cref{eq:beta_inequality},
\eq[eq:thm_sparse_reconstruction_arbitrary_x0_b]{\ensuremath{
        \norm{\tilde{\e}}_2 \leq &\norm{\e}_2 + \norm{A \br{\x_0 - \xtilde_0}}_2
         \leq \norm{\e}_2 + \lambdalarge \norm{\x_0 - \xtilde_0}_1
         = \norm{\e}_2 + \lambdalarge \tauk \br{\x_0}.
}}
The assumption that $\Fonel \br{\xhat} \leq \Fonel \br{\projk\br{\x_0}}$ now 
reads
\eq{\ensuremath{
    \norm{A \xhat - \y}_2 + \lambda \tauk \br{\xhat} 
    \leq \norm{A \cdot \projk \br{\x_0} - \y}_2 
    = \norm{A \cdot \projk \br{\tilde{\x}_0} - \y}_2 
    = \Fonel\br{\projk \br{\tilde{\x}_0}}.
}}
Since $\y = A\tilde{\x}_0 + \tilde{\e}$, $\Fonel \br{\xhat} \leq \Fonel 
\br{\projk\br{\tilde{\x}_0}}$ and $\tilde{\x}_0$ is $k$\nbdash sparse, using 
\cref{eq:thm_sparse_reconstruction_main_inequality_bound_when_x0_sparse} with 
$\x_0$, $\e$ replaced respectively by $\tilde{\e}$, $\tilde{\x}_0$ gives 
\eq{\ensuremath{
\norm{\projk \br{\xhat} - \xtilde_0}_1 \leq \tfrac{1}{\ripmin} \br{ 1 + 
\max\brc{1,\tfrac{\lambdalarge}{\lambda}} } \norm{\tilde{\e}}_2.
}}
Inserting this into \cref{eq:thm_sparse_reconstruction_arbitrary_x0_a}, and 
applying \cref{eq:thm_sparse_reconstruction_arbitrary_x0_b}, gives
\eq{\ensuremath{
        \norm{\projk \br{\xhat} - \x_0}_1 
        &\leq 
            \tfrac{1}{\ripmin} \br{ 1 + 
            \max\brc{1,\tfrac{\lambdalarge}{\lambda}} } \norm{\tilde{\e}}_2 + 
            \tauk \br{\x_0}
    \\ &\leq 
    \tfrac{1}{\ripmin} \br{ 1 + \max\brc{1,\tfrac{\lambdalarge}{\lambda}} } 
    \br{\norm{\e}_2 + \lambdalarge \tauk \br{\x_0}} + \tauk \br{\x_0},
}}
which proves part 1 of the \lcnamecref{thm:sparse_reconstruction_p1} for a 
general $\x_0$. We now prove 
part~\ref{item:thm_sparse_reconstruction1_p1_sparse_case}. Suppose $\xhat$ 
itself is $k$\nbdash sparse, then by \cref{eq:x_x0_e},
\eq{\ensuremath{
\norm{\xhat - \xtilde_0}_1 \leq &\tfrac{2}{\ripmin} \norm{\tilde{\e}}_2.
}}
Inserting 
\cref{eq:thm_sparse_reconstruction_arbitrary_x0_a,eq:thm_sparse_reconstruction_arbitrary_x0_b}
 to the above inequality yields \cref{eq:recovery_error_bound_2}.
\end{proof}

\subsection{Properties of the GSM auxiliary functions}
\label{appendix:proofs_properties_of_gsm_aux}
%
The following \lcnamecref{thm:mu_limits_gamma} describes some continuity 
properties of $\mukg\ofz$. 
\begin{lemma}\label{thm:mu_limits_gamma}
    For any $\z \in \R[d]$, the function $\mukg \ofz$ is monotone-increasing 
    w.r.t. $\gamma$, and
    \eq[eq:thm_mu_limits_gamma_inf]{\ensuremath{
        \limit[\gamma \app -\infty]\ \mukg \ofz = \min_{\abs{\Lambda}=k} 
        \summ[i \in \Lambda] z_i,
        \qquad
        \limit[\gamma \app 0]\ \mukg \ofz = \frac{k}{d} \summ[i=1][d]z_i,
        \qquad
        \limit[\gamma \app \infty]\ \mukg \ofz = \max_{\abs{\Lambda}=k} \summ[i 
        \in \Lambda] z_i.
        }}
    Moreover, the limits as $\gamma \app \infty$ and $\gamma \app -\infty$ are 
    uniform over $\z \in \R[d]$. Specifically, 
\begin{equation}
	\label{eq:mu_approximation_bound_gamma_neg}
    \min_{\abs{\Lambda}=k} \summ[i \in \Lambda] z_i  \leq \mukg \ofz \leq 
    \min_{\abs{\Lambda}=k} 
            \summ[i \in \Lambda] z_i - \frac{1}{\gamma} \log {\binom{d}{k}} 
            \qquad \forall\z \in \mathbb{R}^d,\ \gamma < 0.
    \end{equation}
    with a similar inequality holding for $\gamma > 0$. 
\end{lemma}
\begin{proof}
\label{proof:mu_limits_gamma}
Consider the limit $\gamma\to 0$, and write $\mu_{k,\gamma}(\z) = \log( 
h_{k,\gamma}(\z)) / \gamma$
where 
\[
h_{k,\gamma}(\z) = \frac{\sum_{|\Lambda|=k} \exp\left(\gamma \sum_{i\in\Lambda} 
z_i\right)}{D}
\]
and $D = \binom{d}{k}$. Then, by L'Hospital's rule, 
\begin{eqnarray}
    \lim_{\gamma \to 0} \mu_{k,\gamma}(\z) = \lim_{\gamma\to 0} 
    \frac{\frac{\partial}{\partial \gamma} h_{k,\gamma}(\z)}{h_{k,\gamma}(\z)} 
    = \frac{1}{D}\sum_{|\Lambda|=k} \sum_{i\in\Lambda} z_i. 
            \nonumber
\end{eqnarray}
Since each $z_i$ appears in $k/d$ of all subsets $\Lambda\subset[d]$ of size 
$k$, the sum above equals $k/d  \sum_i z_i$.  

Next, we address the limits as $\gamma \rightarrow \pm\infty$. 
Let $s_{\max} = \max_{|\Lambda|=k} \sum_{i\in\Lambda}z_i$ and $\gamma > 0$. 
Then, 
\[
\mu_{k,\gamma}(\z) \leq \frac1{\gamma}
    \log\left ( \exp (\gamma s_{\max}) \right) = s_{\max}.
\]
On the other hand, 
\[
\mu_{k,\gamma}(\z) \geq \frac1{\gamma}
    \log\left(\frac{\exp(\gamma s_{\max}) }{D}\right) \geq 
    s_{\max}-\frac{1}{\gamma}
    \log\br{D}.
\]
Thus, by the two above inequalities, for all $\gamma > 0$,
\eq[eq:bound_mu_k_gamma]{\ensuremath{\max_{\abs{\Lambda}=k} 
\sum_{i\in\Lambda}z_i-\frac{1}{\gamma}
    \log\br{D} \leq \mukg \ofz \leq \max_{\abs{\Lambda}=k} 
    \sum_{i\in\Lambda}z_i.}}
As $\gamma\to\infty$ the second term on the left-hand side vanishes. Thus, as 
$\gamma\to\infty$, uniformly in $\z$, 
$\mu_{k,\gamma}(\z)\to\max_{|\Lambda|=k}\sum_{i\in\Lambda}z_i$. 
Inserting \cref{eq:gsm_identity_gammaneg} 
into \cref{eq:bound_mu_k_gamma} and replacing $\z$ by $-\z$ gives
\eq{\ensuremath{\max_{\abs{\Lambda}=k} 
\sum_{i\in\Lambda}\br{-z_i}-\frac{1}{\gamma}
    \log\br{D} \leq -\muargs{k}{-\gamma} \br{\z} \leq \max_{\abs{\Lambda}=k} 
    \sum_{i\in\Lambda}\br{-z_i} = -\min_{\abs{\Lambda}=k} 
    \sum_{i\in\Lambda}{z_i},
    }}
whence \cref{eq:mu_approximation_bound_gamma_neg} follows. In addition, 
uniformly in $\z$, 
$\mu_{k,\gamma}(\z)\to\min_{|\Lambda|=k}\sum_{i\in\Lambda}z_i$ as $\gamma\to 
-\infty$.

We now prove that $\mukg \ofz$ increases monotonically w.r.t. $\gamma$. 
Let ${\bf s}\in\mathbb{R}^D$ be the vector
whose $D$ coordinates contain all the partial sums $\sum_{i\in\Lambda}z_i$ of 
all subsets $\Lambda\subset[d]$ of size $k$. Then
\[
\mu_{k,\gamma}(\z) = \frac1\gamma \log\left(  \frac1D \sum_{j=1}^D \exp(\gamma 
s_j)    \right).
\]
Differentiating 
with respect to $\gamma$ at any finite $\gamma \neq 0$ yields
\begin{equation*}
\frac{\partial}{\partial \gamma} \muargs{k}{\gamma} \ofz = \frac{1}{\gamma^2} 
    \frac{\sum_{j=1}^D (\gamma s_j) \exp \br{ \gamma s_j }}{\sum_{j=1}^D 
            \exp \br{ \gamma s_j }} -\frac{1}{\gamma^2} \log \br{ \frac{1}{D} 
            \sum_{j=1}^D \exp \br{ \gamma s_j } }.
\end{equation*}
Let $y_j = \exp \br{\gamma s_j}$ and $p_j = \frac{y_j}{\sum_{i=1}^D y_i}$. 
Then, 
\begin{equation*}
\begin{split}
\gamma^2 \tfrac{\partial}{\partial \gamma} \muargs{k}{\gamma} \ofz 
    &= \sum_{j=1}^D  p_j \log \br{y_j} - \log \br{ \tfrac{1}{D} {\sum_{j=1}^D}  
    y_j } 
\\ &= \sum_{j=1}^D p_j \log \br{p_j \sum_{i=1}^D  y_i} - \log \br{ \tfrac{1}{D} 
\sum_{j=1}^D y_j }
\end{split}
\end{equation*}
Since $\sum_j p_j=1$, expanding the right-hand side, the term $\log(\sum_i 
y_i)$ cancels, and we obtain
\begin{equation*}
\gamma^2 \frac{\partial}{\partial \gamma} \muargs{k}{\gamma} \ofz 
= \sum_{j=1}^D p_j \log \br{p_j} - \log { \frac{1}D }.
\end{equation*}
Since $-\sum_{j=1}^D p_j \log \br{p_j}$ is the entropy of the probability 
distribution $\br{p_1,\ldots,p_D}$, it is smaller or equal to $\log D$. 
Hence, $\tfrac{\partial}{\partial \gamma}\muargs{k}{\gamma} \ofz \geq 0$ for 
all finite $\gamma \neq 0$. Since $\mukg \ofz$ is continuous w.r.t. $\gamma$ at 
$\gamma=0$, it follows that $\mukg \ofz$ is monotone increasing w.r.t. 
$\gamma$. 
\end{proof}

Before proving \cref{thm:theta_limits_gamma},
we first state an important property of $\gsm \ofz$.
\begin{lemma}\label{thm:sum_theta}
    Let $\z \in \R[d]$. Then for all $ 0 \leq k \leq d,\ \gamma \in 
    \br{-\infty,\infty}$,
    \begin{equation}
            \label{eq:sum_gsm}
            \gsmi{i}\ofz \in \brs{0,1},\ i=1,\ldots,d
            \qquad \mbox{and} \qquad
            \summ[i=1][d] \gsmi{i}\ofz = k.
    \end{equation}
\end{lemma}
\begin{proof}
\label{proof:sum_theta}
For $k=0,d$, the proof is trivial. Suppose that $0<k<d$. The fact that 
$\gsmi{i}  \in[0,1]$ follows directly from the definition of $\gsm\ofz$ in 
\cref{eq:def_gsm}. 
Next, summing over all $d$ coordinates, each $k$-tuple $\Lambda \subset 
\brs{d}$ appears once in the denominator,
but $k$ times in the numerator --- once for each $i$ that belongs to $\Lambda$. 
Therefore, the sum equals $k$, proving \cref{eq:sum_gsm}.
\end{proof}

\begin{proof}[Proof of \cref{thm:theta_limits_gamma}]
    \label{proof:theta_limits_gamma}
    We present the proof for the case $\gamma \app \infty$. The proof for 
    $\gamma\to-\infty$ is similar. If $k=d$, then $\brs{d}$ is a disjoint union 
    of $\Lambda_a$ and $\Lambda_b$, and the proof follows directly from the 
    definition of $\w_{d,\gamma}\ofz$. Suppose that $0<k<d$. For $\z \in 
    \R[d]$, 
    define $s_{\max} = \max_{|\Lambda|=k} \sum_{i\in\Lambda} z_i$. We say that 
    a $k$-tuple $\Lambda \subset \brs{d}$ is \emph{maximal} if  $\sum_{i \in 
    \Lambda} z_i=s_{\max}$. Then we may rewrite \cref{eq:def_gsm} as 
    \begin{equation}
    \gsmi{i} \ofz = 
    \frac{ \sum_{\abs{\Lambda} = k,\ i \in \Lambda} \exp \br{ \gamma \br{ 
    \sum_{j\in\Lambda} {z_j} - s_{\max} } } }
    { \sum_{ \abs{\Lambda} = k } \exp \br{ \gamma \br{ \sum_{j\in\Lambda} {z_j} 
    -  s_{\max} } } }.
        \nonumber              
    \end{equation}
Now, if $\Lambda$ is maximal, then for any $\gamma \in\mathbb{R}$, $\exp \br{ 
\gamma \br{ \sum_{j\in\Lambda} {z_j} - s_{\max} } } = 1$.
    In contrast, if $\Lambda$ is not maximal, then $\sum_{j\in\Lambda} {z_j} < 
    s_{\max}$, and thus  
    $\exp \br{ \gamma \br{ \sum_{j\in\Lambda} {z_j} - s_{\max} } } \app[\gamma 
    \app \infty] 0$.
    Therefore, in the limit $\gamma \rightarrow \infty$, 
    $\gsmiargs{k}{\gamma}{i} \ofz$ equals the number of maximal $k$-tuples 
    $\Lambda$ that contain $i$, divided by the total number of maximal 
    $k$-tuples.
    
    If $i \in \Lambda_a$, then every maximal $k$-tuple $\Lambda$ must contain 
    the index $i$, otherwise it would not be maximal.  Thus, 
    $\underset {\gamma \rightarrow \infty } {\lim} \ \gsmi{i}\ofz = 1.$
    If $i \notin \Lambda_a \cup \Lambda_b$, then every $k$-tuple that contains 
    $i$ is surely not maximal, and then 
    $\underset {\gamma \rightarrow \infty } {\lim} \ \gsmi{i}\ofz = 0.$
    Finally, define $\Lambda_c \eqdef \br{\Lambda_a \cup \Lambda_b}^c$.
    By \cref{thm:sum_theta},
    \eq{\ensuremath{
        k = \summ[i=1][d] \gsmiargs{k}{\gamma}{i} \ofz = \sum_{i \in \Lambda_a} 
        \gsmiargs{k}{\gamma}{i} \ofz + 
        \sum_{i \in \Lambda_b} \gsmiargs{k}{\gamma}{i} \ofz +
        \sum_{i \in \Lambda_c} \gsmiargs{k}{\gamma}{i} \ofz.}}
    Now take the limit as $\gamma\to\infty$. In the right-hand side above, the 
    sum over $i \in \Lambda_a$ tends to $\abs{\Lambda_a}$, whereas the sum over 
    $i\in\Lambda_c$ tends to zero. As for the middle sum, by definition for all 
    $i\in\Lambda_b$, the values $z_i$ are equal. Hence, the corresponding 
    values $\gsmiargs{k}{\gamma}{i} \ofz$ are also all equal. 
Therefore, 
    \begin{equation*}
    \lim_{\gamma \app \infty} \gsmiargs{k}{\gamma}{i} \ofz = \frac{k - 
    \abs{\Lambda_a}}{\abs{\Lambda_b} } \qquad \forall i \in \Lambda_b.
    \end{equation*}    
\end{proof}

\begin{proof}[Proof of \cref{thm:gsm_identities}]
For $k=0,d$, the proof is trivial and thus omitted. Suppose that $0<k<d$. For 
$\gamma=0$, \cref{eq:gsm_identity_dcomplement} holds by definition. Let $\gamma 
\neq 0,\pm \infty$. Then
\eq{\ensuremath{\mukg \ofz - \summ[i=1][d]z_i
&= \frac{1}{\gamma} \log \br{ \frac{1}{\binom{d}{k}} \sum_{ \abs{\Lambda} = k } 
\exp \br{ \gamma \sum_{i\in\Lambda} {z_i} } } 
- \frac{1}{\gamma}\log\br{\exp\br{\gamma \summ[i=1][d]z_i}}
\\ &= \frac{1}{\gamma} \log \br{ \frac{1}{\binom{d}{k}} \sum_{ \abs{\Lambda} = 
k } \exp \br{ \gamma \sum_{i\in\Lambda} {z_i} - \gamma \summ[i=1][d]z_i} } 
\\ &= \frac{1}{\gamma} \log \br{ \frac{1}{\binom{d}{d-k}} \sum_{ \abs{\Lambda} 
= d-k } \exp \br{ -\gamma \sum_{i\in\Lambda} {z_i} } }
= -\muargs{d-k}{-\gamma}\br{\z}. 
}}
Thus, for any finite $\gamma$, $\mukg \ofz + \muargs{d-k}{-\gamma}\br{\z} = 
\summ[i=1][d]z_i$. Differentiating with respect to $z_i$ gives that 
$\gsmiargs{k}{\gamma}{i}\ofz + \gsmiargs{d-k}{-\gamma}{i}\ofz = 1$. Thus,  
\cref{eq:gsm_identity_dcomplement} holds for any finite $\gamma$. Taking the 
limits as $\gamma \app \pm \infty$ yields that 
\cref{eq:gsm_identity_dcomplement} holds also for $\gamma = \pm \infty$.
\end{proof}

The following \lcnamecref{thm:mu_limits_gamma} describes the convexity of 
$\mukg \ofz$. 
\begin{lemma}\label{thm:mu_convex}
    Let $0 \leq k \leq d$, $\gamma \in \brs{-\infty,\infty}$. Then $\mukg \ofz$ 
    is convex in $\z$ for $\gamma \geq 0$ and concave in $\z$ for $\gamma \leq 
    0$. In particular, for any $\z,\z_0\in\mathbb{R}^d$, 
    \eq[mu_convexity_inequality]{\ensuremath{
        \mukg\ofz &\geq \mukg\br{\z_0}
        + \inprod{\gsm \br{\z_0}}{\z - \z_0} \qquad \gamma \geq 0
    \\         \mukg\ofz &\leq \mukg\br{\z_0}
    + \inprod{\gsm \br{\z_0}}{\z - \z_0} \qquad \gamma \leq 0. 
}}
\end{lemma}
\begin{proof}
    \label{proof:mu_convex}
At $\gamma = 0$, $\muargs{k}{0} \ofz$ is linear in $\z$, and thus both convex 
and concave, and by the definitions in \cref{eq:def_mukg,eq:def_gsm}, the 
\lcnamecref{thm:mu_convex} holds.
Next, suppose that $\gamma > 0$. Given $\z \in \R[d]$, 
define $D=\binom{d}{k}$ and let ${\bf s}\in\mathbb{R}^D$ be the vector
whose $D$ coordinates contain the sums $\sum_{i\in\Lambda}z_i$ of all subsets 
$\Lambda\subset[d]$ of size $k$.
From the definition of $\mukg \ofz$ in \cref{eq:def_mukg}, 
\eq
{
    \mukg \ofz = \muargs{1}{\gamma} \br{{\bf s}} = \frac1\gamma 
    \log\left(\frac1D \sum_{j=1}^D  \exp(\gamma s_j) \right).
}
The right-hand side is a \emph{log-sum-exp} function, which is known to be 
convex (see \cite[\S3.1.5, pg. 72]{boyd2004convex}). At $\gamma = \infty$, 
$\muargs{1}{\infty} \br{\argdot}$ is the \emph{maximum} function, which is also 
convex. Hence, for $\gamma > 0$, $\muargs{1}{\infty} \br{s}$ is convex in $s$. 
Since the transformation $\z \mapsto {\bf s}$ is linear, $\mukg \ofz$ is convex 
in $\z$.

For $0 < \gamma < \infty$, $\mukg \ofz$ is differentiable and $\nabla_\z \mukg 
\ofz = \gsm \ofz$. Therefore, \cref{mu_convexity_inequality} follows from the 
convexity of $\mukg\ofz$. 
Keeping $\z$, $\z_0$ fixed and taking the limit as $\gamma \app \infty$ yields 
that \cref{mu_convexity_inequality} also holds at $\gamma = \infty$.
%
%
%
Finally, by \cref{eq:gsm_identity_gammaneg}, 
$\muargs{k}{\gamma} \ofz = -\muargs{k}{-\gamma} \br{-\z}$, implying that  
$\mukg \ofz$ is concave for $\gamma<0$. Hence, the second equation in  
\cref{mu_convexity_inequality} holds for $\gamma<0$ as well.
\end{proof}

\begin{proof}[Proof of \cref{thm:recursion_aqg}]
For brevity, we only state the main steps of the proof. We define $M_q^{\br{r}} 
\ofz$ as the sum of the $q$ largest entries of $\br{z_r,\ldots,z_d}$. Namely, 
for $1 \leq r \leq d+1$,
\eq[eq:def_Mqr]{\ensuremath{
    M_q^{\br{r}} \ofz \eqdef \threecase
    {0}{q=0}
    {\summ[i=1][q] z_{\br{i}}^{\br{r}}}
    {1 \leq q \leq d-r+1}{-\infty}{q > d-r+1.}}}
For $r=1$, denote $M_q \ofz \eqdef M_q^{\br{1}} \ofz$. It can be shown that 
$s_{q,\gamma}^{\br{r}} \ofz$ of \cref{eq:def_aqgr} satisfy the recursion for $1 
\leq r \leq d-1$ and $1 \leq q \leq d-r$, 
\eq[eq:recursion_aqgr_prelim]{\ensuremath{
    s_{q,\gamma}^{\br{r}} \ofz 
    = &s_{q,\gamma}^{\br{r+1}} \ofz \exp \br{ \gamma \br{M_q^{\br{r+1}} \ofz - 
    M_q^{\br{r}} \ofz} }
    + \\&s_{q-1,\gamma}^{\br{r+1}} \ofz \exp \br{ \gamma \br{z_r + 
    M_{q-1}^{\br{r+1}} \ofz - M_q^{\br{r}} \ofz} }}}
The terms inside the exponents in \cref{eq:recursion_aqgr_prelim} can be shown 
to satisfy by the following formulas 
\eq[eq:replace_M_with_subplus]{\ensuremath{
M_q^{\br{r+1}} \ofz - M_q^{\br{r}} \ofz &= -\subplus{z_r - 
z_{\br{q}}^{\br{r+1}} }
\\ z_r + M_{q-1}^{\br{r+1}} \ofz - M_q^{\br{r}} \ofz &= -\subminus{z_r - 
z_{\br{q}}^{\br{r+1}}}}}
for $1 \leq r \leq d-1$ and $1 \leq q \leq d-r$. From 
\cref{eq:replace_M_with_subplus}, the lemma follows.
\end{proof}

\subsection{Properties of the GSM penalty}
\label{appendix:proofs_properties_of_gsm_penalty}
%
\begin{proof}[Proof of \cref{thm:tau_limits_gamma}]
    \label{proof:tau_limits_gamma}
    Recall that by \cref{eq:penalty_by_aux}, $\taukg \ofx = 
    \muargs{d-k}{-\gamma} \br{\abs{\x}}$. 
    By \cref{thm:mu_limits_gamma}, $\muargs{d-k}{-\gamma} \ofz$ is monotone 
    decreasing with respect to $\gamma$, and thus so is $\taukg \ofx$. Moreover,
    \eq{\ensuremath{\limit[\gamma \app 0] \taukg \ofx 
        = \limit[\gamma \app 0] \muargs{d-k}{-\gamma} \br{\abs{\x}}
        = \frac{d-k}{d} \summ[i=1][d] \abs{x_i}.}}      
    Next, replacing $\z$ by $\abs{\x}$, $k$ by $d-k$ and $\gamma$ by $-\gamma$ 
    in  \cref{eq:mu_approximation_bound_gamma_neg} gives
    \eq{\ensuremath{\min_{\abs{\Lambda}=d-k} \summ[i \in \Lambda] \abs{x_i} 
    \leq \muargs{d-k}{-\gamma} \br{\abs{\x}} \leq \min_{\abs{\Lambda}=d-k} 
    \summ[i \in \Lambda] \abs{x_i} + \frac{1}{\gamma} \log \binom{d}{k}.}}
    Combining \cref{eq:penalty_by_aux,eq:def_tauk_b} with the above inequality 
    yields Eq.~\cref{eq:tau_approximation_bound}.
    \end{proof}

\begin{proof}[Proof of \cref{thm:w_limits_gamma}]
\label{proof:w_limits_gamma}
For $k=0,d$ the claim follows from \cref{eq:def_wkg}. For $0 < k < d$, let $\x 
\in \R[d]$ and suppose w.l.o.g. that $\abs{x_1} \leq \abs{x_2} \leq \cdots \leq 
\abs{x_d}$.
Let $\z = - \abs{\x}$. Then, 
\eq[eq:proof_thm_w_limits_gamma_1]{\ensuremath{\xord{k} = \abs{x_{d-k+1}} = 
-z_{d-k+1}, \quad
\xord{k+1} = \abs{x_{d-k}} = -z_{d-k}.}}
Let $\Lambda_a$, $\Lambda_b$ be the two subsets defined in 
\cref{thm:w_limits_gamma}, namely
\eq{\ensuremath{
    \Lambda_a \eqdef \setst{i \in \brs{d}}{\abs{x_i} < \xord{k}},\quad
     \Lambda_b \eqdef \setst{i \in \brs{d}}{\abs{x_i} = \xord{k}}.
     }}
Similarly, let $\tilde{\Lambda}_a$, $\tilde{\Lambda}_b$ be as defined for $\z$ 
in \cref{thm:theta_limits_idxset_def1} in the case $\gamma \app \infty$, with 
$k$ replaced by $d-k$,
\eq{\ensuremath{
    \tilde{\Lambda}_a \eqdef \setst{i \in \brs{d}}{z_i > z_{d-k}},\quad
    \tilde{\Lambda}_b \eqdef \setst{i \in \brs{d}}{z_i = z_{d-k}}.
    }}
Suppose first that $\xord{k} = \xord{k+1}$. Then by 
\cref{eq:proof_thm_w_limits_gamma_1},
$z_{d-k+1} = z_{d-k}$,
which implies that 
\begin{gather*}
\begin{aligned}
i \in \Lambda_a &\iff \abs{x_i} < \xord{k} &\iff z_i > z_{d-k+1} &\iff z_i > 
z_{d-k} &&\iff i \in \tilde{\Lambda}_a,
\\ i \in \Lambda_b &\iff \abs{x_i} = \xord{k} &\iff z_i = z_{d-k+1} &\iff z_i = 
z_{d-k} &&\iff i \in \tilde{\Lambda}_b.
\end{aligned}
\end{gather*}
Since in this case $\Lambda_a = \tilde{\Lambda}_a$ and $\Lambda_b = 
\tilde{\Lambda}_b$, Eq.~\cref{eq:w_at_gamma_infty}
follows directly from \cref{eq:penalty_by_aux} and 
\cref{thm:theta_limits_gamma}.

Next, suppose that $\xord{k} > \xord{k+1}$. Then
$\Lambda_a = \brc{1,\ldots,d-k}$. In this case, Eq.~\cref{eq:w_at_gamma_infty}, 
which we need to prove, simplifies to
\eq[eq:proof_w_limits_simple_limit]{\ensuremath{   
    \underset {\gamma \rightarrow \infty } {\lim} \ \wfuncidx{i}\ofx = 
    \twocase{1}{i \in \Lambda_a}{0}{\text{otherwise.}}
}  }
To prove the first part of \cref{eq:proof_w_limits_simple_limit}, note that 
by   
\cref{eq:proof_thm_w_limits_gamma_1}, $z_{d-k+1} < z_{d-k}$,
so
$$\Lambda_a = \brc{1,\ldots,d-k} = \tilde{\Lambda}_a \cup \tilde{\Lambda}_b.$$
Now, by \cref{eq:gsm_at_gamma_infty}, for $i \in \tilde{\Lambda}_a$,  
$\underset {\gamma \rightarrow \infty } {\lim} \ \wfuncidx{i}\ofx 
    = \lim_{\gamma \app \infty} \gsmiargs{d-k}{\gamma}{i} \ofz = 1.
$
If $i \in \tilde{\Lambda}_b$, then also by \cref{eq:gsm_at_gamma_infty},
\begin{equation*}
\underset {\gamma \rightarrow \infty } {\lim} \ \wfuncidx{i}\ofx 
    = \lim_{\gamma \app \infty} \gsmiargs{d-k}{\gamma}{i} \ofz = 
    \tfrac{d-k-\abs{\tilde{\Lambda}_a}}{\abs{\tilde{\Lambda}_b} } = 
    \tfrac{\abs{\tilde{\Lambda}_b}}{\abs{\tilde{\Lambda}_b}} = 1.
\end{equation*}
To conclude the proof, note that if $i \notin \Lambda_a$, then $i \notin 
\tilde{\Lambda}_a \cup \tilde{\Lambda}_b$, and by \cref{eq:gsm_at_gamma_infty},
\begin{equation}
\underset {\gamma \rightarrow \infty } {\lim} \ \wfuncidx{i}\ofx 
= \lim_{\gamma \app \infty} \gsmiargs{d-k}{\gamma}{i} \ofz = 0.
    \nonumber
\end{equation}
\end{proof}

\begin{proof}[Proof of \cref{thm:sum_w}]
\label{proof:sum_w}
This lemma follows directly by combining \cref{thm:sum_theta} with 
\cref{eq:penalty_by_aux}.
\end{proof}

\begin{proof}[Proof of \cref{thm:Ftwolg_majorizer}]
\label{proof:Ftwolg_majorizer}
Let $\gamma \in \brs{0,\infty}$. By \cref{thm:mu_convex},
for any $\z,\z_0 \in \R[d]$,
\eq{\ensuremath{\muargs{d-k}{-\gamma} \ofz 
\leq \muargs{d-k}{-\gamma} \br{\z_0} + \inprod{\gsmargs{d-k}{-\gamma} 
\br{\z_0}}{\z - \z_0}.}}
Let $\z = \abs{\x}$, $\z_0 = \abs{\xtilde}$. Combining the above with 
\cref{eq:penalty_by_aux,eq:def_taumaj} proves the 
\lcnamecref{thm:Ftwolg_majorizer}, since 
\eq{\ensuremath{\taukg \ofx \leq \taukg \br{\xtilde} + \inprod{\wfunc 
\br{\xtilde}}{\abs{\x} - \abs{\xtilde}} = \taumajkg \br{\x, \xtilde}.}}
\end{proof}

\subsection{Convergence analysis}
\label{appendix:proofs_convergence_analysis}
To analyze the convergence of \cref{alg:mm,alg:homotopy}, we first state and 
prove two auxiliary lemmas regarding the directional derivatives of the GSM 
penalty $\tauk\ofx$ and of its majorizer $\taumajkg \br{\x,\xtilde}$ at $\x = 
\xtilde$.
Note that when referring to a function $g \br{\x,\xtilde}$ of two variables, 
the directional derivative is with respect to the first variable,
\eq{\ensuremath{\nabla_{\vvec} g \br{\x,\xtilde} \eqdef \lim_{t \searrow 0} 
\frac{ g \br{\x+t\vvec, \xtilde} - g \br{\x,\xtilde} }{t}.
}}
\begin{lemma}\label{thm:dirdev_tau_phi_gamma_finite}
        For any $\x \in \R[d]$ and $0 \leq \gamma < \infty$, the directional 
        derivative in direction $\vvec \neq \zerovec$ of $\taukg \ofx$ and of 
        $\taumajkg\br{\x,\x}$ are equal and are given by  
        \eq[eq:dir_div_tau_finite_gamma]{\ensuremath{
        \nabla_{\vvec} \taukg \br{\x} = \nabla_{\vvec} \taumajkg \br{\x, \x} = 
        \sum_{i=1}^d w_{k,\gamma}^i \ofx \delta \br{x_i,v_i} v_i,}}
        where $\delta \br{\alpha,\beta}$ is defined for $\alpha, \beta \in \R$ 
        by
        \eq[eq:def_delta]{\ensuremath{
            \delta \br{\alpha,\beta} \eqdef 
            \twocase{\mathrm{sign}\br{\alpha}}{\alpha \neq 
            0,}{\mathrm{sign}\br{\beta}}{\alpha = 0.}}}
\end{lemma}
\begin{lemma}\label{thm:dirdev_tau_phi_gamma_infty}
    Let $\x, \vvec \in \R[d]$. If $\x$ is non-ambiguous, then 
    \cref{eq:dir_div_tau_finite_gamma} holds also for $\gamma = \infty$. For a 
    general $\x \in \R[d]$, let
    $$\Lambda_a \eqdef \setst{i }{\abs{x_i} < \xord{k}}, \qquad \Lambda_b 
    \eqdef \setst{i }{\abs{x_i} = \xord{k}}.$$
        Then
        \begin{align}
        \label{eq:dirdev_tau_infty}
        \nabla_{\vvec} \tauk \ofx = \summ[i \in \Lambda_a] \delta \br{x_i,v_i} 
        v_i + 
                \min \setst{\sum_{i \in \Lambda} \delta \br{x_i,v_i} v_i}
                {\twofloors{\Lambda \subseteq \Lambda_b,}{\abs{\Lambda} = 
                d-k-\abs{\Lambda_a}}}                
        \\ \label{eq:dirdev_taumaj_infty}
        \nabla_{\vvec} \taumajkinf \br{\x,\x} = \summ[i \in \Lambda_a] \delta 
        \br{x_i,v_i} v_i + 
                {\mean} \setst{\sum_{i \in \Lambda} \delta \br{x_i,v_i} v_i}
                {\twofloors{\Lambda \subseteq \Lambda_b,}{\abs{\Lambda} = 
                d-k-\abs{\Lambda_a}}}.
        \end{align}    
\end{lemma}
\Cref{thm:dirdev_tau_phi_gamma_finite,thm:dirdev_tau_phi_gamma_infty} lead to 
the following \lcnamecref{thm:dirdevs_are_equal}, which states that the 
objective $\Ftwolg \ofx$ and majorizer $\Gtwolg \br{\x, \xtilde}$ have the same 
first-order behavior about $\x=\xtilde$. This, in turn, leads to desirable 
convergence properties of \cref{alg:mm}, as outlined in 
\cref{thm:alg1_convergence_gamma_finite}.
\begin{corollary}\label{thm:dirdevs_are_equal}
    If $\gamma$ is finite or $\x$ is non-ambiguous, then
    \eq[eq:dirdev_F_G]{\ensuremath{ \nabla_{\vvec} \Ftwolg \br{\x} &= 
    \nabla_{\vvec} \Gtwolg \br{\x,\x} = \inprod{A\x-\y}{A\vvec} + \lambda 
    \sum_{i=1}^d \wfuncidx{i} \ofx \delta \br{x_i,v_i} v_i,
        \\  \nabla_{\vvec} \Fonelg \br{\x} &= \nabla_{\vvec} \Gonelg \br{\x,\x} 
        \\ &= \twocase{\inprod{\frac{A\x-\y}{\norm{A\x-\y}_2}}{A\vvec} + 
        \lambda \sum_{i=1}^d \wfuncidx{i} \ofx \delta \br{x_i,v_i} v_i}{A\x 
        \neq \y}{\norm{A\vvec}_2 + \lambda \sum_{i=1}^d \wfuncidx{i} \ofx 
        \delta \br{x_i,v_i} v_i}{A\x=\y.} }}
\end{corollary}
{We remark that \cref{eq:dirdev_F_G} with $\gamma = \infty$ does not generally 
hold for an ambiguous $\x$. In particular, $\nabla_{\vvec}\Ftwol \br{\x}$ may 
be strictly smaller than $\nabla_{\vvec}\Gtwolinf \br{\x,\x}$ for such $\x$. 
Thus, an ambiguous point that is stationary for $\Gtwolinf \br{\argdot,\x}$ may 
not be stationary for $\Ftwol\br{\argdot}$.}

\begin{proof}[Proof of \cref{thm:dirdev_tau_phi_gamma_finite}]
    \label{proof:dirdev_tau_phi_gamma_finite}
    By \cref{eq:penalty_by_aux}, for any $t > 0$,
    $$\taukg \br{\x + t \vvec} = \muargs{d-k}{-\gamma} \br{\abs{\x + t 
    \vvec}}.$$
    Since $\gamma < \infty$, $\muargs{d-k}{-\gamma} \ofz$ is infinitely 
    differentiable and 
    $\nabla \muargs{d-k}{-\gamma} \ofz= \gsmargs{d-k}{-\gamma} \ofz.$
    By Taylor's expansion of $\muargs{d-k}{-\gamma} \br{\z}$ at $\z = \abs{\x + 
    t \vvec}$ about $\z_0 = \abs{\x}$,
    \begin{gather*}
        \muargs{d-k}{-\gamma} \br{\abs{\x + t \vvec}} = 
        \\ \muargs{d-k}{-\gamma} \br{\abs{\x}} + \inprod{ 
        \gsmargs{d-k}{-\gamma} \br{\abs{\x}}  }{\abs{\x + t \vvec} - \abs{\x}} 
        + o \br{\norm{\abs{\x + t \vvec} - \abs{\x}}_2 },
    \end{gather*}
    By the triangle inequality,
    $\norm{ \abs{\x + t \vvec} - \abs{\x} }_2 \leq t \norm{\vvec}_2$.
 Therefore,
    \eq{\ensuremath{
        \muargs{d-k}{-\gamma} \br{\abs{\x + t \vvec}} = \muargs{d-k}{-\gamma} 
        \br{\abs{\x}} + \inprod{ \gsmargs{d-k}{-\gamma} \br{\abs{\x}}  
        }{\abs{\x + t \vvec} - \abs{\x}} + o \br{t}.
    }}
    Note that if $t$ is small enough, then for any nonzero $x_i$, 
    $\text{sign}\br{x_i + t v_i} = \text{sign}\br{x_i}$, and thus
    \eq[eq:small_t]{\ensuremath{
        \abs{x_i + t v_i} - \abs{x_i} = t \cdot \delta \br{x_i,v_i} v_i.}}
    Thus,
    \eq{\ensuremath{
        \nabla_{\vvec} \taukg \br{\x} &= \lim_{t \searrow 0}\frac{ \taukg 
        \br{\x + t \vvec} - \taukg \br{\x}}{t}
        = \lim_{t \searrow 0}\frac{ \muargs{d-k}{-\gamma} \br{\abs{\x + t 
        \vvec}} - \muargs{d-k}{-\gamma} \br{\abs{\x}}}{t}
        \\ &= \lim_{t \searrow 0} \brs{ \inprod{ \gsmargs{d-k}{-\gamma} 
        \br{\abs{\x}}  }{\frac{\abs{\x + t \vvec} - \abs{\x}}{t}} + \frac{o 
        \br{t}}{t} }
        \\ &= \sum_{i=1}^d \gsmiargs{d-k}{-\gamma}{i} \br{\abs{\x}} \delta 
        \br{x_i,v_i} v_i
        = \sum_{i=1}^d \wfuncidx{i} \ofx \delta \br{x_i,v_i} v_i.}}   
    Finally, by the definition of $\taumajkg$ in \cref{eq:def_taumaj},
    $$\taumajkg \br{\x + t \vvec, \x} = \taukg{\ofx} + \inprod{\wfunc 
    \ofx}{\abs{\x + t \vvec} - \abs{\x}}.$$
    Combining the above with \cref{eq:small_t} proves the 
    \lcnamecref{thm:dirdev_tau_phi_gamma_finite}, since
    \eq{\ensuremath{
        \nabla_{\vvec} \taumajkg \br{\x,\x}
        &= \lim_{t \searrow 0} \frac{ \taumajkg \br{\x + t \vvec, \x} - 
        \taumajkg \br{\x , \x} } { t }
         = \lim_{t \searrow 0} \frac{ \taumajkg \br{\x + t \vvec, \x} - \taukg 
         \br{\x} } { t }
        \\ &= \lim_{t \searrow 0} \frac{ \inprod{\wfunc \ofx}{\abs{\x + t 
        \vvec} - \abs{\x}}  } { t }
        = \sum_{i=1}^d \wfuncidx{i} \ofx \delta \br{x_i,v_i} v_i.
    }} 
\end{proof}

\begin{proof}[Proof of \cref{thm:dirdev_tau_phi_gamma_infty}]
    \label{proof:dirdev_tau_phi_gamma_infty}
    Let $t > 0$. Define the sets of indices $\hat{\Lambda}_a$, 
    $\hat{\Lambda}_b$ for the vector $\x + t \vvec$ the same way $\Lambda_a$ 
    and $\Lambda_b$ are defined for $\x$. Namely,
\eq{\ensuremath{\hat{\Lambda}_a \eqdef \setst{i }{\abs{x_i + t v_i} < \abs{\x + 
t\vvec}_\br{k}},
    \quad
    \hat{\Lambda}_b \eqdef \setst{i }{\abs{x_i + t v_i} = \abs{\x + 
    t\vvec}_\br{k}}.
}}
Then, for any $j \in \Lambda_b$, $\tauk \br{\x}$ can be expressed by
\eq[eq:tau_difference_a]{\ensuremath{\tauk \br{\x} &= \summ[i \in \Lambda_a] 
\abs{x_i} + \br{d-k-\abs{\Lambda_a}} \abs{x_j}.
}}
Similarly, for any $j \in \hat{\Lambda}_b$, 
\eq[eq:tau_difference_b]{\ensuremath{ \tauk \br{\x + t\vvec} &= \summ[i \in 
\hat{\Lambda}_a] \abs{x_i + t v_i} + \br{d-k-\abs{\hat{\Lambda}_a}} \abs{x_{j} 
+ t v_{j}}.}}
For a small enough $t>0$, any strict inequality between entries of $\abs{\x}$ 
is maintained in $\abs{\x + t\vvec}$. Thus, if $\abs{x_i} < \xord{k}$, 
then for sufficiently small $t$, 
$\abs{x_i + t v_i} < \abs{\x + t \vvec}_\br{k}$. Hence,
\eq[eq:sec_dirdev_inf_proof_setineq1]{\ensuremath{
    \Lambda_a \subseteq \hat{\Lambda}_a.
}}
Similarly, for all indices $i$ such that $\abs{x_i} > \xord{k}$, $\abs{x_i + t 
v_i} > \abs{\x + t \vvec}_\br{k}$. Thus $\br{\Lambda_a \cup \Lambda_b}^c 
\subseteq \br{\hat{\Lambda}_a \cup \hat{\Lambda}_b}^c$.
Therefore,
$\hat{\Lambda}_a \cup \hat{\Lambda}_b \subseteq \Lambda_a \cup \Lambda_b$.
Since $\hat{\Lambda}_a$ and $\hat{\Lambda}_b$ are disjoint, this implies that  
\eq[eq:sec_dirdev_inf_proof_setineq2]{\ensuremath{\hat{\Lambda}_b =\br{ 
\hat{\Lambda}_a \cup \hat{\Lambda}_b} \setminus \hat{\Lambda}_a \subseteq 
\br{\Lambda_a \cup \Lambda_b} \setminus \Lambda_a = \Lambda_b.}}
Also note that
\eq[eq:sec_dirdev_inf_proof_setineq3]{\ensuremath{\hat{\Lambda}_a \setminus 
\Lambda_a \subseteq \br{\hat{\Lambda}_a \cup \hat{\Lambda}_b} \setminus 
\Lambda_a \subseteq \br{\Lambda_a \cup \Lambda_b} \setminus \Lambda_a = 
\Lambda_b.}}

We now express $\tauk \br{\x + t \vvec}$ of \cref{eq:tau_difference_b} in terms 
of $\Lambda_a$ and $\Lambda_b$. From \cref{eq:sec_dirdev_inf_proof_setineq1}, 
it follows that the three sets $\brc{\Lambda_a, \hat{\Lambda}_a \setminus 
\Lambda_a, \hat{\Lambda}_b}$ form a disjoint partition of $\hat{\Lambda}_a \cup 
\hat{\Lambda}_b$. Consider the formula for $\tauk \br{\x + t \vvec}$ in 
\cref{eq:tau_difference_b} as a sum of $d-k$ terms: $\abs{\hat{\Lambda}_a}$ of 
which originate from $\hat{\Lambda}_a$ and the remaining 
$d-k-\abs{\hat{\Lambda}_a}$ originate from $\hat{\Lambda}_b$. By 
\cref{eq:sec_dirdev_inf_proof_setineq1}, the sum must contain all the elements 
of $\Lambda_a$. The remaining $d-k-\abs{\Lambda_a}$ elements belong to 
$\hat{\Lambda}_a \setminus \Lambda_a$ or $\hat{\Lambda}_b$, both of which are 
subsets of $\Lambda_b$, according to 
\cref{eq:sec_dirdev_inf_proof_setineq2,eq:sec_dirdev_inf_proof_setineq3}. 
Necessarily, these remaining elements are exactly the $d-k-\abs{\Lambda_a}$ 
smallest-magnitude entries of $\abs{\x + t \vvec}$ among all entries with 
indices in $\Lambda_b$. Therefore, $\tauk \br{\x + t\vvec}$ can be expressed by
\eq[eq:tau_difference_c]{\ensuremath{\tauk \br{\x + t\vvec} &= 
    \summ[i \in \Lambda_a] \abs{x_i + t v_i} 
    + \min \setst{\sum_{i \in \Lambda}\abs{x_{i} + t v_{i}}}
    {\Lambda \subseteq \Lambda_b,\ \abs{\Lambda} = d-k-\abs{\Lambda_a}}.}}
Since $\abs{x_i}$ is the same for any $i \in \Lambda_b$, we can reformulate 
\cref{eq:tau_difference_a} and express $\tauk \br{\x}$ by 
\eq[eq:tau_difference_d]{\ensuremath{\tauk \br{\x} &= 
    \summ[i \in \Lambda_a] \abs{x_i} 
    + \min \setst{\sum_{i \in \Lambda}\abs{x_{i}}}
    {\Lambda \subseteq \Lambda_b,\ \abs{\Lambda} = d-k-\abs{\Lambda_a}}.}}
Since any $\Lambda \subseteq \Lambda_b$ of size $d-k-\abs{\Lambda_a}$ attains 
the minimum in \cref{eq:tau_difference_d}, {we can subtract the right-hand 
sides of \cref{eq:tau_difference_c,eq:tau_difference_d} inside the minimum term 
and get} 
\eq[eq:tau_difference_e]{\ensuremath{ &\tauk \br{\x + t \vvec} - \tauk \br{\x} 
= \summ[i \in \Lambda_a] {\br{\abs{x_i + t v_i} - \abs{x_i}}}\ + 
    \\ &\min \setst{\sum_{i \in \Lambda} \br{\abs{x_{i} + t v_i} - \abs{x_{i}}}}
    {\Lambda \subseteq \Lambda_b,\ \abs{\Lambda} = d-k-\abs{\Lambda_a}}
    .}}
Suppose that $t>0$ is small enough. Then plugging \cref{eq:small_t} into 
\cref{eq:tau_difference_e} yields
\begin{equation}
\begin{gathered}\label{eq:tauk_finite_diff}
    \frac{\tauk \br{\x + t \vvec} - \tauk \br{\x}}{t} = 
    \\ \summ[i \in \Lambda_a] \delta \br{x_i,v_i} v_i + 
    \min \setst{\sum_{i \in \Lambda} \delta \br{x_i,v_i} v_i}
    {\Lambda \subseteq \Lambda_b,\ \abs{\Lambda} = d-k-\abs{\Lambda_a}}.
\end{gathered}
\end{equation}
Therefore, \cref{eq:dirdev_tau_infty} holds.

We now turn to calculate $\nabla_{\vvec} \taumajkinf \br{\x, \x}$.    
By the definition of $\taumajkinf$,
\eq{\ensuremath{\taumajkinf \br{\x+t\vvec, \x} = \tauk \br{\x} + \inprod{ 
\wfunckinfty \br{\x}} {\abs{\x+t\vvec}-\abs{\x} }.}}
Thus,
\begin{equation}\label{taumajk_dirdev_lim}
\begin{split}
    \nabla_{\vvec} \taumajkinf \br{\x, \x} 
     &= \lim_{t \searrow 0}\frac{\taumajkinf \br{\x+t\vvec, \x} - \taumajkinf 
     \br{\x, \x}}{t}
     = \lim_{t \searrow 0}\frac{\taumajkinf \br{\x+t\vvec, \x} - \tauk 
     \br{\x}}{t}
     \\ &= \lim_{t \searrow 0}\frac{\inprod{ \wfunckinfty \br{\x}} 
     {\abs{\x+t\vvec}-\abs{\x} }}{t}
     = \summ[i=1][d] w_{k,\infty}^i\ofx \cdot \lim_{t \searrow 0}
    \frac{ \abs{x_i + t v_i} - \abs{x_i}}{t}    
    \\ &\overset{\text{(a)}}{=} \summ[i=1][d] w_{k,\infty}^i\ofx \cdot \delta 
    \br{x_i,v_i} v_i     
    \\ &\overset{\text{(b)}}{=} \summ[i \in \Lambda_a] \delta \br{x_i,v_i} v_i 
    + \frac{d-k-\abs{\Lambda_a}}{\abs{\Lambda_b} } \sum_{i \in \Lambda_b} 
    \delta \br{x_i,v_i} v_i,
\end{split}
\end{equation}    
where (a) follows from \cref{eq:small_t} and (b) from 
\cref{eq:w_at_gamma_infty}.
Consider \cref{taumajk_dirdev_lim} and note that
\eq{\ensuremath{\frac{d-k-\abs{\Lambda_a}}{\abs{\Lambda_b} } \sum_{i \in 
\Lambda_b} \delta \br{x_i,v_i} v_i
    = \mean \setst{\sum_{i \in \Lambda} \delta \br{x_i,v_i} v_i}
        {\Lambda \subseteq \Lambda_b,\ \abs{\Lambda} = d-k-\abs{\Lambda_a}},}}
since the above left-hand side is the average of the set $\setst{\delta 
\br{x_i,v_i} v_i}{i \in \Lambda_b}$ multiplied by $d-k-\abs{\Lambda_a}$, and 
the right-hand side is the average of all possible sums of 
$d-k-\abs{\Lambda_a}$ members of the same set.
Hence, \cref{eq:dirdev_taumaj_infty} holds.

Lastly, suppose that $\x$ is non-ambiguous. Since $\xord{k+1} < \xord{k}$, 
$\abs{\Lambda_a} = d-k$. Hence, the sums in 
\cref{eq:dirdev_tau_infty,eq:dirdev_taumaj_infty} are over sets $\Lambda$ of 
size 
${d-k-\abs{\Lambda_a} = 0}$, and thus vanish, making the formulae in  
\cref{eq:dirdev_tau_infty,eq:dirdev_taumaj_infty} coincide.
\end{proof}

Before proving our main convergence result for \cref{alg:mm}, namely 
\cref{thm:alg1_convergence_gamma_finite}, we prove the following auxiliary 
\lcnamecref{thm:compact_sublevelset}.
\begin{lemma}\label{thm:compact_sublevelset}
    Suppose that any $k$ columns of $A$ are linearly independent.
    Then for all $\lambda>0$, $\gamma \in \brs{0,\infty}$, $c \in \R$, the set 
    $S = \brc{\x \in \R[d] \ \Bigm\vert \ \Ftwolg \ofx \leq c}$ is compact. 
\end{lemma}
\begin{proof}
    Since $\Ftwolg$ is continuous, the set $S$ is closed. From 
    \cref{thm:tau_limits_gamma}, it follows that $\Ftwolg\ofx \geq \Ftwol\ofx$ 
    for all $\x \in \R[d]$ and $\gamma\in\brs{0,\infty}$. It thus suffices to 
    prove the lemma for $\gamma = \infty$. Suppose by contradiction that there 
    exists a sequence $\brc{\x^t}_{t=1}^\infty$ such that $\norm{\x^t} \app 
    \infty$ and yet $\Ftwol\br{\x^t} \leq c$. Consider the sequence 
    $\brc{\projk\br{\x^t}}_{t=1}^\infty$ and recall that $\tauk\ofx =  
    \norm{\x-\projk\ofx}_1$. Since $\Ftwol\ofx \geq \lambda \tauk\ofx$, the 
    assumption $\Ftwol\br{\x^t}\leq c$ for all $t$ implies that $\norm{\x^t - 
    \projk\br{\x^t}}_1 \leq\frac{c}{\lambda}$. This, in turn, implies that 
    $\norm{\projk\br{\x^t}}_1 \app[t \app \infty] \infty$. Now, let $\sigma_1$ 
    be the largest singular value of $A$, and let $\sigma_{\min}$ be the 
    smallest singular value of any $n \times k$ sub-matrix of $A$. The 
    assumption that any $k$ columns of $A$ are linearly independent implies 
    that $\sigma_{\min} > 0$. Now,
    \eq{\ensuremath{\Ftwol{\ofx} \geq \frac{1}{2} \norm{A\x-\y}_2^2 
    &\geq \frac{1}{2}\Bigbr{\norm{A \projk\ofx}_2 - \norm{\y}_2 - 
    \norm{A\br{\x-\projk\ofx}}_2 }^2 
    \\
    &\geq \frac{1}{2}\Bigbr{\sigma_{\min} \norm{\projk\ofx}_2 - \norm{\y}_2 - 
    \sigma_1\norm{\x-\projk\ofx}_2 }^2.
    }}
    Plugging $\x=\x^t$ into the inequality above, with $\norm{\projk\br{\x^t}} 
    \app[t \app \infty] \infty$ while $\norm{\x - \projk\ofx}_2$ is bounded, 
    implies that $\Ftwol\br{\x^t} \app \infty$, which is a contradiction.    
\end{proof}

We are now ready to prove \cref{thm:alg1_convergence_gamma_finite}. Our proof 
is based on \cite[Theorem 1]{razaviyayn2013unified}.
\begin{proof}[Proof of \cref{thm:alg1_convergence_gamma_finite}]
\label{proof:alg1_convergence_gamma_finite}
 Since $\gamma$ is finite, $\Gtwolg \br{\x,\xtilde}$ is continuous in 
 $\br{\x,\xtilde}$. This fact, together with 
 \cref{thm:dirdev_tau_phi_gamma_finite}, imply that the objective $\Ftwolg 
 \ofx$ and majorizer $\Gtwolg \br{\x,\xtilde}$ satisfy the conditions for 
 Theorem~1 in \cite{razaviyayn2013unified}, which guarantees that any partial 
 limit of $\brc{\x^t}_{t=0}^\infty$ is a stationary point of $\Ftwolg$. We now 
 show that all partial limits belong to the same level\nbdash set of $\Ftwolg$. 
 {Note that the sequence $\brc{\Ftwolg\br{\x^t}}_{t=0}^{\infty}$ is 
 lower-bounded by zero and is monotone decreasing by properties of the MM 
 scheme. Thus, $\Ftwolg\br{\x^t}$ converges to a limit $\alpha \geq 0$ as $t 
 \app \infty$. Let $\xtilde$ be a partial limit of $\x^t$ and let 
 $\brc{\x^{t_l}}_{l=0}^{\infty}$ be a subsequence that converge to $\xtilde$. 
 By continuity of $\Ftwolg$,
 all partial limits have the same objective value, since
 $$\Ftwolg\br{\xtilde} 
 = \Ftwolg\br{\limit[l \app \infty] \x^{t_l} }
 = \limit[l \app \infty] \Ftwolg\br{ \x^{t_l}} 
 = \limit[t \app \infty] \Ftwolg\br{ \x^{t}} 
 = \alpha.$$
}
By \cref{thm:compact_sublevelset}, the conditions of \cite[Corollary 
1]{razaviyayn2013unified} hold, so $\lim_{t \rightarrow \infty} d \br{\x^t, 
\xstatset} = 0$.
\end{proof}

\begin{proof}[Proof of \cref{thm:stationary_pt_is_localmin}]
    \label{proof:stationary_pt_is_localmin}
    Suppose by contradiction that $\xtilde \in \R[d]$ is a stationary point of 
    $\Ftwol$, but not a local minimum. Then there exists a sequence 
    $\brc{\x^t}_{t=1}^\infty$ such that $\x^t \app[t \app \infty] \xtilde$ and 
    $\forall t \geq 1$, $\Ftwol \br{\xtilde} > \Ftwol \br{\x^t}$. 
    Let $t$ be large enough such that the vectors $\abs{\xtilde}$ and 
    $\abs{\x^t}$ are sorted in decreasing order by the same permutation and, if 
    $\tilde{x}_i \neq 0$, then 
    $\text{sign}\br{x^t_i}=\text{sign}\br{\tilde{x}_i}$. Denote $\vvec \eqdef 
    \x^{t} - \xtilde$. 
    Then
    \eq{\ensuremath{%
        0 > &\Ftwol\br{\x^t} - \Ftwol\br{\xtilde}
        = \frac{1}{2} \br{\norm{A \x^t - \y}_2^2 - \norm{A \xtilde - \y}_2^2} + 
        \lambda \br{\tauk\br{\x^t} - \tauk\br{\xtilde}}
        \\ \overset{\mbox{(a)}}{\geq} &\inprod{A\xtilde-\y}{A\vvec} + \lambda 
        \br{\tauk\br{\x^t} - \tauk\br{\xtilde}}
        \overset{\mbox{(b)}}{=} \inprod{A\xtilde-\y}{A\vvec} + \ldots
        \\ &\lambda \Biggbrs{\summ[i \in \Lambda_a] \delta \br{\tilde{x}_i,v_i} 
        v_i + \min \setst{\sum_{i \in \Lambda} \delta \br{\tilde{x}_i,v_i} 
        v_i}  {\Lambda \subseteq \Lambda_b,\ \abs{\Lambda} = 
        d-k-\abs{\Lambda_a}}}
        \\ &\overset{\mbox{(c)}}{=} \nabla_{\vvec} \Ftwol \ofxhat,
    }}
    where $\Lambda_a = \setst{i}{\abs{\tilde{x}_i} < \abs{\tilde{x}_\br{k}}}$, 
    $\Lambda_b = \setst{i}{\abs{\tilde{x}_i} = \abs{\tilde{x}_\br{k}}}$, (a) 
    follows from the convexity of $\norm{A\xtilde-\y}_2^2$, (b) follows from a 
    similar argument to the one used for \cref{eq:tauk_finite_diff}, and (c) is 
    by \cref{eq:dirdev_tau_infty}. Thus, $\Ftwol$ has a negative directional 
    derivative at $\xtilde$, contradicting the assumption that $\xtilde$ is a 
    stationary point.
\end{proof}

To analyze the convergence of \cref{alg:mm} at $\gamma = \infty$, we first 
prove an auxiliary \lcnamecref{thm:limit_of_stationaries}.
\begin{lemma}\label{thm:limit_of_stationaries}
    {Let $\xtilde$ be a limit of stationary points of $\Ftwol$, and suppose 
    that $\xtilde$ is non-ambiguous. Then $\xtilde$ is stationary.}
\end{lemma}
\begin{proof}
    {Let $\brc{\x^t}_{t=1}^\infty$ be a sequence of stationary points of 
    $\Ftwol$ such that $\x^t \overset{t \app \infty}{\longrightarrow} \xtilde$. 
    We shall show that if $\xtilde$ is non-ambiguous,
    then 
    $\nabla_{\vvec}\Ftwol\br{\xtilde} \geq 0$
    for any $\vvec\in \R[d]$.}
    
    {To this end consider a fixed $\vvec\in \R[d]$ and let $i \in \brs{d}$. If 
    $\tilde{x}_i \neq 0$, then for a large enough $t$, $\mbox{sign}\br{x^t_i} = 
    \mbox{sign}\br{\tilde{x}_i}$ and thus
    using the definition of the function
    $\delta(\alpha,\beta)$ in \cref{eq:def_delta}, 
    }
    \eq{\ensuremath{\limit[t\app\infty]\delta\br{x^t_i,v_i}v_i = 
    \delta\br{\tilde{x}_i,v_i}v_i.}}
    {If $\tilde{x}_i = 0$, then for all $t \geq 1$, 
    \begin{equation*}
    \delta\br{x^t_i,v_i}v_i 
        \leq \abs{v_i}
        = \textup{sign}\br{v_i}v_i
        = \delta\br{0,v_i}v_i
        = \delta\br{\tilde{x}_i,v_i}v_i.
    \end{equation*}    
    In both cases,}
    \eq[eq:proof_limit_of_stationaries_a]{\ensuremath{\limsup_{t\app\infty}\delta\br{x^t_i,v_i}v_i
     \leq \delta\br{\tilde{x}_i,v_i}v_i.}}

    {Suppose indeed that $\xtilde$ is non-ambiguous. Then there exists some 
    finite $T$ such that for any $t > T$, $\x^t$ is non-ambiguous and the 
    top\nbdash $k$ indices of $\x^t$ are exactly those of $\xtilde$. Thus, for 
    $t > T$, $\wfunckinfty\br{\x^t} = \wfunckinfty\br{\xtilde}$.}
    Therefore,
    \begin{equation*}
    \begin{split}
    \nabla_{\vvec} \Ftwol\br{\xtilde} 
    &\overset{\mbox{(a)}}{=} \inprod{A\xtilde-\y}{A\vvec} + \lambda 
    \sum_{i=1}^d \wfuncidx{i} \br{\xtilde} \delta \br{\tilde{x}_i,v_i} v_i
    \\ &\overset{\mbox{(b)}}{\geq} \inprod{A\xtilde-\y}{A\vvec} + \lambda 
    \sum_{i=1}^d \wfuncidx{i} \br{\xtilde} \limsup_{t \app \infty} 
    \delta\br{x^{t_i},v_i} v_i
    \\ &\overset{\mbox{ }}{=} \limit[t\to\infty] 
    \Bigbrs{\inprod{A\x^t-\y}{A\vvec}} + \lambda \sum_{i=1}^d \limsup_{t \app 
    \infty} \Bigbrs{ \wfuncidx{i} \br{\x^t} \delta\br{x^{t_i},v_i} v_i }    
    \\ &\overset{\mbox{(c)}}{\geq} \limsup_{t\to\infty} 
    \Bigbrs{\inprod{A\x^t-\y}{A\vvec} + \lambda \sum_{i=1}^d \wfuncidx{i} 
    \br{\x^t} \delta\br{x^{t_i},v_i} v_i }    
    \\ &\overset{\mbox{(d)}}{=}\limsup_{t \app \infty} \nabla_{\vvec} 
    \Ftwol\br{\x^t} \overset{\mbox{(e)}}{\geq} 0,
    \end{split}
    \end{equation*}
    {where (a), (d) follow from the formula for $\nabla_{\vvec} \Ftwol$ in 
    \cref{eq:dirdev_F_G}, (b) follows from 
    \cref{eq:proof_limit_of_stationaries_a}, (c) holds since a sum of limsups 
    is greater or equal to the corresponding limsup of sums, and (e) follows 
    from $\x^t$ being a stationary point for all $t$. Therefore, $\xtilde$ is a 
    stationary point of $\Ftwol$.}
\end{proof}

The following \lcnamecref{thm:alg1_convergence_gamma_infty} describes the 
convergence of \cref{alg:mm} at $\gamma = \infty$. 
\begin{lemma}\label{thm:alg1_convergence_gamma_infty}
    Suppose that any $k$ columns of $A$ are linearly independent. Let 
    $\brc{\x^t}_{t=0}^{\infty}$ be the iterates of \cref{alg:mm} with $\gamma = 
    \infty$. Then, starting from any initial point,
    \begin{enumerate}
        \item \label{item:alg1_convergence_gamma_infty_part1} There exists a 
        finite number of iterations $T$ such that for any $t \geq T$, 
        $\Ftwol\br{\x^t} = \Ftwol\br{\x^T}$, and if $\x^t$ is non-ambiguous, it 
        is a local minimum of $\Ftwol$.
        \item \label{item:alg1_convergence_gamma_infty_part2} The sequence 
        $\brc{\x^t}_{t=0}^{\infty}$ is bounded, and any non-ambiguous partial 
        limit of $\x^t$ is a local minimum of $\Ftwol$. 
    \end{enumerate}
\end{lemma}

\begin{proof}
    {At $\gamma = \infty$, by \cref{eq:w_at_gamma_infty} the image of 
    $\w_{k,\infty} \ofx$ over $\x \in \R[d]$ is finite. Moreover, 
    $\w_{k,\infty} \ofx$ is uniquely determined by $\x$. In turn, the function 
    $\Gtwolinf\br{\argdot,\x^{t}}$ is uniquely determined by $\w_{k,\infty} 
    \br{\x^t}$. Thus, only a finite number of distinct functions 
    $\Gtwolinf\br{\argdot,\x^{t}}$ are minimized in the course of 
    \cref{alg:mm}.} Therefore there exists $T > 0$ such that for any $t \geq 
    T$, there is a $\tilde{t} < T$ for which $\Gtwolinf\br{\x,\x^{t}} = 
    \Gtwolinf\br{\x,\x^{\tilde{t}-1}}$ $\forall \x \in \R[d]$. For $t \geq T$ 
    with corresponding $\tilde{t} < T$,
    \begin{equation*}
    \begin{split}
    \Ftwol\br{\x^{t}}
        &\overset{(a)}{=} \Gtwolinf\br{\x^{t},\x^{t}}
        \overset{(b)}{\geq} \Gtwolinf\br{\x^{t+1},\x^{t}}
        \\
        &\overset{(c)}{=} \Gtwolinf\br{\x^{t+1},\x^{\tilde{t}-1}}
        \overset{(d)}{\geq} \Gtwolinf\br{\x^{\tilde{t}},\x^{\tilde{t}-1}}
        \overset{(e)}{\geq} \Ftwol\br{\x^{\tilde{t}}},
        \end{split}
    \end{equation*}
    with (a),(e) hold since $\Gtwolinf$ is a majorizer of $\Ftwol$, (b), (d) 
    results from $\x^{t+1}$, $\x^{\tilde{t}}$ being global minimizers of 
    $\Gtwolinf\br{\argdot,\x^t}$, $\Gtwolinf\br{\argdot,\x^{\tilde{t}-1}}$ 
    respectively, and (c) follows from our choice of $\tilde{t}$.
    On the other hand, since $t \geq T \geq \tilde{t}$, $\Ftwol\br{\x^t} \leq 
    \Ftwol\br{\x^T} \leq \Ftwol\br{\x^{\tilde{t}}}$. 
    Thus, for any $t \geq T$,
    \eq{\ensuremath{\Ftwol\br{\x^{t}} = \Ftwol\br{\x^{T}}.}}  
    Moreover, since
    \eq{\ensuremath{\Ftwol\br{\x^{t}} = \Ftwol\br{\x^{t+1}}
        \leq \Gtwolinf\br{\x^{t+1},\x^{t}}
        \leq \Gtwolinf\br{\x^{t},\x^{t}}
        = \Ftwol\br{\x^{t}},
    }}
    we have that 
    \eq{\ensuremath{\Gtwolinf\br{\x^{t},\x^{t}} = \Gtwolinf\br{\x^{t+1},\x^{t}} 
    = \min_{\x \in \R[d]}\Gtwolinf\br{\x,\x^{t}}.}}
    Thus, $\x^t$ is a global minimum of $\Gtwolinf\br{\argdot,\x^t}$. By 
    \cref{thm:dirdevs_are_equal}, if $\x^t$ is non-ambiguous, it is a 
    stationary point of $\Ftwol$. Recall that by 
    \cref{thm:stationary_pt_is_localmin}, all stationary  points of $\Ftwol$ 
    are local minima. Thus, for any $t \geq T$, the objective $\Ftwol\br{\x^t}$ 
    becomes constant, and $\x^t$ is either a local minimum of $\Ftwol$, or an 
    ambiguous vector. Thus, part~\ref{item:alg1_convergence_gamma_infty_part1} 
    of the \lcnamecref{thm:alg1_convergence_gamma_infty} is proven.
    
    {Combining part 1 with \cref{thm:compact_sublevelset} implies that the 
    iterates $\x^t$ are bounded and thus have at least one partial limit. Let 
    $\xtilde$ be such a partial limit, and let $\brc{\x^{t_i}}_{i=1}^\infty$ be 
    a subsequence that converges to it. If $\x^{t_i}$ contains an infinite 
    number of ambiguous points, then $\xtilde$ is also ambiguous. On the other 
    hand, if $\x^{t_i}$ contains only a finite number of ambiguous points, then 
    by part 1, $\x^{t_i}$ contains an infinite number of local minima of 
    $\Ftwol$. Thus, $\xtilde$ is a limit of local minima of $\Ftwol$, and by 
    \cref{thm:limit_of_stationaries,thm:stationary_pt_is_localmin}, it is 
    either ambiguous or a local minimum. In conclusion, we have shown that in 
    all cases $\xtilde$ is either ambiguous or a local minimum of $\Ftwol$, and 
    part~\ref{item:alg1_convergence_gamma_infty_part2} is proven.}   
\end{proof}

\begin{proof}[Proof of \cref{thm:alg2_convergence}]
    \label{proof:alg2_convergence}    
    By \cref{thm:alg1_convergence_gamma_finite}, for any $r$, $\x^*_r$ is a 
    stationary point of $\mathrm{F}_{\lambda,\gamma_r}$. Since $\gamma_{r+1} > 
    \gamma_r$, by \cref{thm:tau_limits_gamma}, 
    $\mathrm{F}_{\lambda,\gamma_{r+1}}\br{\x^*_{r}} \leq 
    \mathrm{F}_{\lambda,\gamma_{r}}\br{\x^*_{r}}$. Since \cref{alg:homotopy} is 
    initialized at $\x^*_{r}$ when solving \cref{pr:P2lg} with $\gamma = 
    \gamma_{r+1}$, $\mathrm{F}_{\lambda,\gamma_{r+1}} \br{\x^*_{r+1}} \leq 
    \mathrm{F}_{\lambda,\gamma_{r+1}}\br{\x^*_{r}}$. Hence, 
    $\mathrm{F}_{\lambda,\gamma_{r+1}} \br{\x^*_{r+1}} \leq 
    \mathrm{F}_{\lambda,\gamma_{r}}\br{\x^*_{r}}$.    
    Thus, part~\ref{item:alg2_convergence_part1} of the 
    \lcnamecref{thm:alg2_convergence} is proven.
    
    If $\gamma_r = \infty$ for some $r$, then since 
    $\brc{\gamma_r}_{r=0}^\infty$ is monotone increasing, $\gamma_{\tilde{r}} = 
    \infty$ for all $\tilde{r}> r$, and parts~\ref{item:alg2_convergence_part2} 
    and~\ref{item:alg2_convergence_part3} follow from 
    \cref{thm:alg1_convergence_gamma_infty}. Otherwise, suppose that $\gamma_r$ 
    is finite for all $r$. Let $C \eqdef \setst{\x \in \R[d]}{\Ftwol \ofx \leq 
    \mathrm{F}_{\lambda,\gamma_0} \br{\x^*_0}}$. By the above, the sequence 
    $\brc{\x^*_r}_{r=0}^\infty$ is contained in $C$, and by 
    \cref{thm:compact_sublevelset}, $C$ is compact. Let $\xstar$ be any partial 
    limit of $\x^*_r$, and suppose that $\xstar$ is non-ambiguous. We shall 
    show that $\xstar$ is a local minimum of $\Ftwol$.
    
    First, we show there exists $\delta > 0$ such that $\wfunc \ofx \app[\gamma 
    \to \infty] \w_{k,\infty} \ofx$ uniformly on the $\ell_{\infty}$ ball
        \eq{\ensuremath{B_\delta \eqdef \setst{\x \in 
        \R[d]}{\norm{\x-\xstar}_\infty \leq \delta}.}}
    Let $\Lambda_{\min} \subset \brs{d}$ be the (unique) index-set 
    corresponding to the $d-k$ smallest-magnitude entries of $\xstar$. Let 
    $\delta \eqdef \frac{\abs{\x^*}_{\br{k}}-\abs{\x^*}_{\br{k+1}}}{3}$. Since 
    $\xstar$ is non-ambiguous, $\delta > 0$.    
    Let $\x \in \R[d]$ such that $\norm{\x-\x^*}_\infty < \delta$. Then the 
    $d-k$ smallest-magnitude entries of $\x$ are indexed by the same set 
    $\Lambda_{\min}$. Moreover, for any $i \in \Lambda_{\min}$, $j \notin 
    \Lambda_{\min}$,
    \eq[eq:proof_alg2_convergence_a]{\ensuremath{\abs{x_i} \leq \abs{x^*_i} + 
    \delta \leq  \br{\abs{x^*_j} - 3\delta} + \delta < \abs{x_j} -  \delta.}}
    Consider an arbitrary index-set $\Lambda \subseteq \brs{d}$ of size $d-k$, 
    such that $\Lambda \neq \Lambda_{\min}$. Then
    \eq[eq:proof_alg2_convergence_b]{\ensuremath{\summ[i \in 
    \Lambda_{\min}]\abs{x_i} - \summ[j \in \Lambda]\abs{x_j} 
        = \summ[i \in \Lambda_{\min} \setminus \Lambda]\abs{x_i} - \summ[j \in 
        \Lambda \setminus \Lambda_{\min}]\abs{x_j}
        \leq - \abs{\Lambda \setminus \Lambda_{\min}} \cdot \delta \leq 
        -\delta,}}
    Therefore, for all $\Lambda \neq \Lambda_{\min}$ of size $d-k$,
    \eq[eq:proof_alg2_convergence_c]{\ensuremath{\exp\br{-\gamma 
    \br{\summ[j\in\Lambda] \abs{x_j} - \summ[j\in\Lambda_{\min}] \abs{x_j}}} 
    \leq \exp\br{-\gamma \delta} \app[\gamma \app \infty] 0. }}
    Thus, for any such $\Lambda \neq \Lambda_{\min}$,
    \eq[eq:proof_alg2_convergence_d]{\ensuremath{
    &\frac{\exp\br{-\gamma \summ[j\in\Lambda] 
    \abs{x_j}}}{\summ[\abs{\tilde{\Lambda}} = d-k]\exp\br{-\gamma 
    \summ[j\in\tilde{\Lambda}] \abs{x_j}}}
    = \frac{\exp\br{-\gamma \br{\summ[j\in\Lambda] \abs{x_j} - 
    \summ[j\in\Lambda_{\min}] \abs{x_j}}}}{\summ[\abs{\tilde{\Lambda}} = 
    d-k]\exp\br{-\gamma \br{\summ[j\in\tilde{\Lambda}] \abs{x_j} - 
    \summ[j\in\Lambda_{\min}] \abs{x_j} }}}
    \\ \overset{(a)}{\leq} &\frac{\exp\br{-\gamma 
    \delta}}{\summ[\abs{\tilde{\Lambda}} = d-k]\exp\br{-\gamma 
    \br{\summ[j\in\tilde{\Lambda}] \abs{x_j} - \summ[j\in\Lambda_{\min}] 
    \abs{x_j} }}}
    \overset{(b)}{\leq} \exp\br{-\gamma \delta} \app[\gamma \app \infty] 0,
    }}
    with (a) following from \cref{eq:proof_alg2_convergence_c}, and (b) holding 
    since the sum at the denominator contains the term for 
    $\tilde{\Lambda}=\Lambda_{\min}$, which equals 1. On the other hand, for 
    $\Lambda = \Lambda_{\min}$,
    \eq{\ensuremath{
    1 \geq &\frac{\exp\br{-\gamma \summ[j\in\Lambda_{\min}] 
    \abs{x_j}}}{\summ[\abs{\tilde{\Lambda}} = d-k]\exp\br{-\gamma 
    \summ[j\in\tilde{\Lambda}] \abs{x_j}}}
    = \frac{1}{\summ[\abs{\tilde{\Lambda}} = d-k]\exp\br{-\gamma 
    \br{\summ[j\in\tilde{\Lambda}] \abs{x_j} - \summ[j\in\Lambda_{\min}] 
    \abs{x_j} }}}
    \\ = &\frac{1}{1 + \summ[\abs{\tilde{\Lambda}} = d-k,\ \tilde{\Lambda} \neq 
    \Lambda_{\min}]\exp\br{-\gamma \br{\summ[j\in\tilde{\Lambda}] \abs{x_j} - 
    \summ[j\in\Lambda_{\min}] \abs{x_j} }}} 
    \\ \geq &\frac{1}{1 + \br{\binom{d}{k}-1} \exp\br{-\gamma \delta}}
    \underset{\gamma \app \infty}{\longrightarrow} 1,
    }}
   with the rightmost inequality following from 
   \cref{eq:proof_alg2_convergence_c}. Thus,
   \eq[eq:proof_alg2_convergence_e]{\ensuremath{\frac{\exp\br{-\gamma 
   \summ[j\in\Lambda_{\min}] \abs{x_j}}}{\summ[\abs{\tilde{\Lambda}} = 
   d-k]\exp\br{-\gamma \summ[j\in\tilde{\Lambda}] \abs{x_j}}} \app[\gamma \app 
   \infty] 1.}}
    Note that the limits in 
    \cref{eq:proof_alg2_convergence_d,eq:proof_alg2_convergence_e} do not 
    depend on $\x$, provided that $\norm{\x-\x^*}_\infty < \delta$. By 
    \cref{eq:w_at_gamma_infty}, this implies that as $\gamma \app \infty$, 
    $\wfunc \ofx$ converges to $\w_{k,\infty} \ofx$ uniformly on $B_\delta$.
      
    Next, let $\vvec \in \R[d]$. We shall now show that for any $\varepsilon > 
    0$ there exists $M$ such that if $i > M$, then $\nabla_{\vvec} \tauk 
    \br{\xstar} - \nabla_{\vvec} \tauargs{k}{\gamma_{r_i}} \br{\x^*_{r_i}} \geq 
    -\varepsilon \norm{\vvec}_1$. Let $\brc{\x^*_{r_i}}_{i=1}^\infty$ be a 
    subsequence of $\x^*_r$ that converges to $\x^*$. Let $I_1$ such that for 
    all $i>I_1$, $\norm{\x^*_{r_i}-\x^*}_{\infty} < \delta$. Recall that 
    $\gamma_r \app[r \app \infty] \infty$. Let $I_2 > I_1$ such that for all $i 
    > I_2$ and all $\x \in B_\delta$, $\norm{\w_{k,\gamma_{r_i}}\br{\x} - 
    \w_{k,\infty}\br{\x} }_\infty < \varepsilon$. Let $M > I_2$ such that if $i 
    > M$, then for any $j$ such that $x^*_j \neq 0$, 
    $\mbox{sign}\br{\br{\x^*_{r_i}}_j} = \mbox{sign}\br{x^*_j}$. Let $i > M$. 
    Since $\x^*$ is non-ambiguous and $\gamma_{r_i}$ is finite, by 
    \cref{thm:dirdev_tau_phi_gamma_infty},
    \eq[eq:thm_alg2_convergence_1]{\ensuremath{ \nabla_{\vvec} \tauk 
    \br{\xstar} - \nabla_{\vvec} \tauargs{k}{\gamma_{r_i}} \br{\x^*_{r_i}}
    = \sum_{j=1}^d \Biggbr{\wargs{k}{\infty}{j} \br{\xstar} \delta 
    \br{x^*_j,v_j} v_j - \wargs{k}{\gamma_{r_i}}{j} \br{\x^*_{r_i}} \delta 
    \br{\br{\x^*_{r_i}}_j,v_j} v_j}.
    }    }
    For any $j \in \brs{d}$, if $x^*_j = 0$ then
\eq{\ensuremath{\delta \br{x^*_j,v_j} v_j = \abs{v_j} \geq \delta 
\br{\br{\x^*_{r_i}}_j,v_j} v_j,}}
and if $x^*_j \neq 0$, then $\mbox{sign}\br{\br{\x^*_{r_i}}_j} = 
\mbox{sign}\br{x^*_j}$, and thus $\delta \br{\br{\x^*_{r_i}}_j,v_j} v_j = 
\delta \br{\br{x^*_j},v_j} v_j$. Therefore, for any $j \in \brs{d}$, $\delta 
\br{x^*_j,v_j} v_j \geq \delta \br{\br{\x^*_{r_i}}_j,v_j} v_j$. Plugging this 
to \cref{eq:thm_alg2_convergence_1}, we have
\eq{\ensuremath{
    &\nabla_{\vvec} \tauk \br{\xstar} - \nabla_{\vvec} 
    \tauargs{k}{\gamma_{r_i}} \br{\x^*_{r_i}}
    \geq \sum_{j=1}^d \Biggbr{\wargs{k}{\infty}{j} \br{\xstar} \delta 
    \br{x^*_j,v_j} v_j - \wargs{k}{\gamma_{r_i}}{j} \br{\x^*_{r_i}} \delta 
    \br{x^*_j,v_j} v_j}
    \\ & = \sum_{j=1}^d \Biggbr{\wargs{k}{\infty}{j} \br{\xstar} - 
    \wargs{k}{\gamma_{r_i}}{j} \br{\x^*_{r_i}} } \delta \br{x^*_j,v_j} v_j
     \geq - \norm{\w_{k,\infty} \br{\xstar} - \w_{k,\gamma_{r_i}} 
     \br{\x^*_{r_i}}}_\infty \sum_{j=1}^d \abs{v_j}
    \\ & \geq -\varepsilon \norm{\vvec}_1.        }}
    
    Now, let $\varepsilon > 0$ and let $i$ be large enough such that 
    $\nabla_{\vvec} \tauk \br{\xstar} - \nabla_{\vvec} 
    \tauargs{k}{\gamma_{r_i}} \br{\x^*_{r_i}} \geq -\varepsilon \norm{\vvec}_1$ 
    and $\norm{\x^*_{r_i} - \xstar}_2 \leq \varepsilon$. Since $\x^*_{r_i}$ is 
    a stationary point of $\mathrm{F}_{\lambda,\gamma_{r_i}}$, $\nabla_{\vvec} 
    \mathrm{F}_{\lambda,\gamma_{r_i}} \br{\x^*_{r_i}} \geq 0$, and thus
    \eq{\ensuremath{\nabla_{\vvec} \Ftwol \br{\xstar} 
        &\geq \nabla_{\vvec} \Ftwol \br{\xstar} - \nabla_{\vvec} 
        \mathrm{F}_{\lambda,\gamma_{r_i}} \br{\x^*_{r_i}} 
    \\ &= \inprod{A\xstar-\y}{A\vvec}-\inprod{A\x^*_{r_i}-\y}{A\vvec} + \lambda 
    \br{\nabla_{\vvec} \tauk \br{\xstar} - \nabla_{\vvec} 
    \tauargs{k}{\gamma_{r_i}} \br{\x^*_{r_i}}}
\\ &= \inprod{A\br{\xstar - \x^*_{r_i}}}{A\vvec} + \lambda \br{\nabla_{\vvec} 
\tauk \br{\xstar} - \nabla_{\vvec} \tauargs{k}{\gamma_{r_i}} \br{\x^*_{r_i}}}
\\ &\geq -\norm{A^tA\vvec}_2 \norm{\xstar - \x^*_{r_i}}_2 - \lambda 
\norm{\vvec}_1 \varepsilon
   \geq -\br{\norm{A^tA\vvec}_2 + \lambda \norm{\vvec}_1}\varepsilon. }}
Since this is true for any $\varepsilon > 0$, $\nabla_{\vvec} \tauk \br{\xstar} 
\geq 0$. Since $\vvec$ is arbitrary, this implies that $\xstar$ is a stationary 
point of $\Ftwol$, which, by \cref{thm:stationary_pt_is_localmin}, implies that 
$\xstar$ is a local minimum of $\Ftwol$. Thus, we have shown that any partial 
limit of $\brc{\x^*_r}_{r=0}^{\infty}$ is either a local minimum of $\Ftwol$ or 
an ambiguous vector, so part~\ref{item:alg2_convergence_part3} of the 
\lcnamecref{thm:alg2_convergence} is proven.

Since the set $C$ is compact, by a similar argument to that used in the proof 
of \cref{thm:alg1_convergence_gamma_finite}, it can be shown that $\lim_{r 
\rightarrow \infty} d \br{\x^*_r, \xoptset \cup \xambset} = 0$, which proves 
part~\ref{item:alg2_convergence_part2} of the \lcnamecref{thm:alg2_convergence}.
\end{proof}

\section[Appendix C. Optimizing F1 at small values of lambda]{{Optimizing 
$\Fonel$ at small 
values of $\lambda$}}
\label{sec:optimizing_fonel_with_small_lambda}
As discussed in \cref{sec:penalty_thresholds}, when using the power\nbdash 1 
objective $\Fonel$ as a relaxation of \cref{pr:P0}, we restrict $\lambda$ to 
$\brs{\lambdasmall, \lambdalarge}$. This strategy is supported by 
\cref{thm:small_lambda_p1l}, which states that for $\lambda < \lambdasmall$, 
all the local minima of $\Fonel$ belong to the subspace of zero residual, and 
by \cref{thm:bad_local_minma_everywhere}, which states that for $\lambda > 
\lambdalarge$, the optimization landscape of $\Fonel$ is plagued by poor local 
minima. To further support this strategy, we now show that under mild 
assumptions, \cref{alg:homotopy} returns the same output for all $\lambda < 
\lambdasmall$. 
We stress that the discussion in this section refers to the modified versions 
of \cref{alg:mm,alg:homotopy}, adapted to minimize the power\nbdash 1 
objectives $\Fonelg$ and $\Fonel$ respectively.
We start by stating the following theorem, which is an extension of 
\cref{thm:small_lambda_p1l}. Proofs are at the end of this section.
\begin{theorem} \label{thm:small_lambda_p1lw_zerores}
        Suppose that $0 < \lambda < \lambdasmall$. Let $\gamma \in 
        \brs{0,\infty}$. Then any stationary point $\xhat$ of $\Fonelg$ 
        satisfies $A \xhat = \y$.
\end{theorem}
It follows from \cref{thm:small_lambda_p1lw_zerores} that if $\lambda < 
\lambdasmall$, all the iterates of \cref{alg:homotopy}, including the output 
$\xhat$, belong to the affine space $\brc{\x \vert A \x = \y}$. Thus, for any 
$\lambda < \lambdasmall$, \cref{alg:homotopy} essentially seeks a solution of 
the equality-constrained problem
\begin{equation*}
\underset{\x}{\textup{min}} \ \tauk \ofx \quad \textup{s.t.} \quad A \x = \y,
\end{equation*}
and does so without leaving the feasible set during the homotopy.

To shed further light on the optimization process, recall the weighted\nbdash 
$\ell_1$ problem \cref{pr:P1lw}:
\begin{equation*}
\min_{\x}\ \Fonelw\br{\x} = \norm{A\x-\y}_2 + \lambda \inprod{\w}{\abs{\x}}.
\end{equation*}
The following lemma shows that for $\lambda < \lambdasmall$ the optimal 
solutions of problem \cref{pr:P1lw} are independent of $\lambda$.
\begin{lemma}\label{thm:small_lamssbda_p1lw_equiv}
Let $\lambda_1, \lambda_2 \in \br{0,\lambdasmall}$ and $\w \in \brs{0,1}^d$ 
with $\summ[i=1][d]w_i = d-k$. Then the two corresponding problems 
\cref{pr:P1lw} with $\lambda = \lambda_1$ and $\lambda=\lambda_2$ are 
equivalent. Namely, $\xhat$ is optimal for \cref{pr:P1lw} with $\lambda = 
\lambda_1$ if and only if it is optimal for \cref{pr:P1lw} with $\lambda = 
\lambda_2$.
\end{lemma}

Recall that \cref{alg:homotopy} consists of iterative calls to \cref{alg:mm}, 
which in turn consists of iterative solutions of instances of \cref{pr:P1lw}. 
Therefore, by \cref{thm:small_lamssbda_p1lw_equiv}, for $\lambda < 
\lambdasmall$, the course of \cref{alg:homotopy} does not depend on $\lambda$, 
as long as problem~\cref{pr:P1lw} has a unique solution. This holds even if 
problem~\cref{pr:P1lw} has multiple solutions, as long as the solution of 
\cref{pr:P1lw} chosen by \cref{alg:mm} is the one with the minimal 
$\ell_2$\nbdash norm -- which is unique. Under this assumption, the output is 
indeed independent of $\lambda$, as stated by the following theorem.
\begin{theorem} \label{thm:small_lambda_same_solution}
Suppose that \cref{alg:mm} sets $\x^t$ to be the global optimum of 
\cref{pr:P1lw}
with the smallest $\ell_2$-norm. Then for all $\lambda < \lambdasmall$, 
\cref{alg:homotopy} computes the same sequence of iterates $\x_r^*$ and returns 
the same output $\xhat$.
\end{theorem}
In conclusion, \cref{thm:small_lambda_same_solution} indicates that in the 
power\nbdash 1 case there is no benefit in using more than one value of 
$\lambda$ below $\lambdasmall$, since all such values are practically 
guaranteed to yield the same result.

We now prove \cref{thm:small_lambda_p1lw_zerores,thm:small_lamssbda_p1lw_equiv}.
\begin{proof}[Proof of \cref{thm:small_lambda_p1lw_zerores}]
Fix $\w = \wfunc\br{\xhat}$. Since $\xhat$ is a stationary point of $\Fonelg$, 
by \cref{thm:dirdev_tau_phi_gamma_finite,thm:dirdev_tau_phi_gamma_infty}, it is 
a stationary point of the majorizer $\Gonelg\br{\x, \xhat}$, and thus a global 
minimum w.r.t. $\x$. Therefore, $\xhat$ is a global minimum of $\Fonelw$, and 
by \cref{thm:small_lambda_p1l_aux}, $A\xhat = \y$.
\end{proof}

\begin{proof}[Proof of \cref{thm:small_lamssbda_p1lw_equiv}]
From \cref{thm:small_lambda_p1l_aux} it follows that problems \cref{pr:P1lw} 
with $\lambda = \lambda_1$ and with $\lambda = \lambda_2$ are both equivalent 
to the constrained problem
$$\min_{\x}\ \inprod{\w}{\abs{\x}} \textup{ s.t. } A\x=\y,$$
and thus have the same set of optimal solutions.
\end{proof}

\begin{proof}[Proof of \cref{thm:small_lambda_same_solution}]
        The main loop of \cref{alg:homotopy} consists of iterative calls to 
        \cref{alg:mm} with gradually increasing values of $\gamma$. Let 
        $\brc{\gamma_r}_{\ell=0}^L$ be the sequence of values used by 
        \cref{alg:homotopy} and let $\brc{\x_r^*}_{\ell=0}^L$ be the 
        corresponding iterates. We now prove by induction that the iterates 
        $\x_r^*$ do not depend on $\lambda$. 
        
        At the first iteration $r=0$, \cref{alg:mm} is called with 
        $\gamma_0=0$, for which it does not take an initialization. Its output 
        $\x_0^*$ is the minimal $\ell_2$\nbdash norm solution of 
        problem~\cref{pr:P1lw} with weights $w_i = \frac{d-k}{d}$, 
        $i=1,\ldots,d$. Thus, $\x_0^*$ is uniquely determined. By 
        \cref{thm:small_lamssbda_p1lw_equiv}, the set of optimal solutions of 
        problem~\cref{pr:P1lw} is independent of $\lambda$. Therefore, since 
        $\x_0^*$ is uniquely determined, it is independent of $\lambda$, and 
        thus so is $\gamma_1$ (which may still depend on $\x_0^*$). 
        
        Now let $r \geq 1$. By applying the same argument recursively to the 
        iterations of \cref{alg:mm}, it can be shown that if $\gamma_r$ and 
        $\x_{r-1}^*$ are independent of $\lambda$, then so are all the iterates 
        of \cref{alg:mm} when given $\gamma_r$ and $\x_{r-1}^*$ as input. 
        Therefore, $\gamma_{r+1}$ and $\x_{r}^*$ are also independent of 
        $\lambda$. This concludes our proof.
\end{proof}

\section[Appendix D. Calculating theta(z)]{Calculating $\gsm \ofz$} 
\label{sec:calculating_theta}
%
%
{As detailed below, we split $\gsm \ofz$ to two subvectors and calculate each 
of them separately. The entries $\gsmi{i}\ofz$ corresponding to the smallest 
$d-2k+2$ entries of $\z$ are calculated by a recursive procedure, 
taking $\mathcal{O}\br{kd}$ operations. As will follow from the lemmas below, 
these $d-2k+2$ entries are "well behaved" in the following sense: All the 
intermediate values in the recursive procedure are bounded in $\brs{0,1}$ and 
are monotonically decreasing -- hence they do not overflow, and an underflow 
can only affect values whose contribution to the final result is anyway 
negligible. For the remaining $2k-2$ entries, the above boundedness property 
does not hold in general. Hence, we calculate them by a separate procedure, 
which takes $\mathcal{O}\br{k^2}$ operations. The calculation of $\mukg \ofz$ 
and $\gsm \ofz$ for $k \in \brc{0,\ldots,\floor{\tfrac{d}{2}}}$ and $\gamma \in 
\brs{0,\infty}$ is summarized in \cref{alg:calc_gsm_main}.}

\begin{algorithm}[t] 
\caption{Calculate $\gsmiargs{k}{\gamma}{i} \ofz$ by a forward recursion}
\label{alg:calc_theta_forward}
\begin{algorithmic}[1]
\Input $\z = \br{z_1,\ldots,z_d} \in \R[d]$, $\ k,i \in \brc{1,\ldots,d}$, $\ 
\gamma \in \br{0,\infty}$, $\ \brc{b_{q,\gamma} \ofz}_{q=0}^k$, $\ 
\brc{z_\br{q}}_{q=0}^{k}$
\Output $\gsmiargs{k}{\gamma}{i} \ofz$
\State $\xi \assign 0 $
\For{$q \assign 1,\ldots,k$}
    $\ \xi \assign \frac{q}{d-q+1} \exp \br{\gamma \br{z_i-z_\br{q}} + 
    b_{q-1,\gamma} \ofz - b_{q,\gamma} \ofz} \cdot \br{1 - \xi }$    
\EndFor
\State\Return $\gsmiargs{k}{\gamma}{i} \ofz = \xi$     
\end{algorithmic}
\end{algorithm}

Let $\dleft \eqdef \max\brc{1,2k-2}$ and $\dright \eqdef d - \dleft$. Define 
the vectors $\zleft \in \R[\dleft]$ and $\zright \in \R[\dright]$ by
\eq{\ensuremath{\zleft \eqdef \br{z_1,\ldots,z_{\dleft}}
\qquad \zright \eqdef \br{z_{\dleft+1},\ldots,z_d}.}}
To calculate $\gsm \ofz$, we take as input 
$\brc{b_{q,\gamma} \ofz}_{q=0}^k$, $\brc{b_{q,\gamma} \ofzleft}_{q=0}^\dleft$ 
and $\brc{b_{q,\gamma} \ofzright}_{q=0}^{k}$, calculated by 
\cref{alg:calc_mu_consecutive}. 
%
Here we also assume that $\z$ is sorted such that $\min 
\brc{z_1,\ldots,z_\dleft} \geq \max \brc{z_{\dleft+1},\ldots,z_d}$. An 
arbitrary vector $\z$ can be sorted in this way in $\mathcal{O}\br{d}$ 
operations, using a selection algorithm to find the $\dleft$\nbdash th largest 
entry.

We first state without proof that $\gsmargs{k}{\gamma} \ofz$ satisfies the 
recursive formula
\eq[eq:recursion_theta_1]{\ensuremath{
\gsmiargs{q}{\gamma}{i} \ofz = \twocase{ c_{q,\gamma}^i\ofz \br{1 - 
\gsmiargs{q-1}{\gamma}{i} \ofz} }{q=1,\ldots,d}
{ 0 }{q=0,}}}
where
\eq[eq:def_c_qg]{\ensuremath{
c_{q,\gamma}^i\ofz \eqdef \exp \br{\gamma \br{z_i-z_\br{q}}} 
\frac{s_{q-1,\gamma} \ofz} {s_{q,\gamma} \ofz}}}
and $s_{q,\gamma} \ofz$ is defined in \cref{eq:def_aqg}.
\Cref{eq:recursion_theta_1} is similar to the recursion 
\cref{eq:simple_gsm_recursion}. A key difference is that the intermediate 
values $t_k^i$ in \cref{eq:simple_gsm_recursion} can be extremely large and 
thus overflow, whereas $\gsmiargs{k}{\gamma}{i} \ofz$ are bounded in 
$\brs{0,1}$ and thus cannot overflow.

Our assumption on the order of $\z$ implies that for all $i > \dleft$ and $q 
\leq k$, the term $\exp \br{\gamma \br{z_i-z_\br{q}}}$ is bounded in 
$\brs{0,1}$ and thus cannot overflow. However, calculating $c_{q,\gamma}^i\ofz$ 
by \cref{eq:def_c_qg} is numerically unsafe, for the following reasons: First, 
the terms $s_{q-1,\gamma} \ofz$ and $s_{q,\gamma} \ofz$ might suffer, 
respectively, an arithmetic overflow or underflow, which could corrupt the 
result. Second, 
$c_{q,\gamma}^i \ofz$ can be extremely large whereas $1 - 
\gsmiargs{q-1}{\gamma}{i} \ofz$ can be extremely small. This may lead to 
multiplying a numerical infinity by a numerical zero, giving a meaningless 
result.
 
To address the first issue mentioned above, we reformulate $c_{q,\gamma}^i\ofz$ 
in terms of $b_{q,\gamma} \ofz$,
\eq[eq:altdef_c_qg]{\ensuremath{
    c_{q,\gamma}^i\ofz = \frac{q}{d-q+1} \exp \br{\gamma \br{z_i-z_\br{q}} + 
    b_{q-1,\gamma} \ofz - b_{q,\gamma} \ofz}.}}
As argued in \cref{sec:calculating_mu}, $\abs{b_{q,\gamma} \ofz}$ is 
upper-bounded by $q\log d$, and thus it cannot overflow. The second issue 
mentioned above is addressed in the following lemma.
\begin{lemma}\label{thm:theta_recursion_safe}
    Let $\z \in \R[d]$, $1 \leq k \leq \frac{d}{2}$ and $\gamma \in 
    \br{0,\infty}$.
    Suppose that $z_j \geq z_i$ for any $j \leq 2k-2 < i$.
    Then $c_{q,\gamma}^i\ofz \leq 1$ for all $i > 2k-2$ and $1 \leq q \leq k$. 
\end{lemma}
\Cref{thm:theta_recursion_safe} guarantees that for $i > \dleft$, 
$c_{q,\gamma}^i\ofz$ cannot overflow, and thus $\gsmiargs{k}{\gamma}{i} \ofz$ 
can be safely calculated by \cref{eq:recursion_theta_1,eq:altdef_c_qg}. This 
recursive procedure is summarized in \cref{alg:calc_theta_forward}. It takes 
$\mathcal{O}\br{k}$ operations for a single index $i$, and can be parallelized 
over $i$.


\begin{algorithm}[t]
\caption{Calculate $\brc{\gsmiargs{q}{\gamma}{i} \ofzleft}_{q=0}^k$ by a 
forward and backwards recursion}
\label{alg:calc_theta_backwards}
\begin{algorithmic}[1]
\Input $\zleft = \br{\zleftentry_1,\ldots,\zleftentry_\dleft} \in \R[\dleft]$, 
$\ k,i \in \brc{1,\ldots,\dleft}$, 
    $\ \gamma \in \br{0,\infty}$, $\ \brc{b_{q,\gamma} 
    \ofzleft}_{q=0}^\dleft$,  $\ \brc{\zleftentry_\br{q}}_{q=0}^{\dleft}$
\Output $\brc{\gsmiargs{q}{\gamma}{i}\ofzleft}_{q=0}^k$
\State Initialize $\gsmiargs{0}{\gamma}{i} \ofzleft \assign 0$, $\ 
\gsmiargs{\dleft}{\gamma}{i} \ofzleft \assign 1\ $ and $\ \qhat \assign k+1$
\For{$q \assign 1,\ldots,k$}
    \State $\eta \assign \gamma \br{\zleftentry_i-\zleftentry_\br{q}} + 
    b_{q-1,\gamma} \ofzleft - b_{q,\gamma} \ofzleft$
    \If{$\eta \leq \log{\frac{\dleft-q+1}{q}}$}
        $\gsmiargs{q}{\gamma}{i} \ofzleft \assign \frac{q}{\dleft-q+1} 
        \exp\br{\eta} \cdot \br{1 - \gsmiargs{q-1}{\gamma}{i} \ofzleft }$    
    \Else
        \ set $\qhat \assign q$ and break
    \EndIf
\EndFor
    \If {$\qhat \leq k$}
        \For{$q \assign \dleft-1,\dleft-2,\ldots,\qhat$}
            \State $\eta \assign b_{q+1,\gamma} \ofzleft - b_{q,\gamma} 
            \ofzleft - \gamma 
            \br{\zleftentry_i-\zleftentry_\br{q+1}}$            
            \State $\gsmiargs{q}{\gamma}{i}\ofzleft \assign 1 - 
            \frac{\dleft-q}{q+1} \exp\br{\eta} \cdot 
            \gsmiargs{q+1}{\gamma}{i}\ofzleft$    
        \EndFor
    \EndIf
\State\Return $\brc{\gsmiargs{q}{\gamma}{i}\ofzleft}_{q=0}^k$     
\end{algorithmic}
\end{algorithm}

For $i \leq \dleft$, $c_{q,\gamma}^i\ofz$ may be arbitrarily large. Thus, 
calculating $\gsmiargs{k}{\gamma}{i} \ofz$ by \cref{eq:recursion_theta_1} might 
yield meaningless results. We therefore use a different method for $i \leq 
\dleft$, based on the following lemma. The proof appears at the end of this 
section.
\begin{lemma}\label{thm:theta_recursion_monotone}
    Let $\z \in \R[d]$, $\gamma \in \br{0,\infty}$ and $i \in 
    \brc{1,\ldots,d}$. Then 
    $c_{q,\gamma}^i\ofz$ is strictly increasing with respect to $q \in 
    \brc{1,\ldots,d}$.
\end{lemma}
It follows from \cref{thm:theta_recursion_monotone} that if $c_{k,\gamma}^i\ofz 
\leq 1$, then all of $\brc{c_{q,\gamma}^i\ofz}_{q=1}^k$ are upper-bounded by 1, 
and thus $\gsmiargs{k}{\gamma}{i}\ofz$ can be safely calculated by 
\cref{eq:recursion_theta_1,eq:altdef_c_qg}.
At indices $i$ where $c_{k,\gamma}^i\ofz > 1$, one possible approach is to 
calculate $\gsmiargs{k}{\gamma}{i}\ofz$ backwards, starting from $q=d$, by the 
recursive formula derived from \cref{eq:recursion_theta_1},
\eq[eq:recursion_theta_2]{\ensuremath{
    \gsmiargs{q}{\gamma}{i} \ofz = \twocase{ 1 - 
    \frac{\gsmiargs{q+1}{\gamma}{i} \ofz}{c_{q+1,\gamma}^i\ofz} 
    }{q=1,\ldots,d-1}
    { 1 }{q=d.}}}
By \cref{thm:theta_recursion_monotone}, if $c_{k,\gamma}^i\ofz > 1$ then 
$\frac{1}{c_{q+1,\gamma}^i\ofz} < 1$ for $q=k,\ldots,d-1$, which ensures that 
an overflow cannot take place. However, this approach would require the 
pre-calculation of $\brc{b_{q,\gamma}\ofz}_{q=0}^d$, which takes 
\cref{alg:calc_mu_consecutive} an $\mathcal{O}\br{d^2}$ operations. Instead, we 
use a similar approach with a smaller running time: For each $i$ such that 
$c_{k,\gamma}^i\ofz > 1$, instead of calculating $\gsmiargs{k}{\gamma}{i}\ofz$ 
directly, we first calculate $\brc{\gsmiargs{q}{\gamma}{i}\ofzleft}_{q=0}^k$  
by \cref{eq:recursion_theta_1,eq:altdef_c_qg,eq:recursion_theta_2}. This only 
requires the pre-calculation of $\brc{b_{q,\gamma}\ofzleft}_{q=0}^{\dleft}$, 
which takes \cref{alg:calc_mu_consecutive} an $\mathcal{O}\br{k^2}$ operations. 
From the result, we then calculate $\gsmiargs{k}{\gamma}{i}\ofz$. We now 
describe this approach in detail.

Let $i \in \brc{1,\ldots,\dleft}$ such that $c_{k,\gamma}^i\ofz > 1$. First, we 
wish to calculate $\brc{\gsmiargs{q}{\gamma}{i} \ofzleft}_{q=0}^k$. If 
$c_{k,\gamma}^i\ofzleft \leq 1$, then by \cref{thm:theta_recursion_monotone}, 
$c_{q,\gamma}^i\ofzleft \leq 1$ for $q=1,\ldots,k$ and thus calculating 
$\brc{\gsmiargs{q}{\gamma}{i} \ofzleft}_{q=0}^k$ can be done by 
\cref{eq:recursion_theta_1,eq:altdef_c_qg}.
Suppose on the other hand that $c_{k,\gamma}^i\ofzleft > 1$. Note that 
$c_{1,\gamma}^i\ofzleft$ is always smaller than 1, as shown in the proof of 
\cref{thm:theta_recursion_safe}. Thus, by \cref{thm:theta_recursion_monotone}, 
there exists a unique $\qhat \in \brc{1,\ldots,k-1}$ such that 
$c_{q,\gamma}^i\ofzleft \leq 1$ for $q=1,\ldots,\qhat$ and 
$c_{q,\gamma}^i\ofzleft > 1$ for $q=\qhat+1,\ldots,\dleft$. We therefore 
calculate $\brc{\gsmiargs{q}{\gamma}{i} \ofzleft}_{q=0}^\qhat$ by 
\cref{eq:recursion_theta_1,eq:altdef_c_qg} and then calculate the remaining 
$\brc{\gsmiargs{q}{\gamma}{i} \ofzleft}_{q=\qhat+1}^{\dleft}$ by 
\cref{eq:recursion_theta_2,eq:altdef_c_qg}. 
This step is summarized in \cref{alg:calc_theta_backwards}. It takes an 
$\mathcal{O}\br{k}$ operations for each index $i \in \brc{1,\ldots,\dleft}$, 
and an additional $\mathcal{O}\br{k^2}$ operations to calculate 
$\brc{b_{q,\gamma}\ofzleft}_{q=0}^\dleft$ by \cref{alg:calc_mu_consecutive}.

\begin{algorithm}[t] 
\caption{Calculate $\Delta_{k,t} \br{\zleft,\zright}$ for $t=0,\ldots,k$}
\label{alg:calc_alpha_and_delta}  
\begin{algorithmic}[1]
\Input $\brc{z_\br{q}}_{q=0}^k$, $\ \brc{\zleftentry_\br{q}}_{q=0}^k$,  $\ 
\brc{\zrightentry_\br{q}}_{q=0}^{\min\brc{k,\dright}}$, where $\z = 
\brs{\zleft, \zright}$, $\zleft \in \R[\dleft]$, $\zright \in \R[\dright]$, $1 
\leq k \leq \dleft$ 
\Output $\Delta_{k,t} \br{\zleft,\zright}$ for $t=0,\ldots,k$
\Variables $\Delta_t \in \R\ $ for $\ t=0,\ldots,k$ 
\State Initialize $\Delta_t \assign 0$ for $0 \leq t \leq k$
\For{$q=1,\ldots,k$}
    \State $t_a \assign \max\brc{0,q-\dright}$, $t_b \assign \min\brc{q,k}$
    \For{$t=t_b,t_b-1,\ldots,\max\brc{t_a,1}$}
        \If{$z_\br{q} \geq \zleftentry_\br{t}$}
            $\Delta_t \assign \Delta_{t-1} + \br{z_\br{q} - \zleftentry_\br{t}}$
        \Else
            $\ \Delta_t \assign \Delta_{t} + \br{z_\br{q} - 
            \zrightentry_\br{q-t}}$
        \EndIf
    \EndFor
    \If{$t_a = 0$}
        $\Delta_0 \assign \Delta_0 + \br{z_\br{q} - \zrightentry_\br{q}}$
    \Else
        $\ \Delta_{t_a-1} \assign 0$
    \EndIf
\EndFor
\State\Return $\Delta_{k,t} \br{\zleft,\zright} = \Delta_t$ for $t=0,\ldots,k$
\end{algorithmic}
\end{algorithm}

\begin{algorithm}[t]
\caption{Calculate $\gsmiargs{k}{\gamma}{i} \ofz$ from 
$\brc{\gsmiargs{q}{\gamma}{i} \ofzleft}_{q=0}^{k}$}
\label{alg:convert_theta}   
\begin{algorithmic}[1]
\Input $b_{k,\gamma}\ofz$, 
$\brc{b_{q,\gamma}\ofzright}_{q=0}^{\min\brc{k,\dright}}$, 
$b_{q,\gamma}\ofzleft$, $\gsmiargs{q}{\gamma}{i} \ofzleft$, 
$\log\br{\twosetbinom{\dleft}{\dright}{k}{q}}$, $\Delta_{k,q} 
\br{\zleft,\zright}$ for $q=0,\ldots,k$, where $\z = \brs{\zleft, \zright} \in 
\R[d]$, $\zleft \in \R[\dleft]$, $\zright \in \R[\dright]$, $\gamma \in 
\br{0,\infty}$, $\ 1 \leq k \leq \dleft$
\State Initialize $\xi \assign 0$
\For{$q=k,k-1,\ldots,\max\brc{1,k-\dright}$}
     \State $\xi \assign \xi + \exp 
     \bigbr{\log\br{\twosetbinom{\dleft}{\dright}{k}{q}} + b_{q,\gamma} 
     \ofzleft + b_{k-q,\gamma} \ofzright - b_{k,\gamma} \ofz - \gamma 
     \Delta_{k,q}\br{\zleft,\zright}} \cdot \gsmiargs{q}{\gamma}{i} \ofzleft$
\EndFor
\State\Return $\gsmiargs{k}{\gamma}{i} \ofz = \xi$
\end{algorithmic}
\end{algorithm}

Next, to calculate $\gsmiargs{k}{\gamma}{i} \ofz$ from 
$\brc{\gsmiargs{q}{\gamma}{i} \ofzleft}_{q=0}^{k}$, we define 
$\Delta_{q,t}\br{\uvec,\vvec}$ for $\uvec \in \R[m]$, $\vvec \in \R[n]$ and $0 
\leq q,t \leq m+n$, by
\eq{\ensuremath{\Delta_{q,t}\br{\uvec,\vvec} \eqdef \twocase{ 
M_q\br{\brs{\uvec, \vvec}} - M_t\br{\uvec} - M_{q-t}\br{\vvec} 
}{\max\brc{0,q-n} \leq t \leq \min\brc{q,m}}{0}{\mbox{otherwise,}}}}
where $M_q^{\br{r}} \ofz$ is as in \cref{eq:def_Mqr} and $\brs{\uvec, \vvec}$ 
denotes the concatenation of $\uvec$ and $\vvec$.
Using the above definition, we state without proof that
\eq[eq:converting_thera]{\ensuremath{
\gsmiargs{k}{\gamma}{i} \ofz = \summ[t=\max\brc{1,k-\dright}][k] \brs{ 
\frac{s_{t,\gamma} \ofzleft s_{k-t,\gamma} \ofzright}{s_{k,\gamma} \ofz}
\exp \br{-\gamma \Delta_{k,t}\br{\zleft,\zright}} \cdot \gsmiargs{t}{\gamma}{i} 
\ofzleft }.}}
 It can be shown by a combinatorial argument that 
 for any
 $\max\brc{1,k-\dright} \leq t \leq k$,
 \eq{\ensuremath{
    \frac{s_{t,\gamma} \ofzleft s_{k-t,\gamma} \ofzright}{s_{k,\gamma} \ofz}
    \exp\br{-\gamma \Delta_{k,t}\br{\zleft,\zright}} \leq 1 }}
and thus the terms in the above left-hand side do not overflow. To avoid 
dividing by an underflowed $s_{k,\gamma} \ofz$ or multiplying by an overflowed 
$s_{t,\gamma} \ofzleft$ or $s_{k-t,\gamma} \ofzright$, we reformulate 
\cref{eq:converting_thera} in terms of $b_{t,\gamma}\ofz$ and get
\eq[eq:converting_theta_logarithmic]{\ensuremath{
\gsmiargs{k}{\gamma}{i} \ofz = \sum_{t=\max\brc{1,k-\dright}}^k \brs{  \exp 
\bigbr{\log\br{\twosetbinom{\dleft}{\dright}{k}{t}} + b_{t,\gamma} \ofzleft + 
b_{k-t,\gamma} \ofzright - b_{k,\gamma} \ofz - \gamma 
\Delta_{k,t}\br{\zleft,\zright}} \cdot \gsmiargs{t}{\gamma}{i} \ofzleft},
}}
where $\twosetbinom{m}{n}{q}{t}$ is defined for $0 \leq q,t \leq m+n$ by
\eq[eq:def_alpha_coeffs]{\ensuremath{
\twosetbinom{m}{n}{q}{t} \eqdef \twocase
{\frac{\binom{m}{t}\binom{n}{q-t}}{\binom{m+n}{q}}}
{ \max\brc{0,q-n} \leq t \leq \min\brc{q,m}}
{0}
{\mbox{otherwise.}}}}
This procedure to calculate $\gsmiargs{k}{\gamma}{i} \ofz$ from 
$\brc{\gsmiargs{q}{\gamma}{i} \ofzleft}_{q=0}^{k}$ is summarized in 
\cref{alg:convert_theta}. It takes an $\mathcal{O}\br{k}$ operations for each 
$i \in \brc{1,\ldots,\dleft}$. Since $\twosetbinom{m}{n}{q}{t}$ may be 
extremely small and suffer an underflow, the calculation of 
$\twosetbinom{m}{n}{q}{t}$ is performed in log space.

\begin{algorithm}[t]  
\caption{Main algorithm for calculating $\mukg\ofz$ and $\gsm\ofz$}
\label{alg:calc_gsm_main}
\begin{algorithmic}[1]
\Input $\z = \br{z_1,\ldots,z_d} \in \R[d]$, $k \in 
\brc{0,\ldots,\floor{\tfrac{d}{2}}}$, $\gamma \in \brs{0,\infty}$
\State \textbf{if} $k=0$, \Return $\mukg\ofz = 0$, $\gsmiargs{k}{\gamma}{i}\ofz 
= 0$ for $i=1,\ldots,d$
\State \textbf{if} $\gamma = 0$, 
    \Return $\mukg\ofz = \frac{k}{d}\summ[i=1][d]z_i$, 
    $\gsmiargs{k}{\gamma}{i}\ofz = \frac{k}{d}$ for $i=1,\ldots,d$
\State \textbf{if} $\gamma = \infty$,
    \Return calculation of $\mukg\ofz$ and $\gsm\ofz$ by 
    $\mbox{\cref{eq:def_mukg}}$, 
    $\mbox{\cref{eq:gsm_at_gamma_infty,thm:theta_limits_idxset_def1}}$
\State Let $\dleft \assign \max\brc{1,2k-2}$, $\dright \assign d-\dleft$
\State Sort $\z$ such that $\min\brc{z_1,\ldots,z_\dleft} \geq 
\max\brc{z_{\dleft+1},\ldots,z_d}$
\State Denote $\zleft = \br{z_1,\ldots,z_\dleft}$, $\zright = 
\br{z_{\dleft+1},\ldots,z_d}$
\State Calculate $\mukg\ofz$, $b_{q,\gamma}\ofz$ and $z_\br{q}$ for 
$q=0,\ldots,k$ by $\mbox{\cref{alg:calc_mu_consecutive}}$
\State Calculate $b_{q,\gamma}\ofzleft$ and $\zleftentry_\br{q}$ for 
$q=0,\ldots,\dleft$ by $\mbox{\cref{alg:calc_mu_consecutive}}$
\State Calculate $b_{q,\gamma}\ofzright$ and $\zrightentry_\br{q}$ for 
$q=0,\ldots,\min\brc{k,\dright}$ by $\mbox{\cref{alg:calc_mu_consecutive}}$
\State Calculate $\Delta_{k,t} \br{\zleft,\zright}$ for $t=0,\ldots,k$ by 
$\mbox{\cref{alg:calc_alpha_and_delta}}$
\State Calculate $\log\br{\twosetbinom{\dleft}{\dright}{k}{q}}$ for 
$t=0,\ldots,k$ by \cref{eq:def_alpha_coeffs}
\For{$i=1,\ldots,d$}
    \If{$i > \dleft$ or $\gamma \br{z_i-z_\br{k}} + b_{k-1,\gamma} \ofz - 
    b_{k,\gamma} \ofz \leq \log\br{\frac{d-k+1}{k}}$}
        \State Calculate $\gsmiargs{k}{\gamma}{i}\ofz$ by 
        $\mbox{\cref{alg:calc_theta_forward}}$
    \Else
        \State Calculate $\brc{\gsmiargs{q}{\gamma}{i} \ofzleft}_{q=0}^k$ by 
        $\mbox{\cref{alg:calc_theta_backwards}}$
        \State Calculate $\gsmiargs{k}{\gamma}{i}\ofz$ by 
        $\mbox{\cref{alg:convert_theta}}$
    \EndIf
\EndFor
\State Unsort $\gsmargs{k}{\gamma}\ofz$ to match the original order of $\z$
\end{algorithmic}
\end{algorithm}

Without careful treatment, small roundoff errors may yield 
$\Delta_{k,t}\br{\zleft,\zright} \neq 0$ in cases where  
$\Delta_{k,t}\br{\zleft,\zright} = 0$. Such small errors may be amplified in 
\cref{eq:converting_theta_logarithmic} when multiplied by $\gamma$ and 
exponentiated. Moreover, such errors may lead to negative values, while it can 
be shown that $\Delta_{q,t} \br{\zleft,\zright} \geq 0$ for all $0 \leq q,t 
\leq \dleft + \dright$. To address this issue, we calculate $\Delta_{k,t} 
\br{\zleft,\zright}$ by
\eq{\ensuremath{
\Delta_{q,t} \br{\uvec,\vvec} = \threecase
{0}
{ q=0 \mbox{ or } t \notin \brs{\max\brc{0,q-n}, \min\brc{q,m}} }
{ \Delta_{q-1,t-1} \br{\uvec,\vvec} + \br{w_\br{q} - u_\br{t}}}
{\text{otherwise if } w_\br{q} \geq u_\br{t} }
{\Delta_{q-1,t} \br{\uvec,\vvec} + \br{w_\br{q} - v_\br{q-t}}}
{\text{otherwise,}}}}
where $\uvec \in \R[m]$, $\vvec \in \R[n]$ and $\w = \brs{\uvec,\vvec}$. This 
ensures that the result equals zero whenever $\Delta_{q,t} 
\br{\zleft,\zright}=0$, and is always nonnegative. This  recursive calculation 
is summarized in \cref{alg:calc_alpha_and_delta}. It takes 
$\mathcal{O}\br{k^2}$ operations and only needs to be done once, as its output 
does not depend on an index $i$.

Finally, our main algorithm to calculate $\mukg \ofz$ and $\gsm \ofz$ for $k 
\in \brc{0,\ldots,\floor{\tfrac{d}{2}}}$ and $\gamma \in \brs{0,\infty}$ is 
summarized in \cref{alg:calc_gsm_main}. It takes $\mathcal{O}\br{kd}$ 
operations and $\mathcal{O}\br{k}$ memory.


\subsection{Proofs}
\label{appendix:proofs_computing_gsm}

\begin{proof}[Proof of \cref{thm:theta_recursion_safe}]
    \label{proof:theta_recursion_safe}
It can be shown that for any $i,q \in \brc{1,\ldots,d}$,
\eq[eq:theta_recursion_factor_a_to_b]{\ensuremath{\frac{q}{d-q+1} \exp 
\br{\gamma \br{z_i-z_\br{q}} + b_{q-1,\gamma} \ofz - b_{q,\gamma} \ofz} = \exp 
\br{\gamma \br{z_i-z_\br{q}}} \frac{s_{q-1,\gamma} \ofz}{s_{q,\gamma} \ofz}.}}
For $q=1$, using the definition \cref{eq:def_aqg} of $s_{q,\gamma} \ofz$, the 
right-hand side of \cref{eq:theta_recursion_factor_a_to_b} amounts to
\eq{\ensuremath{\exp\br{\gamma \br{z_i-z_\br{1}}} \frac{s_{0,\gamma} 
\ofz}{s_{1,\gamma} \ofz} = \frac{\exp\br{\gamma 
\br{z_i-z_\br{1}}}}{\summ[j=1][d]\exp\br{\gamma \br{z_j-z_\br{1}}}} < 1.}}
Thus, our claim holds for $q=1$. 
Next, suppose by contradiction that the claim does not hold for $q \geq 2$. 
Then, by \cref{eq:theta_recursion_factor_a_to_b} and the definition of 
$s_{q,\gamma}$ in \cref{eq:def_aqg},
\eq{\ensuremath{ 1
    &< \exp \br{\gamma \br{z_i-z_\br{q}}} \frac{s_{q-1,\gamma} 
    \ofz}{s_{q,\gamma} \ofz}
    = \exp \br{\gamma \br{z_i-z_\br{q}}}
    \frac{    \sum_{\abs{\Lambda} = q-1} \exp \br{\gamma 
    \br{\sum_{j\in\Lambda}z_j - \summ[j=1][q-1] z_{\br{j}}}}}
    {\sum_{\abs{\Lambda} = q} \exp \br{\gamma \br{\sum_{j\in\Lambda}z_j - 
    \summ[j=1][q] z_{\br{j}}}}}
    \\&= \exp \br{\gamma z_i}
    \frac{    \sum_{\abs{\Lambda} = q-1} \exp \br{\gamma \sum_{j\in\Lambda}z_j 
    }}
    {\sum_{\abs{\Lambda} = q} \exp \br{\gamma \sum_{j\in\Lambda}z_j }}.
}}
Thus,
\begin{equation}
\exp \br{\gamma z_i} \summ[\abs{\Lambda} = q-1] \exp\Bigbr{\gamma  
\summ[j\in\Lambda]z_j } > \summ[\abs{\Lambda} = q] \exp\Bigbr{\gamma 
\summ[j\in\Lambda]z_j }.\label{eq:proof_theta_recursion_safe_1}
\end{equation}
The left-hand side of \cref{eq:proof_theta_recursion_safe_1} can be decomposed 
as 
\begin{equation}
\begin{aligned}
\exp \br{\gamma z_i} \summ[\abs{\Lambda} = q-1] \exp\Bigbr{\gamma  
\summ[j\in\Lambda]z_j}
    = \summ[{\scriptstyle\abs{\Lambda} = q,\ }{\scriptstyle i \in \Lambda}] 
    \exp\Bigbr{\gamma  \summ[j\in\Lambda]z_j }
    + \exp \br{2\gamma z_i} \summ[{\scriptstyle\abs{\Lambda} = q-2,\ 
    }{\scriptstyle i \notin \Lambda}] \exp\Bigbr{\gamma  \summ[j\in\Lambda]z_j 
    }.\label{eq:proof_theta_recursion_safe_2} 
\end{aligned}
\end{equation}
Inserting \cref{eq:proof_theta_recursion_safe_2} to 
\cref{eq:proof_theta_recursion_safe_1}, we have
\begin{equation}
\begin{aligned}
\label{eq:proof_theta_recursion_safe_3}
\exp \br{2\gamma z_i} \summ[\abs{\Lambda} = q-2,\ i \notin \Lambda] 
\exp\Bigbr{\gamma  \summ[j\in\Lambda]z_j}
&> \summ[\abs{\Lambda} = q] \exp \Bigbr{\gamma \summ[j\in\Lambda]z_j }
- \summ[\abs{\Lambda} = q,\ i \in \Lambda] \exp\Bigbr{\gamma  
\summ[j\in\Lambda]z_j}
\\&= \summ[\abs{\Lambda} = q,\ i \notin \Lambda] \exp\Bigbr{\gamma  
\summ[j\in\Lambda]z_j}.
\end{aligned}
\end{equation}
We now represent each set $\Lambda$ on the right side of 
\cref{eq:proof_theta_recursion_safe_3} by a disjoint union $\Lambda=\Omega 
\mathbin{\mathaccent\cdot\cup} J$, where $\Omega$ and $J$ are sets of sizes 
$q-2$ and $2$ respectively. To account for the multiple number of such 
representations for each $\Lambda$, we divide by $\binom{q}{2}$. Thus,
\begin{equation}
\label{eq:proof_theta_recursion_safe_4}
\exp \br{2 \gamma z_i} \summ[\twofloors{\scriptstyle\abs{\Omega} = 
q-2}{\scriptstyle i \notin \Omega}] \exp \Bigbr{\gamma \summ[j\in\Omega]z_j } > 
\frac{1}{\binom{q}{2}} \summ[\twofloors{\scriptstyle\abs{\Omega} = 
q-2}{\scriptstyle i \notin \Omega}]  \Biggbrs{ \exp\Bigbr{\gamma\summ[j \in 
\Omega]z_j } \summ[\twofloors{\scriptstyle \abs{J} = 2,\ i \notin 
J}{\scriptstyle J\cap\Omega=\phi}] \exp \Bigbr{\gamma \summ[j\in J]z_j}}.
\end{equation}
For any index-set $\Omega$ of size $q-2$, denote
\eq{\ensuremath{w_\Omega \eqdef \frac{1}{\binom{q}{2}} \exp\Bigbr{\gamma 
\summ[j\in\Omega]z_j},\qquad s_\Omega \eqdef \summ[\twofloors{\scriptstyle 
\abs{J} = 2,\ i \notin J}{\scriptstyle J\cap\Omega=\phi}] 
\exp\Bigbr{\gamma\summ[j \in J]z_j }.}}
Reformulating \cref{eq:proof_theta_recursion_safe_4} with the above 
definitions, we have
\eq{\ensuremath{\exp \br{2 \gamma z_i} \binom{q}{2} \summ[\abs{\Omega} = q-2,\ 
i \notin \Omega] w_\Omega > \summ[\abs{\Omega} = q-2,\ i\notin \Omega] w_\Omega 
s_\Omega.}}
Therefore,
\begin{equation}
\label{eq:proof_theta_recursion_safe_5}
\exp \br{2 \gamma z_i} > \frac{1}{\binom{q}{2}}\frac{\summ[\abs{\Omega} = q-2,\ 
i \notin \Omega]  w_\Omega s_\Omega}{\summ[\abs{\Omega} = q-2,\ i\notin \Omega] 
w_\Omega}
    \geq \frac{1}{\binom{q}{2}} \min_{\abs{\Omega} = q-2,\ i\notin \Omega} 
    s_\Omega,
\end{equation}    
The minimum in the right-hand side of \cref{eq:proof_theta_recursion_safe_5} is 
attained by any index-set $\Omega$ that corresponds to the $q-2$ largest 
entries among $\setst{z_j}{j\neq i}$. By our assumption, the subvector 
$\br{z_1,\ldots,z_{2k-2}}$ {contains the $2k-2$ largest entries of $\z$, and 
note that $2k-2 > q-2$.} Thus, there exists an index-set $\Omega_0 \subseteq 
\brc{1,\ldots,2k-2}$ of size $q-2$ that minimizes $s_\Omega$, and $i \notin 
\Omega_0$. Therefore, by \cref{eq:proof_theta_recursion_safe_5},
\begin{equation}
\label{eq:proof_theta_recursion_safe_6}
\exp \br{2 \gamma z_i} > \frac{1}{\binom{q}{2}} s_{\Omega_0} = 
\frac{1}{\binom{q}{2}}\summ[\twofloors{\scriptstyle \abs{J} = 2,\ i \notin 
J}{\scriptstyle J\cap\Omega_0=\phi}] \exp\Bigbr{\gamma\summ[j \in J]z_j }.
\end{equation}

By our assumption on $\z$, the subvector $\br{z_1,\ldots,z_{2k-2}}$ {contains 
$2k-2-\br{q-2}$ (and thus at least $q$) entries} $z_j$ such that $j \notin 
\Omega_0$ and $z_j \geq z_i$. Let $\Lambda_0$ be the index-set of $q$ such 
entries. Then $\abs{\Lambda_0}=q$, $i \notin \Lambda_0$ and $\Lambda_0 \cap 
\Omega_0 = \phi$. Keeping only the summands for which $J \subseteq \Lambda_0$ 
in \cref{eq:proof_theta_recursion_safe_6}, and discarding of all other 
summands, yields
\eq[eq:proof_theta_recursion_safe_7]{\ensuremath{\exp \br{2 \gamma z_i} 
    > \frac{1}{\binom{q}{2}} \summ[\abs{J} = 2,\ J \subseteq \Lambda_0] 
    \exp\Bigbr{\gamma\summ[j \in J]z_j }.}}
Note that in \cref{eq:proof_theta_recursion_safe_7}, in the sum over possible 
sets $J$, $z_j \geq z_i$ for all $j \in J$. Therefore,
\begin{equation*}
\exp \br{2 \gamma z_i} 
    > \frac{1}{\binom{q}{2}} \summ[\abs{J} = 2,\ J \subseteq \Lambda_0] 
    \exp\br{2\gamma z_i} = \exp\br{2\gamma z_i},
\end{equation*}
which is a contradiction. Thus, our claim is proven.
\end{proof}

\begin{proof}[Proof of \cref{thm:theta_recursion_monotone}]
    \label{proof:theta_recursion_monotone}
    By \cref{eq:theta_recursion_factor_a_to_b}, we need to show that for 
    $i=1,\ldots,d$, 
    the quantity $\exp \br{\gamma \br{z_i-z_\br{q}}} \frac{s_{q-1,\gamma} 
    \ofz}{s_{q,\gamma} \ofz}$
    is strictly increasing with respect to $q$. From \cref{eq:def_aqg},
    \eq{\ensuremath{
        &\frac{\exp \br{\gamma \br{z_i-z_\br{q+1}}} \frac{s_{q,\gamma} 
        \ofz}{s_{q+1,\gamma} \ofz}}
        {\exp \br{\gamma \br{z_i-z_\br{q}}} \frac{s_{q-1,\gamma} 
        \ofz}{s_{q,\gamma} \ofz}}
        = \frac{\exp \br{\gamma z_\br{q}}}{\exp \br{\gamma z_\br{q+1}}} 
        \frac{\br{s_{q,\gamma} \ofz}^2}{s_{q+1,\gamma} \ofz s_{q-1,\gamma} \ofz}
        \\ &= \frac{\exp \br{\gamma z_\br{q}} \summfixed[\abs{\Lambda} = q] 
        \exp \br{\gamma \br{\summfixed[j\in\Lambda] z_j - \summfixed[j=1][q] 
        z_{\br{j}}}}  \summfixed[\abs{\Omega} = q] \exp \br{\gamma 
        \br{\summfixed[j\in\Omega] z_j - \summfixed[j=1][q] z_{\br{j}}}}}{\exp 
        \br{\gamma z_\br{q+1}} \summfixed[\abs{\Lambda} = q+1] \exp \br{\gamma 
        \br{\summfixed[j\in\Lambda] z_j - \summfixed[j=1][q+1] z_{\br{j}}}} 
        \summfixed[\abs{\Omega} = q-1] \exp \br{\gamma 
        \br{\summfixed[j\in\Omega] z_j - \summfixed[j=1][q-1] z_{\br{j}}}}}
        \\ &= \frac{\sum_{\abs{\Lambda} = \abs{\Omega} = q } \exp \br{\gamma 
        \br{\sum_{j\in\Lambda}z_j + \sum_{j\in\Omega}z_j}} 
        }{\sum_{\abs{\Lambda} = q+1,\ \abs{\Omega} = q-1 } \exp \br{\gamma 
        \br{\sum_{j\in\Lambda}z_j + \sum_{j\in\Omega}z_j}}}.
    }}     
    Each term $\exp \br{\gamma \br{\sum_{j\in\Lambda}z_j + 
    \sum_{j\in\Omega}z_j}}$ in the above display is represented by two index 
    sets $\Lambda,\Omega \subseteq \brs{d}$. In the denominator, $\Lambda$ and 
    $\Omega$ consist of $q+1$ and $q-1$ unique indices respectively, whereas in 
    the numerator, both sets consist of $q$ unique indices. Note that each 
    exponential term in the numerator and denominator may appear multiple 
    times, as it may have several representations by different combinations 
    $\Lambda$ and $\Omega$. We shall now show that each exponential term 
    appears more times in the numerator than in the denominator, whence their 
    quotient is larger than 1. Let $\exp \br{\gamma \br{\sum_{j\in\Lambda}z_j + 
    \sum_{j\in\Omega}z_j}}$ be a term that appears in the denominator, with 
    $\abs{\Lambda}=q+1$ and  $\abs{\Omega}=q-1$. Let $h = \abs{\Lambda \cap 
    \Omega}$, and note that $0 \leq h \leq q-1$. By a combinatorial argument, 
    it can be shown that the aforementioned term appears exactly 
    $\binom{2q-2h}{q-h-1}$ times in the denominator and $\binom{2q-2h}{q-h}$ 
    times in the numerator. Furthermore,
    \eq{\ensuremath{\frac{\binom{2q-2h}{q-h}}{\binom{2q-2h}{q-h-1}}
        = \frac{q-h+1}{q-h} > 1,}}
    which completes our proof.
\end{proof}


\section[Appendix E. Computational details]{Computational details} 
\label{appendix:computational_details}
%

\subsection{GSM Internal Parameters}

We implemented our GSM method in Matlab, except for the computation of 
$\taukg\ofx$ and $\wfunc \ofx$, which was coded in C. We now describe several 
of the internal parameters in our implementation.  

\myparagraph{Updating $\gamma$ in \cref{alg:homotopy}} 
Let $\x_0^*$ be a global minimizer of \cref{pr:P2lzero}, which corresponds to 
$\gamma_0=0$. The next value $\gamma_1>0$ is chosen so that the weight vector 
$\w_{k,\gamma_1} \br{\x_0^*}$ is nearly uniform, and thus the first instance of 
\cref{pr:P2lg} is close to the convex \cref{pr:P2lzero}. For a small $\delta_0 
> 0$, we set 
\begin{equation}
\label{eq:def_gamma1}
\begin{split}
\gamma_1 &= \delta_0 \cdot \br{\max_{\abs{\Lambda}=d-k} \sum_{i \in \Lambda} 
\abs{\x_0^*}_i - \min_{\abs{\Lambda}=d-k} \sum_{i \in \Lambda} 
\abs{\x_0^*}_i}^{-1}.
\end{split}
\end{equation}
This rule guarantees that the ratio of the largest to smallest exponents in 
\cref{eq:def_wkg} is $\exp \br{\delta_0} \approx 1 + \delta_0$. This, in turn, 
{can be shown to imply} that all  entries of $\w_{k,\gamma_1} \br{\x_0^*}$ are 
nearly uniform,  
\eq{\ensuremath{\abs{\frac{w^i_{k,\gamma_1}\br{\x_0^*} - 
w^j_{k,\gamma_1}\br{\x_0^*}}{w^i_{k,\gamma_1}\br{\x_0^*}}} \lessapprox  
\delta_0 
\qquad \forall i,j \in \brs{d}.   }}
Next, for $r \geq 2$, $\gamma$ is increased exponentially by a user-chosen 
growth rate parameter $\delta_\gamma>0$, 
 \eq[eq:gamma_increase]{\ensuremath{\gamma_r = \br{1+\delta_\gamma} 
 \gamma_{r-1}.
 }}

To accelerate the homotopy scheme, once every $n_\gamma$ iterations, we 
increase $\gamma$ by a multiplicative factor $\br{1+\bar{\delta}_\gamma}$,
where $\bar{\delta}_{\gamma}$ is much larger than $\delta_{\gamma}$.  If the 
resulting iterate $\x_{r}^*$ satisfies
\eq[eq:cond_gamma_large_increase]{\ensuremath{\norm{\x_{r}^* - \x_{r-1}^*}_1 
\leq \frac{\norm{\y}_2}{\max_{i=1,\ldots,d} \norm{{\bf a}_i}_2} 
\varepsilon_{\x},}}
with $\varepsilon_{\x}= 10^{-6}$, then the larger increase is accepted and the 
algorithm proceeds as usual. Otherwise the increase is revoked, the algorithm 
backtracks to the beginning of iteration $r$, and $\gamma_r$ is set by 
\cref{eq:gamma_increase}.
We used $\delta_0 = 10^{-4}$, $\delta_\gamma = 0.02$, $\bar{\delta}_\gamma = 9$ 
and $n_{\gamma}=10$.


\myparagraph{Stopping criteria}
\Cref{alg:mm} was stopped if
$$\Ftwolg \br{\x^t} \geq \br{1-10^{-6}}\Ftwolg \br{\x^{t-1}},$$
or if for two consecutive iterations
$$\Ftwolg \br{\x^t} \geq \br{1-10^{-3}}\Ftwolg \br{\x^{t-1}}.$$ 

For \cref{alg:homotopy} we used the following stopping criteria: (i) 10 
consecutive iterations with $\x_{r}^*$ all $k$\nbdash sparse and with the same 
support; (ii) 4 consecutive values of $\gamma$ with the weight vector 
$\w_{k,\gamma_r} \br{\x_r^*}$ being approximately $\br{d-k}$-sparse, namely  
$\taudmk\br{\w_{k,\gamma_r} \br{\x_r^*}} \leq \br{d-k} \cdot \varepsilon_w$,
with $\varepsilon_w = 10^{-5}$. Once one of these conditions was met, a final 
iteration was run with $\gamma = \infty$.

\myparagraph{Dealing with ambiguous vectors}
As discussed in \cref{sec:convergence_analysis}, \cref{alg:homotopy} may output 
an ambiguous vector $\xhat$ that is not a local optimum of \cref{pr:P2l} 
or~\cref{pr:P1l}. Empirically, such $\xhat$ invariably has support size smaller 
than $k$. 
To overcome this problem, we coded two post-processing schemes, a slow and 
thorough one and a faster alternative. 
The slow scheme augments the support by running LS-OMP. 
The fast scheme, in contrast, performs only a single iteration of OMP. 
Specifically, given an ambiguous vector
$\xhat$ with $\|\xhat\|_0<k$ and
support set $\Lambda$, it picks the index $i\notin\Lambda$ that maximizes
$|\inprod{{\bf a}_i}{A\xhat-\y}|$. 
It then computes a new solution by least squares minimization $\|A\x-\y\|_2$
over vectors $\x$ with support 
$\Lambda\cup \{i\}$. 

In \cref{sec:sparse_recovery_simulation}, where $d \leq 1000$, we applied the 
slow scheme, whereas in \cref{sec:comparison_mip}, where $d \geq 5000$, we 
applied the faster scheme.

\subsection[Optimizing F by GSM, DC-programming and ADMM]{Optimizing 
$\Ftwol$ 
by GSM, DC-programming and ADMM}

Here we present some technical details on the comparison between 
the DC Programming and ADMM methods, as described in \cite{bertsimas2017trimmed}
and our GSM in solving problem~\cref{pr:P2l}
at a given value of $\lambda$. 

Note that instead of \cref{pr:P2l}, \cite{bertsimas2017trimmed} 
considered the following regularized objective, with $\eta\ll 1$
\[
\min_{\x} \frac{1}{2} \norm{A \x - y}_2^2 + \lambda \tauk \br{\x} + \eta 
\norm{\x}_1.
\]
We implemented in Matlab their DC programming and ADMM schemes, using MOSEK 
\cite{mosekcite} and YALMIP \cite{yalmipcite}. We used two values
for $\eta$: $10^{-2}$ as in \cite{bertsimas2017trimmed}, and $\eta=10^{-6}$. 

In some problem instances, the performance of DC-programming and of ADMM may 
strongly depend on their initial vector $\x^0$. We considered the two following 
initializations: 
(i) the standard one, $\x^0=\zerovec$; (ii) $\x^0$ is the
solution of \cref{pr:P2l} with 
a very small $\lambda = 10^{-5} \lambdabar$, obtained by the same method
started from the zero vector. 

For a fair comparison, we ran two variants of our GSM: (i) vanilla GSM, as 
described in the main text; (ii) we first ran vanilla GSM, but with
$\lambda=10^{-5}\lambdabar$. Denote the
resulting reference solution by $\x_{\mbox{\tiny Ref}}$. 
We then ran \cref{alg:homotopy}
but with a following change in line \ref{line:homotopy_solve}: If for $\gamma = 
\gamma_r$, $\Ftwolg\br{\xtilde} < \Ftwolg\br{\x_{r-1}^*}$, then \cref{alg:mm} 
is initialized by $\x_{\mbox{\tiny Ref}}$ instead of by $\x_{r-1}^*$.

For all three methods, the best of its two outputs, in terms of $\Ftwol\ofx$, 
was chosen.

\subsection[Methods for solving (P0)]{Methods for solving \cref{pr:P0}}
%
We now present technical details for the various methods in solving 
\cref{pr:P0},
at a given sparsity level $k$. Several of these methods output multiple 
solutions that depend on say a regularization parameter $\lambda$ or even two 
parameters. 
For each tested method, we took each of
its candidate solutions, projected it
to its $k$ largest-magnitude entries, and solved a least-squares problem on its 
support. The resulting solution with the smallest $\ell_2$ residual norm was 
chosen.

\myparagraph{GSM}
For the simulations of
\cref{sec:sparse_recovery_simulation}, which involve relatively small matrices,
we ran our GSM method with 50 values of $\lambda$. 
For the power\nbdash 2 variant, we used exponentially increasing values
$\lambda_i = 10^{-8 \frac{50-i}{50-1}} \br{1+\delta_\lambda}\bar{\lambda}$ with 
$i=1,\ldots,50$, $\delta_\lambda = 10^{-4}$ and $\lambdabar$ as in 
\cref{eq:lambdabar}. For the power\nbdash 1 variant, we used the sequence
\eq{\ensuremath{\lambda_i = \br{1+\delta_\lambda} \lambdalarge \cdot \tan \br{ 
\frac{50-i}{50-1} \arctan \br{ \frac{1-\delta_\lambda}{1+\delta_\lambda} \cdot 
\frac{\lambdasmall}{\lambdalarge}}  + \frac{i-1}{50-1} \frac{\pi}{4} }, \quad 
i=1,\ldots,50,}}
with $\lambdasmall$, $\lambdalarge$ as in \cref{eq:def_lambda_b}. This set of 
values is roughly twice denser near $\lambdasmall$ than near $\lambdalarge$. 
For faster runtime, in both power\nbdash 1 and power\nbdash 2 variants, if the 
obtained solutions were $k$\nbdash sparse for 7 consecutive values of 
$\lambda$, the algorithm was stopped without further executions for higher 
values of $\lambda$. The condition used to determine if $\x$ is $k$\nbdash 
sparse was
\eq[eq:def_epsilon_x]{\ensuremath{\tauk\ofx \leq k \cdot \varepsilon_{\x},}}
with $\varepsilon_{\x} = 10^{-6}$. 

For the simulations in \cref{sec:comparison_mip}, which involved much larger 
matrices, we used
only the following 7 values of $\lambda$: $\lambda_i = 10^{-3 \frac{7-i}{7-1}} 
\br{1+\delta_\lambda}\bar{\lambda}$ with $i=1,\ldots,7$. The algorithm was 
stopped at the first value of $\lambda$ that yielded a $k$-sparse solution.

\myparagraph{$\ell_p$ minimization}
IRLS and IRL1 minimize the $\ell_p$-regularized problem 
\begin{equation*}
\min_{\x} \frac{1}{2}\norm{A \x - \y}^2 + \lambda \norm{\x}_p^p
\end{equation*}
as a surrogate for the original problem~\cref{pr:P0}. In our simulations we ran 
both methods with the following 11 values of 
$p \in \brc{10^{-8}, 0.05, 0.1, 0.2, 0.3, \ldots, 0.9}$,
similarly to \cite{foucart2009sparsest}. 
For each $p$, we considered 90 exponentially increasing values of $\lambda$,
$\lambda_i = 10^{-8} \cdot 1.5^{i-1},\ i=1,\ldots,90$. 

As mentioned in \cref{sec:mm}, IRLS and IRL1 require weight regularization.
For IRLS, we used
\eq[eq:irls_a]{\ensuremath{ 
w^{t}_i \eqdef \br{\abs{x^{t-1}_i}^2+\varepsilon^2}^{\frac{p}{2}-1},
\qquad
\x^{t} \eqdef \argmin_{\x} \frac{1}{2}\norm{A \x - \y}_2^2 + \lambda 
\summ[i=1][d] {w^{t}_i \abs{x_i}^2},
}}
initialized by $\varepsilon=1$ and $\x^0 = \argmin_{\x} \norm{\x}_1\ 
\text{s.t.}\ A\x = \y$, 
as in \cite{lai2013improved}. If for three consecutive iterations the objective 
decreases by less than 0.1\%, $\varepsilon$ is updated by the adaptive rule of 
\cite{lai2013improved},
$\varepsilon \assign \min \brc{\varepsilon, \alpha_{\varepsilon} 
\abs{\x^{t}}_{\br{k+1}} }$, 
with $\alpha_\varepsilon = 0.9$. Iterations were stopped once 
$\varepsilon<10^{-8}$, or when the iterates $\x^t$ were $k$\nbdash sparse with 
a constant support for 10 consecutive iterations. 

For IRL1, we used the following reweighting scheme, similar to 
\cite{chartrand2008iteratively}, 
\eq{\ensuremath{ w^{t}_i &\eqdef \br{\abs{x^{t-1}_i}+\varepsilon}^{p-1},
\qquad
\x^{t} \eqdef \argmin_{\x} \frac{1}{2}\norm{A \x - \y}_2^2 + \lambda 
\summ[i=1][d] {w^{t}_i \abs{x_i}},}}
with the same initialization as for IRLS, and using the same update rule for 
$\varepsilon$.




\myparagraph{Least-Squares OMP}
LS-OMP was implemented in Matlab according to \cite[pg. 
37-38]{elad2010sparsebook},
modified to stop when reaching the target sparsity level $k$. 

\myparagraph{Basis Pursuit Denoising}
We solved the following constrained form
of BP, 
\eq{\ensuremath{\min_{\x} \norm{\x}_1\ \text{s.t.}\ \norm{A\x-\y}_2 \leq 
\delta,}}
with the following $61$ values of $\delta \in \setst{2^i 
\norm{A\x_0-\y}}{-30\leq i\leq 30}$.


\myparagraph{Iterative Support Detection}
We used the Matlab implementation of the Threshold-ISD algorithm \cite[sec. 
4.1]{wang2010sparse}, provided by the authors. 
Their method requires as input the noise parameter $\sigma$. We gave ISD a 
slight advantage by providing it with the ground-truth value. 



\myparagraph{Coordinate descent + local combinatorial search}
We used the \texttt{L0Learn} R package version 1.2.0, implemented by the 
authors of \cite{hazimeh2020fast}. 
To comply with our problem setup, 
we used the parameters \texttt{penalty="L0"} and \texttt{intercept=FALSE}. In 
all experiments, we used \texttt{maxSuppSize=$k$}, with $k$ being the 
ground-truth cardinality of $\x_0$.

\end{appendices}

\end{document}